\documentclass[11pt]{article}

\usepackage[utf8]{inputenc}
\usepackage[T1]{fontenc}
\usepackage{amsmath,amssymb,amsthm,mathtools}
\usepackage{mathrsfs}
\usepackage[a4paper,margin=1in]{geometry}
\usepackage{enumitem}
\usepackage{microtype}
\usepackage[numbers]{natbib}
\usepackage{xcolor}
\usepackage[
  colorlinks=true,
  linkcolor=blue!45!black,
  citecolor=blue!45!black,
  urlcolor=blue!45!black
]{hyperref}

\allowdisplaybreaks

\newtheorem{theorem}{Theorem}[section]
\newtheorem{lemma}[theorem]{Lemma}
\newtheorem{proposition}[theorem]{Proposition}
\newtheorem{corollary}[theorem]{Corollary}
\theoremstyle{definition}

\newtheorem{assumption}[theorem]{Assumption}
\theoremstyle{remark}
\newtheorem{remark}[theorem]{Remark}

\newcommand{\RR}{\mathbb{R}}

\newcommand{\EE}{\mathbb{E}}
\newcommand{\PP}{\mathbb{P}}

\newcommand{\HH}{\mathcal{H}}

\newcommand{\FF}{\mathcal{F}}

\newcommand{\LL}{\mathcal{L}}

\newcommand{\ind}{\mathbf 1}
\newcommand{\norm}[1]{\left\lVert#1\right\rVert}
\newcommand{\abs}[1]{\left\lvert#1\right\rvert}
\newcommand{\ip}[2]{\left\langle#1,#2\right\rangle}
\newcommand{\one}{\mathbf{1}}
\newcommand{\dd}{\,\mathrm{d}}

\DeclareMathOperator{\op}{op}

\DeclareMathOperator{\tr}{tr}

\DeclareMathOperator{\Span}{span}
\DeclareMathOperator{\supp}{supp}
\DeclareMathOperator*{\argmin}{arg\,min}

\title{Kernel Ridge Regression over H\"older Classes:\\Function and Derivative Losses}
\author{
  Yuxuan Hou\\
  {\small Qiuzhen College}\\
  {\small Tsinghua University}
  \and
  Yicheng Li\\
  {\small School of Statistics}\\
  {\small East China Normal University}
}
\date{}
\hypersetup{
  pdftitle={Kernel Ridge Regression over H\"older Classes: Function and Derivative Losses},
  pdfauthor={Yuxuan Hou and Yicheng Li}
}

\begin{document}

\maketitle

\begin{abstract}
We study ordinary kernel ridge regression over full H\"older--Zygmund
balls, measuring errors in Sobolev norms and in pointwise and uniform
losses for mixed derivatives.  For targets of smoothness $s$, every
RKHS norm equivalent to $H^{s+d/2}$ gives the optimal $H^r$ risk for
$0\le r<s$; integer choices of $r$ include integrated squared derivative
losses.  Pointwise and uniform guarantees require additional kernel
structure.  A localized wavelet kernel attains the sharp derivative
rates with the same regularization orders as for function estimation,
whereas equivalent changes of the RKHS norm can destroy the population
bias rates.  For the ordinary Gaussian kernel with a shrinking bandwidth,
we prove a uniform derivative risk bound with the optimal polynomial
rate and an explicit, unoptimized logarithmic factor.  The results allow
bounded measurable design densities bounded away from zero and
conditionally sub-Gaussian noise.
\end{abstract}

\section{Introduction}
\label{sec:introduction}
\subsection{Motivation}
Nonparametric regression is one of the most important fields in statistical learning, and minimax rates of convergence are central to its theory~\citep{stone1977consistent,stone1980optimal,yang1999information}. Two of the most commonly considered types of function classes are reproducing kernel Hilbert space (RKHS) classes and H\"older classes. RKHS classes (e.g., Sobolev classes) are defined by global norms. In this setting, minimax-optimal estimators are usually obtained using penalized empirical risk minimization (e.g., kernel ridge regression)~\citep{caponnetto2007optimal,zhang2024optimality}, and the minimax rates differ between global loss functions (e.g., $L^2$ loss) and local loss functions (e.g., $L^\infty$ and pointwise losses)~\citep{tuo2024asymptotic}. H\"older spaces are defined by local smoothness conditions. In this setting, minimax-optimal estimators are usually based on local polynomials~\citep{stone1982optimal,chhor2022benign,Tsybakov2009}, and the minimax rates of convergence are similar across different loss functions.

Local polynomial methods estimate functions and derivatives by fitting local
polynomials~\citep{fan1995local}.  Kernel ridge regression offers another
construction: fit the function once and differentiate the resulting kernel
expansion.  Existing prediction guarantees for Gaussian and Sobolev kernels
\citep{hamm2021adaptive,steinwart2009optimal} motivate asking which derivative
losses this construction controls.  Integrated squared derivative errors
are included in Sobolev norms.  Pointwise and uniform derivative errors
require a separate analysis of evaluation functionals and spatial localization.

In this paper, we investigate the kernel ridge regression (KRR) estimator with Sobolev and Gaussian kernels under H\"older source conditions. We study whether the KRR estimator can achieve minimax-optimal rates of convergence for estimating the target function itself and its derivatives under $L^2$, $L^\infty$, and pointwise losses.
\subsection{Contributions}
For a target of smoothness $s$, fix a mixed derivative of order $m<s$.
Our results concern the full H\"older--Zygmund ball, including integer $s$:
\begin{itemize}
    \item Every fixed RKHS norm equivalent to $H^{s+d/2}$ gives the
    optimal squared $H^r$ risk, of order
    \[
      n^{-2(s-r)/(2s+d)},\qquad 0\le r<s.
    \]
    Taking $r=m$ directly yields the derivative $L^2$ result.
    \item For an exactly diagonal localized wavelet kernel, squared
    derivative pointwise and uniform risks have orders
    $n^{-2(s-m)/(2s+d)}$ and $(\log n/n)^{2(s-m)/(2s+d)}$.
    The corresponding regularization orders, $n^{-1}$ and
    $(\log n)/n$, do not depend on $m$.  Population counterexamples
    show why Sobolev norm equivalence alone does not give these local
    bias bounds.
    \item For the ordinary Gaussian kernel, bandwidth
    $h_n=n^{-1/(2s+d)}$ and regularization $\lambda_n=n^{-b}$ with fixed
    $b\ge1$ give squared uniform derivative risk at most
    \[
      C n^{-2(s-m)/(2s+d)}\{\log(en)\}^{7d/2+4m}.
    \]
    This proves the optimal polynomial rate with an explicit logarithmic
    loss; no optimality is claimed for the logarithmic exponent.
\end{itemize}
\subsection{Related Works}
\paragraph{Minimax estimation over H\"older classes.}
\citet{stone1980optimal,stone1982optimal} establish optimal rates for estimating smooth functions and their derivatives under pointwise and global losses. Local polynomial constructions attain these rates under suitable assumptions, providing benchmarks for kernel estimators.

\paragraph{Kernel and spectral learning.}
\citet{steinwart2009optimal} analyze clipped regularized least squares under real-interpolation conditions, while \citet{zhang2024optimality} establish optimal spectral learning rates under Hilbert source conditions. \citet{fischer2020sobolev} provide Sobolev-norm guarantees; their general error bound, combined with real-interpolation bias estimates, also yields derivative $L^2$ rates over H\"older classes. Gaussian-kernel prediction has been studied with shrinking bandwidths~\citep{hamm2021adaptive} and fixed bandwidths under additional eigenfunction conditions~\citep{lin2025fixed}.

\paragraph{Derivative estimation and local risks.}
\citet{liu2023derivatives} derive general uniform bounds and nearly optimal $L^2$ rates for KRR derivatives in Fourier settings, with GP extensions in \citet{liu2026optimal}. Related local guarantees are available for GP means~\citep{yang2017frequentist}, Sobolev KRR with targets in the fitted RKHS~\citep{tuo2024asymptotic}, and Gaussian priors on spaces of increasing dimension~\citep{yoo2016supremum}. Our focus is ordinary KRR over full H\"older--Zygmund balls, distinguishing Sobolev-risk guarantees from pointwise and uniform guarantees that require additional kernel structure.

\section{Problem formulation}
\label{sec:problem-formulation}

Let $\Omega=[0,1]^d$.  We observe independent and identically distributed pairs $(X_i,Y_i)_{i=1}^n$ satisfying
\begin{equation}
  Y_i=f^\ast(X_i)+\epsilon_i,
  \qquad
  \EE[\epsilon_i\mid X_i]=0.
  \label{eq:regression-model}
\end{equation}
The following assumption about design and noise is common in the literature of nonparametric regression.

\begin{assumption}[Design and noise]
\label{ass:design-noise}
There are constants $0<p_-\le p_+<\infty$ and $\sigma>0$ such that
\[
  p_-\le p(x)\le p_+
  \quad\text{for almost every }x\in\Omega,
\]
and
\[
  \EE[\exp(t\epsilon_i)\mid X_i]
  \le \exp(\sigma^2t^2/2),
  \qquad t\in\RR.
\]
\end{assumption}

For $s>0$ and $F>0$, let
\[
  \FF_s(F)
  :=\{f\in B^s_{\infty,\infty}(\Omega):
       \norm{f}_{B^s_{\infty,\infty}}\le F\}.
\]
For noninteger $s$, this is the usual H\"older ball up to equivalent norms; at integer $s$ it is the corresponding Zygmund ball.

Let $\mathcal{P}(s,F,\sigma,p_-,p_+)$ be the collection of joint laws $P$
of $(X,Y)$ under which $X$ has a density $p_X$ satisfying the bounds in
Assumption~\ref{ass:design-noise},
$Y=f_P^\ast(X)+\epsilon$ with $f_P^\ast\in\FF_s(F)$, and the conditional
mean and sub-Gaussian requirements of that assumption hold for $\epsilon$.

We use the restriction-space definition of $B^s_{\infty,\infty}(\Omega)$,
so each member of $\FF_s(F)$ has an extension to $\RR^d$ with Besov norm
at most $2F$.  Fix a multi-index $\alpha\in\mathbb N_0^d$ of order
$m:=|\alpha|<s$.  Its classical mixed derivative $D^\alpha f$ has a
continuous representative on the closed cube.  The case $m=0$ includes
estimation of the function itself.  All loss functions below are squared;
the KRR training objective remains ordinary squared error on the responses.

For $0\le r<s$, the Sobolev loss is
\[
  \LL_{H^r}(\widetilde f,f)
  :=\norm{\widetilde f-f}_{H^r(\Omega)}^2.
\]
For an integer $m<s$, Sobolev norm equivalence gives
\begin{equation}
  \LL_{H^m}(\widetilde f,f)
  \asymp\sum_{|\beta|\le m}
    \norm{D^\beta(\widetilde f-f)}_{L^2(\Omega)}^2.
  \label{eq:derivative-sobolev-inclusion}
\end{equation}
Thus the integrated derivative loss
$\LL_{\alpha,2}(\widetilde f,f):=
\norm{D^\alpha(\widetilde f-f)}_{L^2(\Omega)}^2$
is controlled by $\LL_{H^m}$.  In particular, $H^0=L^2$.
The density bounds make the corresponding $L^2(P_X)$ losses equivalent.

For a fixed interior point $x_0\in(0,1)^d$, the derivative pointwise loss is
\[
  \LL_{\alpha,x_0}(\widetilde f,f)
  :=\abs{D^\alpha\widetilde f(x_0)-D^\alpha f(x_0)}^2,
\]
and the derivative $L^\infty$ loss is
\[
  \LL_{\alpha,\infty}(\widetilde f,f)
  :=\norm{D^\alpha\widetilde f-D^\alpha f}_\infty^2.
\]
For any loss $\LL$, write
\begin{equation}
  \mathfrak R_n\bigl(\mathcal{P}(s,F,\sigma,p_-,p_+),\LL\bigr)
  :=\inf_{\widetilde f}
    \sup_{P\in \mathcal{P}(s,F,\sigma,p_-,p_+)}
    \EE_{P}\LL(\widetilde f,f_P^\ast)
  \label{eq:minimax-risk}
\end{equation}
for the minimax risk, where the infimum is over all measurable estimators.

The classical derivative-estimation benchmarks
\citep{stone1980optimal,stone1982optimal,GineNickl2016} are
\begin{equation}
\begin{aligned}
  \mathfrak R_n\bigl(\mathcal{P}(s,F,\sigma,p_-,p_+),\LL_{H^r}\bigr)
  &\asymp n^{-2(s-r)/(2s+d)},\\
  \mathfrak R_n\bigl(\mathcal{P}(s,F,\sigma,p_-,p_+),\LL_{\alpha,2}\bigr)
  &\asymp n^{-2(s-m)/(2s+d)},\\
  \mathfrak R_n\bigl(\mathcal{P}(s,F,\sigma,p_-,p_+),\LL_{\alpha,x_0}\bigr)
  &\asymp n^{-2(s-m)/(2s+d)},\\
  \mathfrak R_n\bigl(\mathcal{P}(s,F,\sigma,p_-,p_+),\LL_{\alpha,\infty}\bigr)
  &\asymp \left(\frac{\log n}{n}\right)^{2(s-m)/(2s+d)}.
\end{aligned}
\label{eq:benchmark-rates}
\end{equation}
We use these as benchmarks and focus on KRR upper bounds.  Smooth bump
alternatives already give the derivative lower bounds, so these benchmarks
also apply to integer-smoothness Zygmund balls.  The matching Sobolev and
localized-kernel upper bounds are Theorems~\ref{thm:sobolev-risk-optimality}
and \ref{thm:localized-minimax}.  The Gaussian result in
Theorem~\ref{thm:gaussian-derivatives} attains the optimal power of $n$
with an additional logarithmic factor.

\section{Minimax optimality under Sobolev loss}
\label{sec:sobolev-minimax}

Let $K$ be a continuous positive definite Mercer kernel whose RKHS $\HH$
is norm-equivalent to $H^{s+d/2}(\Omega)$, that is,
\begin{equation}
  \HH=H^{s+d/2}(\Omega)\text{ as set and }
  c_\HH\norm{g}_{H^{s+d/2}}
  \le \norm{g}_\HH
  \le C_\HH\norm{g}_{H^{s+d/2}},
 \quad \forall g\in\HH.
  \label{eq:rkhs-sobolev-equivalence}
\end{equation}

The kernel ridge regression(KRR) estimator is
\begin{equation}
  \widehat f_\lambda
  :=\argmin_{f\in\HH}
  \left\{
    \frac1n\sum_{i=1}^n(Y_i-f(X_i))^2
    +\lambda\norm{f}_\HH^2
  \right\}.
  \label{eq:krr-estimator}
\end{equation}
The following result shows that the KRR algorithm attains the minimax rate on H\"older class.
\begin{theorem}[Sobolev-risk optimality for every equivalent RKHS]
\label{thm:sobolev-risk-optimality}
Let $s>0$ and $0\le r<s$.  Suppose Assumption~\ref{ass:design-noise} and \eqref{eq:rkhs-sobolev-equivalence} hold.  With $\lambda_n\asymp n^{-1}$, KRR satisfies
\begin{equation}
  \sup_{P\in\mathcal P(s,F,\sigma,p_-,p_+)}
  \EE_P\norm{\widehat f_{\lambda_n}-f_P^\ast}_{H^r}^2
  \le C n^{-2(s-r)/(2s+d)}.
  \label{eq:sobolev-risk-upper}
\end{equation}
The constant $C$ is independent of $n$ and uniform over the stated
distribution class; it may depend on the fixed RKHS norm-equivalence constants.
\end{theorem}
\begin{corollary}[Integrated squared derivative loss]
\label{cor:derivative-l2}
Under the assumptions of Theorem~\ref{thm:sobolev-risk-optimality}, for
every fixed multi-index $\alpha$ with $m=|\alpha|<s$ and
$\lambda_n\asymp n^{-1}$,
\begin{equation}
  \sup_{P\in\mathcal P(s,F,\sigma,p_-,p_+)}
  \EE_P\norm{D^\alpha\widehat f_{\lambda_n}
        -D^\alpha f_P^\ast}_{L^2(\Omega)}^2
  \le C n^{-2(s-m)/(2s+d)}.
  \label{eq:derivative-l2-upper}
\end{equation}
The same conclusion holds for the sum over all $|\alpha|\le m$ and
for $L^2(P_X)$ in place of $L^2(\Omega)$.
\end{corollary}
\begin{proof}
Take $r=m$ in Theorem~\ref{thm:sobolev-risk-optimality} and use
\eqref{eq:derivative-sobolev-inclusion} and the density bounds.
\end{proof}
\begin{remark}
The invariance under equivalent RKHS norms is specific to Sobolev losses.
Section~\ref{sec:population-bias-counterexamples} shows that a universal
endpoint analogue based only on norm equivalence is false.
\end{remark}

The proof is given in Appendix~\ref{app:sobolev-minimax}. 

\section{Wavelet and kernel setup}
\label{sec:problem-setup}

For the localized kernel results, we fix a continuous, compactly supported,
boundary-corrected tensor-product CDV orthonormal wavelet system on
$\Omega=[0,1]^d$, and put $\nu=2s+d$ and $R=\nu/2=s+d/2$.  In the detailed MRA notation of
Assumption~\ref{cdv:ass:cdv}, let $\Pi_0$ be the projection onto the
coarse scaling space and let $Q_j$ be the projection onto $W_j$.  We use
the exact grouping
\begin{equation}
  P_0:=\Pi_0+Q_0,
  \qquad
  P_j:=Q_j\quad(j\ge1),
  \qquad
  \sum_{j\ge0}2^{-\nu j}P_j
  =\Pi_0+\sum_{j\ge0}2^{-\nu j}Q_j.
  \label{eq:coarse-block-identification}
\end{equation}
Thus the unscaled operator used in the main text is exactly the one used
in Appendix~\ref{app:uniform-cdv-resolvent}.  Regrouping any other fixed
finite number of initial scales merely replaces this fixed finite-dimensional
low-frequency block by an equivalent one and changes only constants.  With
the convention in~\eqref{eq:coarse-block-identification}, we absorb the
scaling functions and the level-zero wavelets into level $0$ and write
\[
  \{\psi_{j,k}:j\ge0,\ k\in\Lambda_j\},
  \qquad
  P_jf:=\sum_{k\in\Lambda_j}\ip{f}{\psi_{j,k}}_{L^2}\psi_{j,k},
  \qquad
  E_j:=\sum_{\ell=0}^jP_\ell.
\]
The following assumptions on this fixed wavelet system make precise the locality used below.  They are satisfied by sufficiently regular Cohen--Daubechies--Vial wavelets \citep{CohenDaubechiesVial1993,GineNickl2016}.

\begin{assumption}[Localized wavelet system]
\label{ass:localized-wavelet}
In addition to the consequences displayed below, the fixed system satisfies
the full CDV$(R)$ structure of Assumption~\ref{cdv:ass:cdv}, evaluated at
$R=\nu/2=s+d/2$: an orthogonal MRA, scale-independent finite interior and
boundary type sets, anchored support localization, the two-way cross-scale
support counts, finite refinement and boundary filters with uniform
fixed-depth bounds, and the endpoint Sobolev coefficient characterization.
The constant function belongs to the coarse space, so
$E_j\one=\one$ for every $j\ge0$.  For constants depending only on this
fixed wavelet system,
\[
  c2^{jd}\le|\Lambda_j|\le C2^{jd},
  \qquad
  \sup_x\sum_{k\in\Lambda_j}\psi_{j,k}(x)^2\le C2^{jd},
  \qquad
  \sup_{k\in\Lambda_j}\norm{\psi_{j,k}}_\infty\le C2^{jd/2},
\]
the wavelets have uniformly finite overlap, and $|\supp(\psi_{j,k})|\le C2^{-jd}$.  Moreover, the kernel $e_j$ of $E_j$ satisfies the cumulative finite propagation bound
\[
  \abs{e_j(x,y)}
  \le C2^{jd}\one\{\norm{x-y}\le C2^{-j}\}.
\]
Finally, the same system characterizes the spaces used below:
\begin{align}
  \norm{f}_{B^s_{\infty,\infty}}
  &\asymp\sup_{j\ge0}2^{js}\norm{P_jf}_\infty,
  \label{eq:besov-wavelet-norm}\\
  \norm{f}_{H^t}^2
  &\asymp\sum_{j\ge0}2^{2tj}\sum_{k\in\Lambda_j}
    \abs{\ip{f}{\psi_{j,k}}_{L^2}}^2,
  \qquad 0\le t\le\nu/2.
  \label{eq:sobolev-wavelet-norm}
\end{align}
In particular, every element of $B^s_{\infty,\infty}$ is represented by its canonical continuous wavelet reconstruction.  The regularity and the number of vanishing moments needed for both norm characterizations, as well as all constants in this assumption, are fixed independently of $n$, $\lambda$, $p$, and $f^\ast$.
\end{assumption}

For the positive pointwise and uniform results, take the exact kernel that is diagonal in this wavelet system
\begin{equation}
  K_\nu(x,y)
  :=c_K\sum_{j=0}^\infty2^{-\nu j}
    \sum_{k\in\Lambda_j}\psi_{j,k}(x)\psi_{j,k}(y),
  \qquad c_K>0,
  \qquad \nu=2s+d,
  \label{eq:localized-wavelet-kernel}
\end{equation}
Thus the low-frequency part is precisely the finite block $c_KP_0$.  Assumption~\ref{ass:localized-wavelet} and $\nu>d$ imply uniform convergence of the series and
\[
  \HH_\nu
  =\left\{f:\sum_{j,k}c_K^{-1}2^{\nu j}
    \abs{\ip{f}{\psi_{j,k}}_{L^2}}^2<\infty\right\}
  \simeq H^{\nu/2}(\Omega)=H^{s+d/2}(\Omega).
\]
Put
\[
  T_0:=\sum_{j\ge0}2^{-\nu j}P_j,
  \qquad
  A:=T_0^{-1}=\sum_{j\ge0}2^{\nu j}P_j,
\]
with $A$ understood through its closed quadratic form, whose form domain is
$\operatorname{Dom}(A^{1/2})=\HH_\nu$.  Let $L_{K_\nu,p}$ denote the
integral operator on $L^2(P_X)$ given by
\[
  (L_{K_\nu,p}g)(x):=\int_\Omega K_\nu(x,y)g(y)p(y)\dd y.
\]
For $f^\ast\in L^2(P_X)$, define the population solution by
\begin{equation}
  f_\lambda:=(L_{K_\nu,p}+\lambda I)^{-1}L_{K_\nu,p}f^\ast.
  \label{eq:population-solution}
\end{equation}
Equivalently, in the reference $L^2(\Omega)$ space it is the unique weak solution of
\begin{equation}
  \left(M_p+\frac{\lambda}{c_K}A\right)f_\lambda=M_pf^\ast.
  \label{eq:localized-population-equation}
\end{equation}
The endpoint behavior of $f^\ast-f_\lambda$ distinguishes general Sobolev-equivalent kernels from the localized construction.
For the localized theorem, set
\begin{equation}
  c_p:=\max\{1,p_+\},
  \qquad
  \eta_\lambda:=\frac{\lambda}{c_Kc_p},
  \qquad
  h_\lambda:=\eta_\lambda^{1/\nu},
  \qquad
  q_h(t):=t+\log(e/h),\qquad 0<h\le1.
  \label{eq:localized-scale}
\end{equation}
We restrict to deterministic $0<\lambda\le c_Kc_p$, so $0<h_\lambda\le1$.  The fixed quantities $c_K$ and $c_p$ can be absorbed into implicit constants, but we retain them in $h_\lambda$ and $q_{h_\lambda}$.

\section{Population-bias counterexamples among Sobolev-equivalent kernels}
\label{sec:population-bias-counterexamples}

The preceding Sobolev-risk theorem is stable under arbitrary equivalent
changes of the RKHS norm.  Endpoint bias is not.  We give two complementary
examples.  On the cube, a bounded block rotation couples a spatially remote,
dense wavelet direction to point evaluation.  On the torus, even translation
invariance and Fourier diagonality do not suffice: there exist arbitrarily
small oscillatory relative perturbations of the Bessel spectrum that
destroy the optimal bias rate.

\begin{theorem}[Wavelet-rotation obstruction on the cube]
\label{thm:population-bias-counterexample}
Let $\Omega=[0,1]^d$, let $x_0$ be an interior point, and suppose $s>0$.
There is one continuous bounded positive definite kernel $K$ whose RKHS is
norm-equivalent to $H^{s+d/2}(\Omega)$, a sequence
$\lambda_L\downarrow0$ with
$h_L:=\lambda_L^{1/(2s+d)}\asymp2^{-L}$, and functions $f_L$ such that
\begin{equation}
  \sup_L\norm{f_L}_{B^s_{\infty,\infty}}<\infty,
  \qquad
  \abs{(I-S_{\lambda_L})f_L(x_0)}
  \gtrsim h_L^{s-d/2}
  \asymp2^{-L(s-d/2)},
  \label{eq:population-bias-counterexample}
\end{equation}
where, for normalized uniform measure on $\Omega$,
$T_Kg(x):=\int_\Omega K(x,y)g(y)\,\dd y$ and
$S_{\lambda,K}:=(T_K+\lambda I)^{-1}T_K$; in
\eqref{eq:population-bias-counterexample} we abbreviate $S_{\lambda,K}$ to
$S_\lambda$.  In particular,
\begin{equation}
  h_L^{-s}
  \norm{I-S_{\lambda_L}}_
  {B^s_{\infty,\infty}\to L^\infty}
  \longrightarrow\infty.
  \label{eq:population-optimal-rate-failure}
\end{equation}
If $s=d/2$, the worst-case bias stays bounded away from zero; if
$0<s<d/2$, it diverges.
\end{theorem}

\begin{proposition}[Stationary Fourier-diagonal obstruction]
\label{prop:fourier-population-bias-counterexample}
Let $s>0$, $x_0\in\mathbb T^d$, and let $K_0$ be the periodic Bessel kernel
with Fourier eigenvalues
\[
  \omega_k=(1+|k|^2)^{-(s+d/2)},
  \qquad k\in\mathbb Z^d.
\]
For every $\eta>0$, there is one real-valued, continuous,
translation-invariant positive definite kernel $K$, with Fourier
eigenvalues $\mu_k$, such that
\[
  \sup_{k\in\mathbb Z^d}
  \abs{\frac{\mu_k}{\omega_k}-1}\le\eta,
  \qquad
  \mathcal H_K\simeq H^{s+d/2}(\mathbb T^d),
\]
where $\mathbb T^d$ carries normalized Lebesgue measure.  Under uniform
design, let $T_K$ be the integral operator induced by $K$ and set
$S_{\lambda,K}:=(T_K+\lambda I)^{-1}T_K$.  There are
$\lambda_L\downarrow0$ and real trigonometric polynomials
$f_L$ in one fixed $B^s_{\infty,\infty}(\mathbb T^d)$ ball satisfying
\begin{equation}
  \abs{(I-S_{\lambda_L,K})f_L(x_0)}
  \ge c(s,d,\eta)
  h_L^{s-d/2},
  \qquad
  h_L:=\lambda_L^{1/(2s+d)}.
  \label{eq:periodic-population-bias-counterexample}
\end{equation}
Consequently, every relative neighborhood of the Bessel spectrum contains
a Sobolev-equivalent kernel for which the optimal $O(h_L^s)$ population
bias bound fails.
\end{proposition}

\begin{remark}[Scope of the counterexamples]
Both results concern one fixed kernel and a sequence of targets in one fixed
Besov ball.  Thus they are uniform-over-the-class, population-level
obstructions, not assertions that one fixed target has nonconvergent bias.
For every $s>0$ they rule out an endpoint bias theorem based only on
Sobolev norm equivalence.  Absolute nonconvergence follows from these
constructions only when $s\le d/2$; for $s>d/2$ the displayed lower bound
still tends to zero, but more slowly than the optimal $h_L^s$ rate.  These
population examples alone also do not rule out every empirical tuning
scheme.
\end{remark}

\begin{corollary}[Obstructions for derivative evaluation]
\label{cor:derivative-bias-obstruction}
Fix a multi-index $\alpha$ with $m=|\alpha|<s$.  Each of the two
constructions above admits a derivative version: on the cube, and
separately on the torus within every positive relative neighborhood of
the Bessel spectrum, there is one fixed Sobolev-equivalent kernel and
targets $f_L$ in one fixed $B^s_{\infty,\infty}$ ball such that
\begin{equation}
 |D^\alpha(I-S_{\lambda_L,K})f_L(x_0)|
 \ge c h_L^{s-m-d/2},
 \qquad h_L=\lambda_L^{1/(2s+d)}\longrightarrow0.
 \label{eq:derivative-bias-obstruction}
\end{equation}
Thus norm equivalence alone does not imply the optimal derivative bias
$O(h_L^{s-m})$.  The lower bound stays away from zero when $s-m=d/2$
and diverges when $s-m<d/2$.
\end{corollary}

These derivative obstructions retain the population and
uniform-over-the-target-class scope of the preceding remark.
Appendix~\ref{app:population-bias} gives all constructions.  It also shows
that the unperturbed radial Bessel kernel has the optimal $O(h_\lambda^s)$
bias, and $O(h_\lambda^{s-m})$ after $m$ derivatives; the periodic
failure is caused by loss of multiplier regularity within frequency blocks.

\section{A localized kernel with optimal derivative pointwise and uniform loss}
\label{sec:localized-kernel}

We now specialize to the exact wavelet kernel \eqref{eq:localized-wavelet-kernel}.  In contrast to the counterexample kernels of Section~\ref{sec:population-bias-counterexamples}, its diagonal scale weights and cumulative finite propagation yield a weighted resolvent that is uniformly localized.  The effective spatial scale is $h_\lambda$ from \eqref{eq:localized-scale}; retaining this normalized scale is essential in the logarithmic terms below.

\begin{theorem}[Derivative pointwise and uniform high-probability bounds]
\label{thm:localized-high-probability}
Suppose Assumptions~\ref{ass:design-noise} and \ref{ass:localized-wavelet} hold, and let $f^\ast\in\FF_s(F)$.  Fix a multi-index $\alpha$ with $m=|\alpha|<s$, an interior point $x_0\in(0,1)^d$, and a deterministic $0<\lambda\le c_Kc_p$ before observing the data.  If $t\ge1$ and
\[
  nh_\lambda^d\ge C_0q_{h_\lambda}(t),
\]
then, with probability at least $1-Ce^{-t}$,
\begin{align}
  \abs{D^\alpha\widehat f_\lambda(x_0)-D^\alpha f^\ast(x_0)}
  &\le C\left\{
    Fh_\lambda^{s-m}
    +\sigma\sqrt{\frac{t}{nh_\lambda^{d+2m}}}
  \right\},
  \label{eq:localized-pointwise-bound}\\
  \norm{D^\alpha\widehat f_\lambda-D^\alpha f^\ast}_\infty
  &\le C\left\{
    Fh_\lambda^{s-m}
    +\sigma\sqrt{
      \frac{q_{h_\lambda}(t)}
      {nh_\lambda^{d+2m}}
    }
  \right\}.
  \label{eq:localized-uniform-bound}
\end{align}
The constants may depend on the fixed derivative order and are uniform over
all measurable design densities with the common bounds in
Assumption~\ref{ass:design-noise}.  No derivatives of the density are required.
\end{theorem}

Balancing bias and stochastic error gives two different regularization scales.

\begin{theorem}[Derivative pointwise and $L^\infty$ minimax optimality]
\label{thm:localized-minimax}
Under the assumptions of Theorem~\ref{thm:localized-high-probability}, there are constants $a_{\mathrm{pt}},a_\infty>0$ and $n_0$ such that, for $n\ge n_0$, the localized wavelet KRR estimator satisfies
\begin{align}
  \sup_{P\in\mathcal P(s,F,\sigma,p_-,p_+)}
  \EE_P\abs{D^\alpha\widehat f_{\lambda_{\mathrm{pt}}}(x_0)-D^\alpha f_P^\ast(x_0)}^2
  &\le Cn^{-2(s-m)/(2s+d)},
  &\lambda_{\mathrm{pt}}&:=a_{\mathrm{pt}}n^{-1},
  \label{eq:pointwise-minimax-upper}\\
  \sup_{P\in\mathcal P(s,F,\sigma,p_-,p_+)}
  \EE_P\norm{D^\alpha\widehat f_{\lambda_\infty}-D^\alpha f_P^\ast}_\infty^2
  &\le C\left(\frac{\log n}{n}\right)^{2(s-m)/(2s+d)},
  &\lambda_\infty&:=a_\infty\frac{\log n}{n}.
  \label{eq:uniform-minimax-upper}
\end{align}
Both rates match the derivative benchmarks in \eqref{eq:benchmark-rates}.
\end{theorem}

The constants may depend on $F$, $\sigma$, $p_-$, $p_+$, $c_K$, $\alpha$,
and the fixed wavelet system, but not on $n$, $f^\ast$, $p$, or $x_0$.
The extra logarithm in the uniform loss reflects a spatial maximum effect
and cannot be obtained by taking the supremum of pointwise expected risks.
Appendix~\ref{app:localized-kernel} supplies the population and empirical
resolvent estimates and the zeroth-order case.
Appendix~\ref{app:derivative-losses} proves the derivative bounds and
the expected squared risks.  The tuning orders remain unchanged because
\[
  h^{2(s-m)}+\frac{1}{nh^{d+2m}}
  =h^{-2m}\left(h^{2s}+\frac{1}{nh^d}\right),
\]
with the analogous identity for the uniform variance term containing
$\log(e/h)$.  Exact optimal tuning constants may still depend on $m$.

\section{Gaussian KRR under derivative uniform loss}
\label{sec:gaussian-derivatives}

Consider the ordinary, nonperiodic Gaussian kernel
\begin{equation}
  k_h(x,y):=\exp\{-\norm{x-y}_2^2/h^2\},\qquad x,y\in\Omega,
  \label{eq:gaussian-kernel}
\end{equation}
and write $\HH_h$ for its RKHS.  The estimator
$\widehat f^G_{h,\lambda}$ minimizes the objective in
\eqref{eq:krr-estimator} with $\HH=\HH_h$.  Equivalently,
\begin{equation}
  \widehat f^G_{h,\lambda}(x)
  =k_{h,x}^{\top}(K_h+n\lambda I_n)^{-1}Y,
  \qquad
  D^\alpha\widehat f^G_{h,\lambda}(x)
  =(D^\alpha k_{h,x})^{\top}(K_h+n\lambda I_n)^{-1}Y.
  \label{eq:gaussian-derivative-estimator}
\end{equation}
Here $K_h=(k_h(X_i,X_j))_{i,j}$ and
$k_{h,x}=(k_h(x,X_i))_i$; derivatives act on $x$.
Thus derivatives are obtained from the same fitted function.

We take the full-dimensional specialization of the deterministic tuning
orders in \citet[Corollary~3.4]{hamm2021adaptive}:
\begin{equation}
  h_n=n^{-1/(2s+d)},\qquad
  \lambda_n=n^{-b},\qquad b\ge1\ \text{fixed}.
  \label{eq:gaussian-tuning}
\end{equation}
That paper gives a prediction-risk theorem with logarithmic factors
for a clipped estimator under its assumptions.  The derivative result
below is proved directly for the untruncated empirical KRR estimator
in Appendix~\ref{app:gaussian-derivatives}.

\begin{theorem}[Gaussian uniform derivative risk]
\label{thm:gaussian-derivatives}
Suppose Assumption~\ref{ass:design-noise} holds, let $s>0$, and fix
$\alpha\in\mathbb N_0^d$ with $m=|\alpha|<s$.  Under
\eqref{eq:gaussian-tuning}, for all sufficiently large $n$,
\begin{equation}
  \sup_{P\in\mathcal P(s,F,\sigma,p_-,p_+)}
  \left(\EE_P\norm{D^\alpha\widehat f^G_{h_n,\lambda_n}
                    -D^\alpha f_P^\ast}_\infty^2\right)^{1/2}
  \le C n^{-(s-m)/(2s+d)}\{\log(en)\}^{7d/4+2m}.
  \label{eq:gaussian-derivative-rms}
\end{equation}
The supremum norm is over the entire closed cube.  Constants may depend
on the fixed model parameters, the multi-index $\alpha$, and $b$.
They are uniform in $n$, the target, and the measurable design density.
\end{theorem}

In particular, the squared loss is bounded by
\begin{equation}
  C n^{-2(s-m)/(2s+d)}\{\log(en)\}^{7d/2+4m}.
  \label{eq:gaussian-derivative-squared}
\end{equation}
The same upper bound controls derivative pointwise and integrated
squared losses, since $\Omega$ has volume one.  For pointwise loss the
proof also gives the slightly smaller noise term without the spatial
maximum factor, as recorded below.  The Sobolev-equivalent-kernel
result of Corollary~\ref{cor:derivative-l2} has no logarithmic loss;
we do not infer that sharper conclusion for Gaussian KRR.

\begin{proposition}[Separate approximation, regularization, and noise terms]
\label{prop:gaussian-intermediate}
Write $L=\log(en)$, $h=h_n$, and $\lambda=\lambda_n$.
For any fixed $A>0$, the left side of
\eqref{eq:gaussian-derivative-rms} is at most
\begin{equation}
  C\left\{
     Fh^{s-m}L^{5d/4+2m}
     +F\sqrt{\frac{\lambda}{h^{d+2m}}}\,L^{7d/4+2m}
     +\frac{\sigma L^{d+2m+1/2}}{\sqrt{nh^{d+2m}}}
   \right\}
   +C_A(F+\sigma)n^{-A}.
  \label{eq:gaussian-intermediate-rms}
\end{equation}
For the root mean squared error at a fixed point, replace the last
numerator $\sigma L^{d+2m+1/2}$ by $\sigma L^{d+2m}$.
\end{proposition}

\begin{remark}[Logarithms and scope]
The factor $L^{2m}$ in \eqref{eq:gaussian-derivative-rms} comes from
a polynomial derivative evaluation bound on each design cell.  The
result has the optimal power of $n$, but does not establish the sharp
uniform rate $(\log n/n)^{2(s-m)/(2s+d)}$ or optimality of its own
logarithmic exponent.  It holds for the full H\"older--Zygmund ball,
including integer $s$, and uses only density bounds.
Both $h_n$ and $\lambda_n$ are unchanged across fixed derivative orders
$m<s$.  The proof applies to the original smooth KRR fit; clipping
the fitted function before differentiating is not part of this result.
\end{remark}

\appendix
\section{Proofs for Sobolev-risk optimality}
\label{app:sobolev-minimax}

Unless stated otherwise, all constants below are uniform over
$P\in\mathcal P(s,F,\sigma,p_-,p_+)$.

Throughout this appendix, set
\begin{equation}
  m:=s+\frac d2,
  \qquad
  \theta_\ast:=\frac{s}{m},
  \qquad
  \gamma:=\frac rm,
  \qquad
  \alpha_0:=\frac{d}{2m}=1-\theta_\ast.
  \label{eq:sobolev-scale-parameters}
\end{equation}
For this appendix, let $\iota:\HH\to L^2(P_X)$ be the canonical
inclusion, so that $T=\iota^\ast\iota$.  Define
\[
  g:=\iota^\ast f^\ast,
  \qquad
  g_Z:=\frac1n\sum_{i=1}^nY_i k_{X_i}.
\]
To distinguish the general Sobolev-equivalent kernel considered here
from the localized kernel used later, throughout this appendix the
population notation means
\begin{equation}
  f_\lambda:=(T+\lambda I)^{-1}g,
  \qquad
  \widehat f_\lambda=(T_X+\lambda I)^{-1}g_Z.
  \label{eq:sobolev-normal-equations}
\end{equation}
Let $L:=\iota\iota^\ast$ on $L^2(P_X)$ and let
$\{(\mu_j,e_j)\}_{j\ge1}$ be its eigensystem, with the eigenvalues in
nonincreasing order.  For $0\le\eta\le1$, define the power space
\[
  [\HH]^\eta
  :=\left\{
    \sum_{j\ge1}a_j\mu_j^{\eta/2}e_j:
    (a_j)_{j\ge1}\in\ell^2
  \right\},
  \qquad
  \norm{f}_{[\HH]^\eta}^2:=\sum_{j\ge1}a_j^2.
\]
Since $\HH=H^m(\Omega)$ is dense in $L^2(P_X)$, the eigenvectors of
$L$ are complete.  Thus $[\HH]^0=L^2(P_X)$ and $[\HH]^1$ is $\HH$
viewed as a function space.

\begin{lemma}[Consequences of Sobolev norm equivalence]
\label{lem:sobolev-power-scale}
For every $0\le\eta\le1$,
\begin{equation}
  [\HH]^\eta
  \simeq H^{m\eta}(\Omega).
  \label{eq:sobolev-power-identification}
\end{equation}
For $0<\eta<1$, this identification also satisfies
\begin{equation}
  [\HH]^\eta\simeq(L^2(P_X),\HH)_{\eta,2}.
  \label{eq:sobolev-interpolation-identification}
\end{equation}
Moreover,
\begin{equation}
  \mu_j\asymp j^{-2m/d}=j^{-1/\alpha_0},
  \qquad
  \mathcal N(\nu)
  :=\tr\!\left(L(L+\nu^{-1}I)^{-1}\right)
  =\tr\!\left(T(T+\nu^{-1}I)^{-1}\right)
  \lesssim \nu^{\alpha_0}.
  \label{eq:sobolev-effective-dimension}
\end{equation}
For every fixed $\alpha_0<\alpha\le1$, the embedding
$[\HH]^\alpha\hookrightarrow L^\infty(\Omega)$ is continuous.
Consequently, there is $M_\alpha<\infty$ such that
\begin{equation}
  \sup_{x\in\Omega}
  \norm{(T+\nu^{-1}I)^{-1/2}k_x}_{\HH}
  \le M_\alpha\nu^{\alpha/2},
  \qquad \nu\ge1.
  \label{eq:sobolev-embedding-index}
\end{equation}
\end{lemma}

\begin{proof}
The density bounds in Assumption~\ref{ass:design-noise} make
$L^2(P_X)$ and Lebesgue $L^2(\Omega)$ equivalent as normed function
spaces.  Together with \eqref{eq:rkhs-sobolev-equivalence}, equivalence
of interpolation couples gives
\[
  (L^2(P_X),\HH)_{\eta,2}
  \simeq (L^2(\Omega),H^m(\Omega))_{\eta,2}
  =H^{m\eta}(\Omega).
\]
For $0<\eta<1$, the spectral definition of the Hilbert interpolation
scale gives \eqref{eq:sobolev-interpolation-identification}; the
endpoint identities follow directly from the power-space definition.
This proves \eqref{eq:sobolev-power-identification} for all
$0\le\eta\le1$.

The wavelet characterization, using truncation for the upper bound and
a dimension argument on each resolution level for the lower bound,
shows that the approximation numbers of
$H^m(\Omega)\hookrightarrow L^2(\Omega)$ are of order $j^{-m/d}$.
Norm equivalence and the min--max principle therefore give
$\mu_j\asymp j^{-2m/d}$.  Splitting the sum defining
$\mathcal N(\nu)$ at $j\asymp\nu^{\alpha_0}$ yields
\[
  \mathcal N(\nu)
  \lesssim
  \nu^{\alpha_0}
  +\nu\sum_{j>\nu^{\alpha_0}}j^{-1/\alpha_0}
  \lesssim\nu^{\alpha_0},
\]
where $\alpha_0<1$.  Finally, if $\alpha>\alpha_0$, then
$m\alpha>d/2$, so \eqref{eq:sobolev-power-identification} and Sobolev
embedding imply $[\HH]^\alpha\hookrightarrow L^\infty$.  In the
eigenbasis of $L$, this says
$\sup_x\sum_j\mu_j^\alpha e_j(x)^2<\infty$.  Since
$t/(1+t)\le t^\alpha$ for $t\ge0$ and $0<\alpha\le1$, it also gives
\eqref{eq:sobolev-embedding-index}.
\end{proof}

\begin{lemma}[Endpoint population approximation]
\label{lem:sobolev-endpoint-bias}
Uniformly over $f^\ast\in\FF_s(F)$ and $0<\lambda\le1$,
\begin{equation}
  \norm{f^\ast-f_\lambda}_{H^r(\Omega)}
  \lesssim F\lambda^{(s-r)/(2s+d)}.
  \label{eq:sobolev-population-rate}
\end{equation}
\end{lemma}

\begin{proof}
For an integer $J\ge j_0$, let $g_J$ be the wavelet projection of
$f^\ast$ onto the scaling block and the levels $j_0\le j\le J$.
By \eqref{eq:besov-wavelet-norm} and
\eqref{eq:sobolev-wavelet-norm},
\begin{align}
  \norm{f^\ast-g_J}_{L^2(\Omega)}^2
  &\lesssim
  F^2\sum_{j>J}2^{jd}2^{-2j(s+d/2)}
  \lesssim F^2 2^{-2Js},
  \label{eq:sobolev-truncation-l2}\\
  \norm{g_J}_{H^m(\Omega)}^2
  &\lesssim
  F^2\left(1+\sum_{j=j_0}^J
    2^{2jm}2^{jd}2^{-2j(s+d/2)}\right)
  \lesssim F^2 2^{Jd}.
  \label{eq:sobolev-truncation-hm}
\end{align}
Choose $J=J_\lambda$ such that
$2^J\asymp\lambda^{-1/(2s+d)}$.  The population solution is the
unique minimizer over $\HH$ of
\[
  Q_\lambda(h)
  :=\norm{f^\ast-h}_{L^2(P_X)}^2
    +\lambda\norm{h}_{\HH}^2.
\]
The full variational inequality, the density upper bound, and
\eqref{eq:rkhs-sobolev-equivalence} give
\begin{align}
  \norm{f^\ast-f_\lambda}_{L^2(P_X)}^2
  +\lambda\norm{f_\lambda}_{\HH}^2
  &=Q_\lambda(f_\lambda)
  \le Q_\lambda(g_J)\notag\\
  &\lesssim
  F^2\left(2^{-2Js}+\lambda2^{Jd}\right)
  \lesssim F^2\lambda^{2s/(2s+d)}.
  \label{eq:sobolev-population-objective}
\end{align}
Consequently, using also the density lower bound,
\begin{equation}
  \norm{f^\ast-f_\lambda}_{L^2(\Omega)}
  \lesssim F\lambda^{s/(2s+d)},
  \qquad
  \norm{f_\lambda}_{H^m(\Omega)}
  \lesssim F\lambda^{-d/\{2(2s+d)\}}.
  \label{eq:sobolev-population-endpoints}
\end{equation}
The same $H^m$ bound holds for $g_J$ by
\eqref{eq:sobolev-truncation-hm}.  Hence, for
$h_\lambda:=f_\lambda-g_J$,
\[
  \norm{h_\lambda}_{L^2(\Omega)}
  \lesssim F\lambda^{s/(2s+d)},
  \qquad
  \norm{h_\lambda}_{H^m(\Omega)}
  \lesssim F\lambda^{-d/\{2(2s+d)\}}.
\]
Sobolev interpolation, applied to $h_\lambda$ rather than directly to
$f^\ast-f_\lambda$, now gives
\begin{align}
  \norm{h_\lambda}_{H^r(\Omega)}
  &\lesssim
  \norm{h_\lambda}_{L^2(\Omega)}^{1-r/m}
  \norm{h_\lambda}_{H^m(\Omega)}^{r/m}\notag\\
  &\lesssim
  F\lambda^{
    \frac{s}{2s+d}(1-r/m)
    -\frac{dr}{2m(2s+d)}}
  =F\lambda^{(s-r)/(2s+d)}.
  \label{eq:sobolev-population-interpolation}
\end{align}
On the other hand, since $r<s$,
\begin{equation}
  \norm{f^\ast-g_J}_{H^r(\Omega)}^2
  \lesssim
  F^2\sum_{j>J}2^{-2j(s-r)}
  \lesssim F^2\lambda^{2(s-r)/(2s+d)}.
  \label{eq:sobolev-truncation-hr}
\end{equation}
Combining \eqref{eq:sobolev-population-interpolation} and
\eqref{eq:sobolev-truncation-hr} proves the claim.
\end{proof}

\begin{lemma}[Sample error at the endpoint rate]
\label{lem:sobolev-sample-error}
Let $\lambda_n\asymp n^{-1}$.  Uniformly over
$f^\ast\in\FF_s(F)$,
\begin{equation}
  \EE\norm{\widehat f_{\lambda_n}-f_{\lambda_n}}_{H^r(\Omega)}^2
  \lesssim n^{-2(s-r)/(2s+d)}.
  \label{eq:sobolev-sample-mean-square}
\end{equation}
\end{lemma}

\begin{proof}
The source exponent used in this proof is only an auxiliary parameter.
Fix
\begin{equation}
  \varepsilon:=\frac{s-r}{2},
  \qquad
  q:=\frac{s-\varepsilon}{m}
     =\frac{\theta_\ast+\gamma}{2},
  \qquad
  \alpha:=\alpha_0+\frac q2.
  \label{eq:sobolev-fixed-source-parameters}
\end{equation}
Then $\gamma<q<\theta_\ast$ and
$\alpha_0<\alpha<\alpha_0+q<1$.  The wavelet characterizations give,
uniformly over the H\"older--Zygmund ball,
\begin{equation}
  \norm{f^\ast}_{H^{s-\varepsilon}}^2
  \lesssim F^2\sum_{j\ge j_0}2^{-2j\varepsilon}
  \lesssim F^2(1+\varepsilon^{-1}),
  \qquad
  \norm{f^\ast}_{L^\infty}\lesssim F.
  \label{eq:sobolev-fixed-source-embedding}
\end{equation}
Thus Lemma~\ref{lem:sobolev-power-scale} implies
$f^\ast\in[\HH]^q$ with a source radius that is uniform over
$\FF_s(F)$ because $\varepsilon$ is fixed.

Put $\nu:=\lambda_n^{-1}\asymp n$,
$T_\nu:=T+\nu^{-1}I$, and
$T_{X,\nu}:=T_X+\nu^{-1}I$.  Since
$T_\nu f_{\lambda_n}=g$, the KRR normal equations give the exact
centering identity
\begin{equation}
  \widehat f_{\lambda_n}-f_{\lambda_n}
  =T_{X,\nu}^{-1}
  \left\{(g_Z-T_Xf_{\lambda_n})-(g-Tf_{\lambda_n})\right\}.
  \label{eq:sobolev-krr-centering}
\end{equation}
The conditional sub-Gaussian assumption implies the Bernstein moment
condition used by \citep{zhang2024optimality}.  Lemma
\ref{lem:sobolev-power-scale} supplies their eigenvalue-decay and
embedding assumptions, and the ridge filter satisfies their filter
conditions with qualification one.  Equation
\eqref{eq:sobolev-fixed-source-embedding}
allows us to take $t=\norm{f^\ast}_{L^\infty}\lesssim F$ and hence
$\Omega_0=\Omega$ up to a null set in their truncation argument.
Applying \citep[Theorem~13]{zhang2024optimality} with source exponent
$q$ therefore gives, for
$\delta\in(0,1)$,
\begin{align}
  &\norm{T_\nu^{-1/2}
    \left\{(g_Z-T_Xf_{\lambda_n})-(g-Tf_{\lambda_n})\right\}}_{\HH}
  \notag\\
  &\quad\lesssim
  \log(2/\delta)\left\{
    \frac{\nu^{\alpha-q/2}}{n}
    +\frac{\nu^{\alpha/2}}{n}
    +\frac{\mathcal N(\nu)^{1/2}}{\sqrt n}
    +\frac{\nu^{(\alpha-q)/2}}{\sqrt n}
  \right\}
  \label{eq:sobolev-theorem13-specialized}
\end{align}
with probability at least $1-\delta$.  Constants depending on the
fixed source radius, the noise parameters, and $F$ are suppressed.
At $\nu\asymp n$, the four terms in braces have respective orders
\begin{equation}
  n^{-\theta_\ast},
  \qquad
  n^{-1+\alpha/2},
  \qquad
  n^{-\theta_\ast/2},
  \qquad
  n^{-\theta_\ast/2-q/4}.
  \label{eq:sobolev-theorem13-orders}
\end{equation}
Because $\alpha<1$ and $0<\theta_\ast<1$, the second term is also
$O(n^{-\theta_\ast/2})$.  Thus the effective-dimension term, rather
than the auxiliary source exponent, determines the endpoint order.

We next record the confidence range needed below.  The covariance
condition in \citep[Lemma~12]{zhang2024optimality} is, up to fixed
constants,
\begin{equation}
  \frac{\nu^\alpha}{n}
  \log\!\left(\frac{C\mathcal N(\nu)}{\delta}\right)
  \le c_0.
  \label{eq:sobolev-covariance-condition}
\end{equation}
Let $b:=1-\alpha>0$, take $\delta=e^{-u}$, and choose a sufficiently
small constant $c>0$.  Since $\nu\asymp n$ and
$\mathcal N(\nu)\lesssim n^{\alpha_0}$,
\eqref{eq:sobolev-covariance-condition} holds uniformly for
\begin{equation}
  1\le u\le U_n:=c n^b
  \label{eq:sobolev-confidence-range}
\end{equation}
when $n$ is sufficiently large.  On the corresponding covariance
event,
\begin{equation}
  \norm{T_\nu^{1/2}T_{X,\nu}^{-1/2}}_{\op}
  =\norm{T_{X,\nu}^{-1/2}T_\nu^{1/2}}_{\op}
  \lesssim1.
  \label{eq:sobolev-covariance-stability}
\end{equation}
For $h\in\HH$, the spectral definition of $[\HH]^\gamma$ gives
\begin{equation}
  \norm{h}_{[\HH]^\gamma}
  \le \nu^{\gamma/2}\norm{T_\nu^{1/2}h}_{\HH}.
  \label{eq:sobolev-power-conversion}
\end{equation}
Combining \eqref{eq:sobolev-krr-centering},
\eqref{eq:sobolev-theorem13-specialized},
\eqref{eq:sobolev-covariance-stability}, and
\eqref{eq:sobolev-power-conversion}, and then using
\eqref{eq:sobolev-power-identification} and a union bound for the two
concentration events, shows that, with
\[
  a_n:=n^{-(\theta_\ast-\gamma)/2}
      =n^{-(s-r)/(2s+d)},
\]
there are constants $c,C>0$ such that
\begin{equation}
  \PP\!\left(
    \norm{\widehat f_{\lambda_n}-f_{\lambda_n}}_{H^r}
    >Cua_n
  \right)
  \le Ce^{-u},
  \qquad 1\le u\le U_n.
  \label{eq:sobolev-sample-tail}
\end{equation}

It remains to integrate this tail without introducing a logarithmic
loss.  Write
$Z_n:=\norm{\widehat f_{\lambda_n}-f_{\lambda_n}}_{H^r}$.
Integration of \eqref{eq:sobolev-sample-tail} up to $Ca_nU_n$ gives
\begin{equation}
  \EE\!\left[Z_n^2\one_{\{Z_n\le Ca_nU_n\}}\right]
  \lesssim
  a_n^2\left(1+\int_1^{U_n}u e^{-u}\dd u\right)
  \lesssim a_n^2.
  \label{eq:sobolev-good-event-expectation}
\end{equation}
For the remaining event, comparison of the empirical KRR objective
with the zero function gives
\begin{equation}
  \norm{\widehat f_{\lambda_n}}_{\HH}^2
  \le\frac{1}{n\lambda_n}\sum_{i=1}^nY_i^2
  \lesssim\sum_{i=1}^nY_i^2,
  \label{eq:sobolev-krr-crude-bound}
\end{equation}
while population optimality gives
$\norm{f_{\lambda_n}}_{\HH}^2\lesssim nF^2$.
The embedding $\HH\simeq H^m\hookrightarrow H^r$ controls the
$H^r$ norms of the two estimators.  Separately,
$\FF_s(F)\hookrightarrow L^\infty$ and conditional sub-Gaussianity
give a fourth moment for $Y_i$ that is uniform over the target ball.
It follows that
\begin{equation}
  \sup_{f^\ast\in\FF_s(F)}\EE Z_n^4\lesssim n^2.
  \label{eq:sobolev-sample-fourth-moment}
\end{equation}
Using \eqref{eq:sobolev-sample-tail} at $u=U_n$ and
Cauchy--Schwarz,
\begin{align}
  \EE\!\left[Z_n^2\one_{\{Z_n>Ca_nU_n\}}\right]
  &\le (\EE Z_n^4)^{1/2}
    \PP(Z_n>Ca_nU_n)^{1/2}\notag\\
  &\lesssim n e^{-c n^b}
  =o(a_n^2).
  \label{eq:sobolev-bad-event-expectation}
\end{align}
Equations \eqref{eq:sobolev-good-event-expectation} and
\eqref{eq:sobolev-bad-event-expectation} prove
\eqref{eq:sobolev-sample-mean-square}.
\end{proof}

\begin{remark}[Why no endpoint limit is taken]
The source bound in \eqref{eq:sobolev-fixed-source-embedding}
deteriorates as $\varepsilon^{-1/2}$ when
$\varepsilon\downarrow0$.  Consequently, a family of bounds with constants depending on
$\varepsilon$ cannot be sent directly to the endpoint.  No such limit
is needed here: Lemma~\ref{lem:sobolev-endpoint-bias} supplies the exact
population rate, while the effective-dimension term in
\eqref{eq:sobolev-theorem13-orders} supplies the exact sample rate for
the single fixed choice \eqref{eq:sobolev-fixed-source-parameters}.
\end{remark}

\begin{proof}[Proof of Theorem~\ref{thm:sobolev-risk-optimality}]
Fix $P\in\mathcal P(s,F,\sigma,p_-,p_+)$.  By the triangle inequality,
Lemmas~\ref{lem:sobolev-endpoint-bias} and
\ref{lem:sobolev-sample-error}, and $\lambda_n\asymp n^{-1}$,
\[
  \EE_P\norm{\widehat f_{\lambda_n}-f_P^\ast}_{H^r}^2
  \le 2\EE_P\norm{\widehat f_{\lambda_n}-f_{\lambda_n,P}}_{H^r}^2
     +2\norm{f_{\lambda_n,P}-f_P^\ast}_{H^r}^2
  \lesssim n^{-2(s-r)/(2s+d)},
\]
where $f_{\lambda_n,P}$ denotes the population ridge solution under $P$.
The constants depend only on the common parameters of the distribution
class.  Taking the supremum over $P$ proves
\[
  \sup_{P\in\mathcal P(s,F,\sigma,p_-,p_+)}
  \EE_P\norm{\widehat f_{\lambda_n}-f_P^\ast}_{H^r}^2
  \lesssim n^{-2(s-r)/(2s+d)}.
\]
This is \eqref{eq:sobolev-risk-upper}.
\end{proof}

\section{Proofs for the population-bias counterexamples}
\label{app:population-bias}

This appendix first gives two zeroth-order constructions and then their
derivative extensions.  On the torus, a
translation-invariant kernel remains diagonal in the Fourier basis, but a
bounded oscillation of its eigenvalues destroys the cancellation of a flat
trigonometric polynomial.  On the cube, a two-dimensional rotation inside
each wavelet superblock couples a spatially remote direction to point
evaluation.  In both cases a single kernel contains all bad blocks.

\subsection{A Fourier-diagonal counterexample on the torus}
\label{sec:fourier-diagonal-counterexample-proof}

Let
\[
  \mathbb T^d=(\mathbb R/\mathbb Z)^d,
  \qquad
  e_k(x):=\exp(2\pi i k\cdot x),
  \qquad k\in\mathbb Z^d,
\]
with normalized Lebesgue measure.  Fix $s>0$ and put
\[
  \mathbf e_1:=(1,0,\ldots,0)\in\mathbb Z^d,
  \qquad
  m:=s+\frac d2,
  \qquad
  \nu:=2m=2s+d,
  \qquad
  \omega_k:=(1+|k|^2)^{-m}.
\]
We use the periodic Sobolev norm
\begin{equation}
  \norm f_{H^m_{\mathrm{per}}}^2
  :=\sum_{k\in\mathbb Z^d}
  (1+|k|^2)^m|\widehat f(k)|^2,
  \label{eq:periodic-sobolev-norm}
\end{equation}
which is an equivalent norm on $H^m(\mathbb T^d)$.
For positive symmetric weights $\mu_{-k}=\mu_k$, consider
\begin{equation}
  K(x,y)=\sum_{k\in\mathbb Z^d}
  \mu_k e_k(x)\overline{e_k(y)}.
  \label{eq:fourier-kernel-general}
\end{equation}
Under uniform design, $T_Ke_k=\mu_ke_k$.  Hence, with
\[
  S_{\lambda,K}:=(T_K+\lambda I)^{-1}T_K,
  \qquad
  B_{\lambda,K}:=I-S_{\lambda,K},
\]
one has
\begin{equation}
  \widehat{B_{\lambda,K}f}(k)
  =b_{\lambda,k}\widehat f(k),
  \qquad
  b_{\lambda,k}:=\frac{\lambda}{\mu_k+\lambda}.
  \label{eq:fourier-residual-multiplier}
\end{equation}
Write $h_\lambda:=\lambda^{1/\nu}$.

The standard radial kernel is not itself a counterexample.  We record its
bias bound in order to isolate the effect of the oscillatory perturbation.

\begin{proposition}[Bias of the radial Bessel kernel]
\label{prop:radial-bessel-bias}
Let $K_0$ have Fourier eigenvalues $\mu_k^0=\omega_k$.  Then, for every
$0<\lambda\le1$,
\begin{equation}
  \norm{B_{\lambda,K_0}f}_{L^\infty(\mathbb T^d)}
  \le
  C\lambda^{s/\nu}
  \norm{f}_{B^s_{\infty,\infty}(\mathbb T^d)}
  =Ch_\lambda^s
  \norm{f}_{B^s_{\infty,\infty}(\mathbb T^d)}.
  \label{eq:radial-bessel-optimal-bias}
\end{equation}
\end{proposition}

\begin{proof}
Let $(\Delta_j)_{j\ge0}$ be a smooth inhomogeneous Littlewood--Paley
decomposition on $\mathbb T^d$.  The residual symbol is the restriction to
$\mathbb Z^d$ of
\[
  b_\lambda(\xi)
  =\frac{\lambda(1+|\xi|^2)^m}
  {1+\lambda(1+|\xi|^2)^m}.
\]
Let $\widetilde\varphi$ be a smooth cutoff equal to one on the Fourier
support of $\Delta_j$, after rescaling to a fixed annulus.  On
$|\xi|\asymp2^j$, rescaling and differentiation give, for all
multi-indices $\alpha$ up to any fixed sufficiently large order,
\[
  \abs{\partial^\alpha
    \bigl[b_\lambda(\xi)\widetilde\varphi(2^{-j}\xi)\bigr]}
  \le
  C_\alpha2^{-j|\alpha|}
  \min\{1,\lambda2^{\nu j}\}.
\]
Taking more than $d$ derivatives, the standard periodic multiplier-kernel
estimate (equivalently, Poisson summation followed by integration by parts)
therefore yields
\begin{equation}
  \norm{B_{\lambda,K_0}\Delta_j}_{L^\infty\to L^\infty}
  \le C\min\{1,\lambda2^{\nu j}\}.
  \label{eq:radial-block-multiplier}
\end{equation}
Since
$\norm{\Delta_jf}_\infty\le
2^{-js}\norm f_{B^s_{\infty,\infty}}$, choose $J$ so that
$\lambda2^{\nu J}\asymp1$.  As $0<s<\nu$,
\begin{align*}
  \norm{B_{\lambda,K_0}f}_\infty
  &\le
  C\norm f_{B^s_{\infty,\infty}}
  \sum_{j\ge0}\min\{1,\lambda2^{\nu j}\}2^{-js}
  \\
  &\le
  C\norm f_{B^s_{\infty,\infty}}
  \left(
    \lambda\sum_{j\le J}2^{(\nu-s)j}
    +\sum_{j>J}2^{-js}
  \right)
  \lesssim
  \lambda^{s/\nu}\norm f_{B^s_{\infty,\infty}}.
\end{align*}
\end{proof}

\begin{lemma}[Real Rudin--Shapiro blocks]
\label{lem:real-rudin-shapiro-blocks}
Let $M$ be a power of two.  There are signs
$\rho_0,\ldots,\rho_{M-1}\in\{-1,1\}$ such that
\begin{equation}
  \sup_{|z|=1}
  \abs{\sum_{q=0}^{M-1}\rho_qz^q}
  \le\sqrt{2M}.
  \label{eq:one-dimensional-rudin-shapiro}
\end{equation}
Consequently, if $R>2\sqrt d\,M$, there is a symmetric set
\[
  \Lambda(R,M)
  :=
  \left(R\mathbf e_1+\{0,\ldots,M-1\}^d\right)
  \cup
  \left(-R\mathbf e_1-\{0,\ldots,M-1\}^d\right)
\]
of cardinality $N=2M^d$, and symmetric signs
$\varepsilon_{-k}=\varepsilon_k\in\{-1,1\}$, such that the real polynomial
\[
  P_{R,M}(x):=
  \sum_{k\in\Lambda(R,M)}\varepsilon_ke_k(x)
\]
satisfies
\begin{equation}
  \norm{P_{R,M}}_\infty\le C_d\sqrt N,
  \qquad
  \abs{\sum_{k\in\Lambda(R,M)}\varepsilon_k}
  \le C_d\sqrt N.
  \label{eq:RS-block-properties}
\end{equation}
\end{lemma}

\begin{proof}
Start with $P_0(z)=Q_0(z)=1$ and recursively define
\[
  P_{q+1}(z)=P_q(z)+z^{2^q}Q_q(z),
  \qquad
  Q_{q+1}(z)=P_q(z)-z^{2^q}Q_q(z).
\]
Both polynomials have coefficients in $\{-1,1\}$, and for $|z|=1$,
\[
  |P_{q+1}(z)|^2+|Q_{q+1}(z)|^2
  =2\bigl(|P_q(z)|^2+|Q_q(z)|^2\bigr).
\]
Thus, when $M=2^q$, $P_q$ has $M$ coefficients and satisfies
\eqref{eq:one-dimensional-rudin-shapiro}.

For $u=(u_1,\ldots,u_d)\in\{0,\ldots,M-1\}^d$, put
$\varepsilon_u:=\prod_{r=1}^d\rho_{u_r}$.  The positive-frequency
polynomial factors as
\[
  e_{R\mathbf e_1}(x)
  \prod_{r=1}^d
  \left(\sum_{q=0}^{M-1}\rho_qe^{2\pi iqx_r}\right),
\]
and hence has supremum norm at most $(2M)^{d/2}$.  Adding its complex
conjugate gives the stated real polynomial.  The second inequality in
\eqref{eq:RS-block-properties} follows by evaluating it at zero.
\end{proof}

\begin{proof}[Proof of Proposition~\ref{prop:fourier-population-bias-counterexample}]
Set $\delta:=\min\{\eta/2,1/8\}$.  Choose a dyadic integer
$L_0=L_0(d,m,\delta)>2\sqrt d$ so large that
$C_{d,m}/L_0\le\delta/80$, with $C_{d,m}$ as in
\eqref{eq:narrow-block-flatness} below.
Take
$R_\ell=L_0 2^{q_\ell}$, where $q_{\ell+1}\ge q_\ell+3$ and
$q_\ell\to\infty$, and put $M_\ell:=R_\ell/L_0$.  Then each $M_\ell$ is
a power of two and all the frequency blocks below are mutually disjoint.  Let
\[
  \Lambda_\ell:=\Lambda(R_\ell,M_\ell),
  \qquad
  N_\ell:=|\Lambda_\ell|=2M_\ell^d,
\]
and use the symmetric signs from
Lemma~\ref{lem:real-rudin-shapiro-blocks}.  Define one perturbed spectrum by
\begin{equation}
  \mu_k=
  \begin{cases}
    (1+\delta\varepsilon_k)\omega_k,
      &k\in\bigcup_\ell\Lambda_\ell,\\
    \omega_k,&\text{otherwise}.
  \end{cases}
  \label{eq:small-perturbed-spectrum}
\end{equation}
Then
\begin{equation}
  \sup_{k\in\mathbb Z^d}
  \abs{\frac{\mu_k}{\omega_k}-1}\le\eta.
  \label{eq:relative-symbol-perturbation}
\end{equation}
Symmetry, positivity, and
$\sum_k\omega_k<\infty$ show that \eqref{eq:fourier-kernel-general}
defines a real-valued, continuous, translation-invariant positive definite
kernel.  Its RKHS satisfies
\begin{equation}
  \frac1{1+\delta}\norm f_{H^m_{\mathrm{per}}}^2
  \le\norm f_{\mathcal H_K}^2
  \le\frac1{1-\delta}\norm f_{H^m_{\mathrm{per}}}^2.
  \label{eq:near-isometric-rkhs-equivalence}
\end{equation}
Thus $\mathcal H_K=H^m(\mathbb T^d)$ as a set, with equivalent norms.
In addition,
\[
  \norm{K-K_0}_{L^\infty(\mathbb T^d\times\mathbb T^d)}
  \le\delta\sum_k\omega_k.
\]

Set
\[
  \lambda_\ell:=\omega_{R_\ell\mathbf e_1}=(1+R_\ell^2)^{-m},
  \qquad
  t_{\ell,k}:=\frac{\omega_k}{\lambda_\ell}.
\]
The relative variation of $|k|^2$ on a cube of side
$R_\ell/L_0$ based at $R_\ell e_1$ is $O_d(L_0^{-1})$.  Hence
\begin{equation}
  \sup_{k\in\Lambda_\ell}|t_{\ell,k}-1|
  \le\frac{C_{d,m}}{L_0}.
  \label{eq:narrow-block-flatness}
\end{equation}

For $t\in[1/2,2]$, define
\begin{align*}
  a_\delta(t)
  &:=\frac12\left(
    \frac1{1+(1+\delta)t}
    +\frac1{1+(1-\delta)t}
  \right),\\
  c_\delta(t)
  &:=\frac12\left(
    \frac1{1+(1+\delta)t}
    -\frac1{1+(1-\delta)t}
  \right)
  =-\frac{\delta t}
  {(1+(1+\delta)t)(1+(1-\delta)t)}.
\end{align*}
The exact residual multiplier on $\Lambda_\ell$ is
\begin{equation}
  b_{\lambda_\ell,k}
  =a_\delta(t_{\ell,k})
  +\varepsilon_kc_\delta(t_{\ell,k}).
  \label{eq:perturbed-residual-decomposition}
\end{equation}
Uniformly on $[1/2,2]$,
\[
  c_\delta(t)\le-\frac{\delta}{20},
  \qquad |a_\delta'(t)|\le C.
\]
The choice of $L_0$ makes the variation error at most
$\delta N_\ell/80$.  Define
\begin{equation}
  f_\ell(x):=
  R_\ell^{-s}N_\ell^{-1/2}P_{R_\ell,M_\ell}(x-x_0).
  \label{eq:fourier-bad-target}
\end{equation}
Its Fourier support lies in one fixed-width dyadic annulus: if
$j_\ell=\lfloor\log_2R_\ell\rfloor$, then
$\Delta_qf_\ell=0$ unless $|q-j_\ell|\le C_{d,L_0}$.  Uniform
$L^\infty$ boundedness of smooth Littlewood--Paley projections and
\eqref{eq:RS-block-properties} give, for each such $q$,
$2^{qs}\norm{\Delta_qf_\ell}_\infty\lesssim1$, and hence
\begin{equation}
  \sup_\ell
  \norm{f_\ell}_{B^s_{\infty,\infty}(\mathbb T^d)}<\infty.
  \label{eq:fourier-target-bounded}
\end{equation}

For $k\in\Lambda_\ell$,
\[
  \widehat f_\ell(k)
  =R_\ell^{-s}N_\ell^{-1/2}
    \varepsilon_ke^{-2\pi i k\cdot x_0}.
\]
Multiplying \eqref{eq:perturbed-residual-decomposition} by
$\varepsilon_k$ and evaluating at $x_0$ yields
\begin{align*}
  B_{\lambda_\ell,K}f_\ell(x_0)
  =R_\ell^{-s}N_\ell^{-1/2}
  \left[
    \sum_{k\in\Lambda_\ell}
      \varepsilon_ka_\delta(t_{\ell,k})
    +\sum_{k\in\Lambda_\ell}c_\delta(t_{\ell,k})
  \right].
\end{align*}
By \eqref{eq:RS-block-properties}, \eqref{eq:narrow-block-flatness}, and
the Lipschitz bound for $a_\delta$,
\[
  \abs{\sum_{k\in\Lambda_\ell}
    \varepsilon_ka_\delta(t_{\ell,k})}
  \le C\sqrt{N_\ell}+\frac{C_{d,m}}{L_0}N_\ell,
\]
whereas
\[
  \sum_{k\in\Lambda_\ell}c_\delta(t_{\ell,k})
  \le-c\delta N_\ell.
\]
For all sufficiently large $\ell$, the second term dominates.  Therefore
\begin{equation}
  \abs{B_{\lambda_\ell,K}f_\ell(x_0)}
  \gtrsim
  \delta R_\ell^{-s}\sqrt{N_\ell}
  \asymp_{\eta,d,m}R_\ell^{d/2-s}
  \asymp h_{\lambda_\ell}^{s-d/2}.
  \label{eq:fourier-bias-lower}
\end{equation}
Dividing by the benchmark $h_{\lambda_\ell}^s$ gives a factor
$h_{\lambda_\ell}^{-d/2}\to\infty$.  When $s=d/2$, the unnormalized
worst-case bias stays bounded away from zero; when $0<s<d/2$, it diverges.
\end{proof}

\begin{remark}
The radial symbol in Proposition~\ref{prop:radial-bessel-bias} has optimal
$O(h_\lambda^s)$ bias.  Proposition~\ref{prop:fourier-population-bias-counterexample}
shows that there exist arbitrarily small, highly oscillatory relative
perturbations of that symbol that destroy the rate.  Thus the obstruction is not Fourier
diagonality itself, but the absence of multiplier regularity within a
frequency annulus.
\end{remark}

\subsection{A wavelet-superblock rotation on the cube}
\label{sec:wavelet-rotation-counterexample-proof}

We now prove Theorem~\ref{thm:population-bias-counterexample}.  This second
construction is on the original domain $\Omega=[0,1]^d$ and uses the
wavelet system of Assumption~\ref{ass:localized-wavelet}.

\begin{proof}[Proof of Theorem~\ref{thm:population-bias-counterexample}]
Throughout the proof the design distribution is uniform on $\Omega$.  The
standard tensor-product CDV system may be chosen with at least one vanishing
moment.  After enlarging the coarse space if necessary, it satisfies
$E_j\one=\one$, $|\Lambda_j|\asymp2^{jd}$, and a ball of radius
$C2^{-L}$ meets at most $C2^{(j-L)d}$ level-$j$ wavelet supports whenever
$j\ge L$.
The
wavelet characterization gives, with $m=s+d/2$,
\begin{align}
  \norm f_{B^s_{\infty,\infty}}
  &\asymp
  \max\left\{
    \norm{P_0f}_\infty,
    \sup_{j\ge1}2^{j(s+d/2)}
    \max_{k\in\Lambda_j}
    \abs{\ip f{\psi_{j,k}}_{L^2}}
  \right\},
  \label{eq:counterexample-besov-coefficients}\\
  \norm f_{H^m}^2
  &\asymp
  \norm{P_0f}_2^2+
  \sum_{j\ge1}2^{2mj}\norm{P_jf}_2^2.
  \label{eq:counterexample-sobolev-coefficients}
\end{align}

We first select fixed-width superblocks carrying large evaluation vectors.
Let $E_j=\sum_{\ell=0}^jP_\ell$, and write $e_j$ for its kernel.  The
square-sum estimate gives $e_j(x,x)\le C2^{jd}$.  Since $E_j$ reproduces
constants, finite propagation and the projection identity imply
\begin{align*}
  1
  &=\abs{\int_\Omega e_j(x,y)\dd y}\\
  &\le
  \abs{\{y:\norm{x-y}\le C2^{-j}\}}^{1/2}
  \left(\int_\Omega e_j(x,y)^2\dd y\right)^{1/2}
  \le C2^{-jd/2}e_j(x,x)^{1/2}.
\end{align*}
Consequently,
\begin{equation}
  c2^{jd}\le e_j(x,x)\le C2^{jd}
  \label{eq:counterexample-projection-diagonal}
\end{equation}
uniformly in $x$ and all sufficiently large $j$.

Fix an integer $M\ge1$.  For an integer $L$, put
\[
  \Xi_L:=\{(j,k):L\le j\le L+M,\ k\in\Lambda_j\},
  \qquad
  W_L:=\Span\{\psi_{j,k}:(j,k)\in\Xi_L\},
\]
and let $N_L:=\dim(W_L)\asymp_M2^{Ld}$.  Moreover,
\[
  \sum_{(j,k)\in\Xi_L}\psi_{j,k}(x_0)^2
  =e_{L+M}(x_0,x_0)-e_{L-1}(x_0,x_0).
\]
Choose $M$ sufficiently large that the lower bound for the first term in
\eqref{eq:counterexample-projection-diagonal} dominates the upper bound for
the second.  Then
\begin{equation}
  \sum_{\xi\in\Xi_L}\psi_\xi(x_0)^2
  \asymp_M2^{Ld}\asymp_MN_L.
  \label{eq:counterexample-evaluation-mass}
\end{equation}

Let
\[
  z_L:=\bigl(\psi_\xi(x_0)\bigr)_{\xi\in\Xi_L},
  \qquad
  u_L:=\frac{z_L}{\norm{z_L}_2},
  \qquad
  U_L:=\sum_{\xi\in\Xi_L}(u_L)_\xi\psi_\xi.
\]
All quantities are real, and
\begin{equation}
  U_L(x_0)=\norm{z_L}_2\asymp\sqrt{N_L}.
  \label{eq:counterexample-evaluation-vector}
\end{equation}
Whenever $(u_L)_\xi\ne0$, the support of $\psi_\xi$ contains $x_0$.
Compact support and $j\ge L$ therefore give
\begin{equation}
  \supp(U_L)\subset
  B_L:=B(x_0,C_02^{-L}),
  \qquad |B_L|\lesssim2^{-Ld}\asymp N_L^{-1}.
  \label{eq:counterexample-near-support}
\end{equation}

Define the remote index set
\[
  \mathcal F_L
  :=\{\xi\in\Xi_L:\supp(\psi_\xi)\cap B_L=\varnothing\}.
\]
At level $j$, the ball $B_L$ meets at most $C2^{(j-L)d}$ wavelet supports.
Since $L\le j\le L+M$ and $M$ is fixed,
\[
  m_L:=|\mathcal F_L|=N_L-O_M(1)\asymp N_L.
\]
Put
\[
  v_L:=m_L^{-1/2}\sum_{\xi\in\mathcal F_L}e_\xi,
  \qquad
  V_L:=\sum_{\xi\in\Xi_L}(v_L)_\xi\psi_\xi,
\]
where $(e_\xi)$ is the coordinate basis.  Then
\begin{equation}
  \norm{v_L}_2=1,
  \qquad
  \max_\xi|(v_L)_\xi|=m_L^{-1/2}\asymp N_L^{-1/2},
  \qquad
  V_L|_{B_L}=0.
  \label{eq:counterexample-far-vector}
\end{equation}
The coordinate supports of $u_L$ and $v_L$ are disjoint, so
$u_L\perp v_L$.

Choose $L_0$ large and let
\[
  \mathcal I:=\{L_0+q(M+1):q=0,1,2,\ldots\}.
\]
The intervals $[L,L+M]$, $L\in\mathcal I$, are disjoint and cover all
sufficiently large levels.  Hence
\[
  L^2(\Omega)=W_{\mathrm{low}}
  \oplus\bigoplus_{L\in\mathcal I}W_L,
\]
where $W_{\mathrm{low}}$ is finite-dimensional.

Fix $0<\delta<1/4$.  On the coefficient space of $W_L$, define
\begin{equation}
  \mathsf M_L
  :=I+\delta(u_L\otimes v_L+v_L\otimes u_L),
  \qquad
  \mu_L:=2^{-L(2s+d)}.
  \label{eq:counterexample-block-precision}
\end{equation}
The vectors
\[
  w_L^+:=\frac{u_L+v_L}{\sqrt2},
  \qquad
  w_L^-:=\frac{u_L-v_L}{\sqrt2}
\]
are eigenvectors of $\mathsf M_L$ with eigenvalues $1+\delta$ and
$1-\delta$, respectively, and $\mathsf M_L$ is the identity on their
orthogonal complement.  Thus
\begin{equation}
  (1-\delta)I\le\mathsf M_L\le(1+\delta)I.
  \label{eq:counterexample-M-coercivity}
\end{equation}

Define a positive self-adjoint precision operator by
\[
  A|_{W_{\mathrm{low}}}:=I,
  \qquad
  A|_{W_L}:=\mu_L^{-1}\mathsf M_L,
  \quad L\in\mathcal I,
\]
with its natural maximal domain.  From
\eqref{eq:counterexample-M-coercivity}, the fixed width of the blocks, and
\eqref{eq:counterexample-sobolev-coefficients},
\begin{equation}
  \norm{A^{1/2}f}_2^2
  \asymp
  \norm{\Pi_{\mathrm{low}}f}_2^2+
  \sum_{L\in\mathcal I}2^{L(2s+d)}\norm{\Pi_{W_L}f}_2^2
  \asymp\norm f_{H^{s+d/2}(\Omega)}^2.
  \label{eq:counterexample-form-sobolev-equivalence}
\end{equation}

Let $T_K=A^{-1}$.  On $W_L$,
$T_K=\mu_L\mathsf M_L^{-1}$, whose eigenvalues in the directions $w_L^+$,
$w_L^-$, and the remaining orthogonal complement are
\[
  \frac{\mu_L}{1+\delta},
  \qquad
  \frac{\mu_L}{1-\delta},
  \qquad
  \mu_L.
\]
Thus only two directions are rotated in each block, and all eigenvalues
remain within fixed multiples of the Sobolev scale.

We verify that $T_K$ has a continuous bounded kernel.  Put
$z_L(x):=(\psi_\xi(x))_{\xi\in\Xi_L}$.  The kernel of the $L$-th block is
\[
  K_L(x,y)
  :=\mu_L
  \ip{\mathsf M_L^{-1}z_L(y)}{z_L(x)}_{\ell^2}.
\]
The levelwise square-sum estimate and fixed block width give
\[
  \sup_x\norm{z_L(x)}_2^2\le C_M2^{Ld}.
\]
Since $\norm{\mathsf M_L^{-1}}_{\op}\le(1-\delta)^{-1}$,
\begin{equation}
  \sup_{x,y}|K_L(x,y)|
  \le C_{\delta,M}\mu_L2^{Ld}
  =C_{\delta,M}2^{-2sL}.
  \label{eq:counterexample-kernel-block-bound}
\end{equation}
Each $K_L$ is a finite sum of products of continuous real wavelets, and
$\sum_{L\in\mathcal I}2^{-2sL}<\infty$.  Adding the finite-dimensional
low-frequency block therefore gives a uniformly convergent, bounded,
continuous, real symmetric positive definite kernel $K$, whose integral
operator is $T_K$.  Its RKHS is $\operatorname{Dom}(A^{1/2})$, with
\[
  \norm f_{\mathcal H_K}^2=\norm{A^{1/2}f}_2^2.
\]
Equation~\eqref{eq:counterexample-form-sobolev-equivalence} proves
$\mathcal H_K\simeq H^{s+d/2}(\Omega)$.

For $L\in\mathcal I$, set
\begin{equation}
  \rho_L:=\mu_L^{1/2}=2^{-L(s+d/2)},
  \qquad
  f_L:=\rho_L\sqrt{N_L}\,V_L,
  \qquad
  \lambda_L:=\mu_L.
  \label{eq:counterexample-target-and-lambda}
\end{equation}
Every nonzero wavelet coefficient of $f_L$ has magnitude comparable to
$\rho_L$, and all occur at levels $L,\ldots,L+M$.  Hence
\begin{equation}
  \sup_{L\in\mathcal I}
  \norm{f_L}_{B^s_{\infty,\infty}}<\infty.
  \label{eq:counterexample-uniform-besov-bound}
\end{equation}

Under uniform design,
\[
  S_\lambda=(T_K+\lambda I)^{-1}T_K=(I+\lambda A)^{-1}.
\]
At $\lambda=\lambda_L=\mu_L$, the residual restricted to $W_L$ is
\[
  (I-S_{\lambda_L})|_{W_L}
  =\mathsf M_L(I+\mathsf M_L)^{-1}.
\]
In the ordered basis $(u_L,v_L)$,
\begin{equation}
  \mathsf M_L(I+\mathsf M_L)^{-1}
  =\frac1{4-\delta^2}
  \begin{pmatrix}
    2-\delta^2&\delta\\
    \delta&2-\delta^2
  \end{pmatrix}.
  \label{eq:counterexample-residual-matrix}
\end{equation}
Since $f_L=\rho_L\sqrt{N_L}\,V_L$, the $u_L$ coefficient of
$(I-S_{\lambda_L})f_L$ is
\[
  c_\delta\rho_L\sqrt{N_L},
  \qquad c_\delta:=\frac\delta{4-\delta^2}>0.
\]
The remaining component is proportional to $V_L$ and vanishes at $x_0$.
Using \eqref{eq:counterexample-evaluation-vector},
\begin{align}
  \abs{(I-S_{\lambda_L})f_L(x_0)}
  &=c_\delta\rho_L\sqrt{N_L}\,U_L(x_0)
  \notag\\
  &\gtrsim\rho_LN_L
  \asymp2^{-L(s+d/2)}2^{Ld}
  =2^{-L(s-d/2)}.
  \label{eq:counterexample-pointwise-lower-bound}
\end{align}
Together with \eqref{eq:counterexample-uniform-besov-bound}, this also gives
\[
  h_L^{-s}
  \norm{I-S_{\lambda_L}}_
  {B^s_{\infty,\infty}\to L^\infty}
  \gtrsim h_L^{-d/2}\longrightarrow\infty,
  \qquad h_L=\lambda_L^{1/(2s+d)}=2^{-L}.
\]
This proves all the asserted quantitative conclusions.  All rotations were installed
simultaneously in mutually orthogonal superblocks, so $K$ is a single fixed
kernel; $\lambda_L$ only selects the block being probed.  The construction is
real-valued throughout.
\end{proof}

\begin{remark}[Comparison of the two mechanisms]
The periodic example changes only eigenvalues in a fixed Fourier basis; the
cube example changes only two eigendirections inside each fixed-width
Sobolev band.  Thus neither translation invariance nor Sobolev eigenvalue
decay alone supplies endpoint control.  Both are population-level,
uniform-over-the-class obstructions.  They do not assert nonconvergence for
one fixed target, nor do they by themselves rule out every empirical tuning
scheme.
\end{remark}

\subsection{Derivative extensions of the two counterexamples}
\label{app:derivative-counterexamples}

\begin{proof}[Proof of Corollary~\ref{cor:derivative-bias-obstruction}]
In this proof write $a=|\alpha|<s$ for the derivative order, to avoid
confusion with the Sobolev exponent used above.

For the cube construction, choose a sufficiently regular CDV system
whose cumulative projections reproduce polynomials of degree at least
$a$.  Such a choice is allowed in the counterexample construction.
Let $p_{x_0}(y)=(y-x_0)^\alpha/\alpha!$ and
$b_j(y)=D_x^\alpha e_j(x_0,y)$.  Polynomial reproduction implies
$\int b_j(y)p_{x_0}(y)\dd y=1$.
On the finite propagation support of $b_j$,
$|p_{x_0}(y)|\le C2^{-ja}$ and the support has volume $O(2^{-jd})$.
Cauchy--Schwarz therefore yields
\[
 1\le C2^{-j(a+d/2)}\norm{b_j}_2.
\]
The derivative upper bounds in
Lemma~\ref{lem:derivative-wavelet-scales} give the reverse inequality.
By the projection identity,
\[
 \sum_{\ell\le j,k}|D^\alpha\psi_{\ell,k}(x_0)|^2
 =\norm{b_j}_2^2\asymp2^{j(d+2a)}.
\]
Choosing the fixed superblock width $M$ sufficiently large as before,
the difference of two cumulative sums gives
\[
 \sum_{\xi\in\Xi_L}|D^\alpha\psi_\xi(x_0)|^2
 \asymp2^{L(d+2a)}.
\]
Replace the evaluation vector in the cube proof by
$z_L=(D^\alpha\psi_\xi(x_0))_{\xi\in\Xi_L}$ and set
$u_L=z_L/\norm{z_L}_2$.  Then
$D^\alpha U_L(x_0)=\norm{z_L}_2\asymp2^{L(a+d/2)}$.
Its support remains inside $B(x_0,C2^{-L})$.
The remote vector $v_L$, the bounded block matrices $\mathsf M_L$,
and the targets $f_L=\rho_L\sqrt{N_L}V_L$ are constructed exactly
as above.  All estimates proving a fixed continuous kernel and
$\HH_K\simeq H^{s+d/2}$ use only $\norm{u_L}_2=1$ and disjoint
coordinate supports, so they remain valid.  As $V_L$ vanishes in a
neighborhood of $x_0$, the same residual matrix gives
\[
 |D^\alpha(I-S_{\lambda_L,K})f_L(x_0)|
 =c_\delta\rho_L\sqrt{N_L}\,D^\alpha U_L(x_0)
 \asymp2^{-L(s-a-d/2)}.
\]
The original Besov coefficient bound on $f_L$ is unchanged.
Derivative evaluation is well defined because
$H^{s+d/2}\hookrightarrow C^a$ when $a<s$.

For the torus construction, translate the positive frequency rectangle
to $R\mathbf v+\{0,\ldots,M-1\}^d$, where
$\mathbf v=(1,\ldots,1)$, and include its negative.  Modulation by
$e_{R\mathbf v}$ leaves the Rudin--Shapiro supremum estimate unchanged.
Set $M=R/L_0$ and $\lambda=(1+dR^2)^{-(s+d/2)}$.
The spectrum still varies by only $O(L_0^{-1})$ across each rectangle.
Consequently the preceding Fourier proof, with this translated
rectangle, gives real targets $u_R$ with
\[
 \norm{u_R}_\infty\lesssim R^{-s},\qquad
 \norm{u_R}_{B^s_{\infty,\infty}}\lesssim1,\qquad
 |(I-S_{\lambda,K})u_R(x_0)|\gtrsim R^{d/2-s},
\]
for a single fixed kernel satisfying the prescribed relative spectral
perturbation bound.  All coordinates of frequencies in these rectangles
have absolute value comparable to $R$.
Define a new real target by
\[
 \widehat f_R(k)=\frac{R^a}{(2\pi i k)^\alpha}\widehat u_R(k)
 \quad\text{on the two rectangles},\qquad
 \widehat f_R(k)=0\quad\text{otherwise}.
\]
Here $(2\pi i k)^\alpha=\prod_j(2\pi i k_j)^{\alpha_j}$ never vanishes
on the support.  The multiplier is uniformly bounded on $L^\infty$
after inserting a smooth cutoff at scale $R$ supported away from all
coordinate hyperplanes: its rescaled inverse Fourier kernel has
uniformly bounded $L^1$ norm, also after periodization.
Thus $\norm{f_R}_\infty\lesssim R^{-s}$ and, by its annular support,
$\norm{f_R}_{B^s_{\infty,\infty}}\lesssim1$.
Conjugate symmetry of the coefficients ensures that $f_R$ is real.
Since differentiation commutes with the diagonal residual,
\[
 D^\alpha(I-S_{\lambda,K})f_R
 =R^a(I-S_{\lambda,K})u_R.
\]
The resulting lower bound is $R^{a+d/2-s}\asymp h^{s-a-d/2}$.
Both constructions can be rescaled into any prescribed positive-radius
Besov ball.  Dividing their lower bounds by $h^{s-a}$ gives the divergent
factor $h^{-d/2}$, proving all assertions.
\end{proof}

For comparison, the radial Bessel bias proof also extends to derivatives:
the block bound \eqref{eq:radial-block-multiplier} and Bernstein's
inequality give
\[
 \norm{D^\alpha B_{\lambda,K_0}f}_\infty
 \le C\norm f_{B^s_{\infty,\infty}}
       \sum_{j\ge0}2^{-j(s-|\alpha|)}
           \min\{1,\lambda2^{\nu j}\}
 \le C h_\lambda^{s-|\alpha|}\norm f_{B^s_{\infty,\infty}}.
\]
The two geometric tails converge because $0<s-|\alpha|<\nu$.

\section{Proofs for the localized kernel bounds}
\label{app:localized-kernel}

For this appendix, write
\[
  c:=c_p,\qquad \widetilde p:=p/c,\qquad
  \kappa:=p_-/c,\qquad
  \eta:=\eta_\lambda,\qquad
  h:=h_\lambda,\qquad
  q_h(t):=t+\log(e/h).
\]
Thus $0<\kappa\le\widetilde p\le1$, $\eta=h^\nu$, and all constants below
are uniform over $P\in\mathcal P(s,F,\sigma,p_-,p_+)$.  We first state the
three zeroth-order results in the notation used below.  This appendix
establishes the case $m=0$ and the inverse estimates used for all derivative
orders.  Appendix~\ref{app:derivative-losses} completes the derivative proofs.

\begin{proposition}[Population approximation]
\label{prop:localized-population-bias}
Under the assumptions of Theorem~\ref{thm:localized-high-probability},
\[
  \norm{f^\ast-f_\lambda}_\infty
  \le CFh^s.
\]
\end{proposition}

\begin{lemma}[Fixed-point concentration]
\label{lem:localized-pointwise-concentration}
Under the assumptions of Theorem~\ref{thm:localized-high-probability}, fix $x_0\in(0,1)^d$ before observing the data.  If
\[
  nh^d\ge C_0q_h(t),
\]
then, with probability at least $1-Ce^{-t}$,
\begin{equation}
\label{eq:localized-pointwise-concentration}
  \abs{\widehat f_\lambda(x_0)-f_\lambda(x_0)}
  \le C\left[
    \norm{f^\ast-f_\lambda}_\infty
    \left\{
      \sqrt{\frac{t}{nh^d}}
      +\frac{t}{nh^d}
    \right\}
    +\sigma\sqrt{\frac{t}{nh^d}}
  \right].
\end{equation}
Consequently,
\[
  \abs{\widehat f_\lambda(x_0)-f^\ast(x_0)}
  \le C\left\{
    Fh^s+\sigma\sqrt{\frac{t}{nh^d}}
  \right\}.
\]
\end{lemma}

\begin{lemma}[Uniform concentration]
\label{lem:localized-uniform-concentration}
Under the assumptions of Theorem~\ref{thm:localized-high-probability}, if
\[
  nh^d\ge C_0q_h(t),
\]
then, with probability at least $1-Ce^{-t}$,
\begin{equation}
\label{eq:localized-uniform-concentration}
  \norm{\widehat f_\lambda-f_\lambda}_\infty
  \le C\left[
    \norm{f^\ast-f_\lambda}_\infty
    \left\{
      \sqrt{\frac{q_h(t)}{nh^d}}
      +\frac{q_h(t)}{nh^d}
    \right\}
    +\sigma\sqrt{\frac{q_h(t)}{nh^d}}
  \right].
\end{equation}
Consequently,
\[
  \norm{\widehat f_\lambda-f^\ast}_\infty
  \le C\left\{
    Fh^s+\sigma\sqrt{\frac{q_h(t)}{nh^d}}
  \right\}.
\]
\end{lemma}

\subsection{Population localization and approximation}

Let
\[
  G:=(I+\eta A)^{-1/2}
  =\sum_{j\ge0}\gamma_jP_j,
  \qquad
  \gamma_j:=\frac{1}{\sqrt{1+\eta2^{\nu j}}},
\]
and let $g_x:=G\delta_x$ denote its $L^2(\Omega)$ kernel row.  This distributional notation is valid because the displayed series defining $g_x$ converges in $L^2$ by the square-sum bound in Assumption~\ref{ass:localized-wavelet}.
Because the wavelets are continuous and the square-sum remains summable after multiplication by $\gamma_j^2$, $x\mapsto g_x$ is $L^2$-continuous.  In particular, all expectations below are Bochner integrals, and
\[
  \ip{g_x}{u}_{L^2}=(Gu)(x),
  \qquad u\in L^2(\Omega).
\]

\begin{lemma}[Half-resolvent localization]
\label{lem:localized-half-resolvent}
Fix $d/2<\alpha<\nu$.  There are constants $C,C_G>0$ such that the integral kernel $g_h$ of $G$ satisfies
\begin{align}
  \sup_x\int_\Omega\abs{g_h(x,y)}\dd y+
  \sup_y\int_\Omega\abs{g_h(x,y)}\dd x&\le C,
  \label{eq:localized-G-schur}\\
  \abs{g_h(x,y)}&\le Ch^\nu\norm{x-y}^{-(d+\nu)}
  \quad\text{if }\norm{x-y}\ge h,
  \label{eq:localized-G-far}\\
  \sup_x\norm{g_x}_2&\le Ch^{-d/2}.
  \label{eq:localized-G-row}
\end{align}
For $z\in\Omega$ and $0<\varepsilon\le1$, let
\[
  w_{z,\varepsilon}(x):=
  \left(1+\varepsilon\norm{x-z}/h\right)^\alpha.
\]
Then
\begin{align}
  \sup_z\norm{
    M_{w_{z,\varepsilon}}GM_{w_{z,\varepsilon}}^{-1}-G
  }_{\op}&\le C_G\varepsilon,
  \label{eq:localized-G-commutator}\\
  \sup_z\norm{w_{z,\varepsilon}g_z}_2&\le Ch^{-d/2},
  &\sup_z\norm{w_{z,\varepsilon}^{-1}}_2&\le Ch^{d/2}\varepsilon^{-d/2}.
  \label{eq:localized-G-weighted-rows}
\end{align}
The same bound holds after interchanging the two weights.
\end{lemma}

\begin{proof}
Put $b_j=(1+(h2^j)^\nu)^{-1/2}$ and $a_j=b_j-b_{j+1}$.  Since $P_j=E_j-E_{j-1}$, Abel summation gives
\[
  G=\sum_{j\ge0}a_jE_j
\]
with convergence both in the $L^2$ operator norm and in the Schur norm of the associated kernel.  Indeed, $a_j\ge0$, $\sum_ja_j\le1$, and
\[
  a_j\le C\min\{(h2^j)^\nu,(h2^j)^{-\nu/2}\}.
  \tag{\thesection.1}
\]
The finite propagation bound for $e_j$ and the volume of Euclidean balls give $\sup_j\norm{e_j}_{\mathrm{Schur}}\le C$, which proves \eqref{eq:localized-G-schur}.

If $r=\norm{x-y}\ge h$, only indices with $2^{-j}\gtrsim r$ contribute to $g_h(x,y)$.  The same propagation estimate and the first term in (\thesection.1) yield
\[
  \abs{g_h(x,y)}
  \le Ch^\nu\sum_{2^j\lesssim r^{-1}}2^{j(d+\nu)}
  \le Ch^\nu r^{-(d+\nu)}.
\]
Orthogonality and the square-sum bound give
\[
  \norm{g_x}_2^2
  \le C\sum_{j\ge0}\frac{2^{jd}}{1+(h2^j)^\nu}
  \le Ch^{-d},
\]
where $\nu>d$ is used in the high-frequency tail.

For the weighted bound, the triangle inequality implies, on the support of $e_j$,
\[
  \max\left\{
    \frac{w_{z,\varepsilon}(x)}{w_{z,\varepsilon}(y)},
    \frac{w_{z,\varepsilon}(y)}{w_{z,\varepsilon}(x)}
  \right\}
  \le \left(1+\frac{C\varepsilon}{h2^j}\right)^\alpha.
\]
Applying the Schur test to the Abel expansion and splitting the sum at $h2^j=\varepsilon$ and $h2^j=1$ gives
\[
\begin{split}
  \norm{M_wGM_w^{-1}-G}_{\op}
  &\le C\varepsilon^\alpha
    \sum_{h2^j\lesssim\varepsilon}(h2^j)^{\nu-\alpha}\\
  &\quad+C\varepsilon
    \sum_{\varepsilon\lesssim h2^j\le1}(h2^j)^{\nu-1}
    +C\varepsilon\sum_{h2^j>1}(h2^j)^{-\nu/2-1}\\
  &\le C_G\varepsilon.
\end{split}
\]
Here $\alpha<\nu$ controls the first sum.  The bound for the reverse conjugation follows by taking adjoints.

Finally, combine \eqref{eq:localized-G-row} on $\{\norm{y-z}<h\}$ with \eqref{eq:localized-G-far} on dyadic annuli.  Since $\alpha<\nu$, this gives $\norm{w_{z,\varepsilon}g_z}_2\le Ch^{-d/2}$.  The Euclidean volume estimate and $2\alpha>d$ give
\[
  \int_\Omega w_{z,\varepsilon}(x)^{-2}\dd x
  \le Ch^d\varepsilon^{-d},
\]
which proves the remaining claim.
\end{proof}

\subsection{Image-space interpolation and population ridge stability}

Put
\[
  T_{\widetilde p}:=T_0M_{\widetilde p},
  \qquad
  Y_{\widetilde p}:=T_{\widetilde p}(L^\infty(\Omega)),
  \qquad
  \norm{u}_{Y_{\widetilde p}}
  :=\inf\left\{\norm g_\infty:T_{\widetilde p}g=u\right\}.
\]

\begin{lemma}[Uniform image-space interpolation]
\label{lem:localized-image-interpolation}
The operator $T_{\widetilde p}$ is injective on $L^\infty(\Omega)$,
$Y_{\widetilde p}$ is a Banach space continuously embedded in
$L^\infty(\Omega)$, and, with $\theta=s/\nu$,
\[
  \bigl(L^\infty(\Omega),Y_{\widetilde p}\bigr)_{\theta,\infty}
  =B^s_{\infty,\infty}(\Omega)
  \tag{\thesection.2}
  \label{eq:localized-image-interpolation}
\]
with equivalent norms.  All constants are uniform over the admissible
design densities.  Equivalently,
\[
  \bigl(Y_{\widetilde p},L^\infty(\Omega)\bigr)_{1-\theta,\infty}
  =B^s_{\infty,\infty}(\Omega).
\]
\end{lemma}

\begin{proof}
Finite overlap, the support bound, and the $L^\infty$ wavelet bound give
\[
  \norm{P_jf}_\infty
  \le
  \sup_x\sum_{k:x\in\supp\psi_{j,k}}
  \norm{\psi_{j,k}}_\infty\norm{\psi_{j,k}}_1\norm f_\infty
  \le C\norm f_\infty,
  \qquad j\ge0.
  \tag{\thesection.3}
  \label{eq:localized-wavelet-projection-infty}
\]
Consequently $T_0=\sum_{j\ge0}2^{-\nu j}P_j$ is bounded on
$L^\infty$.  It is also injective: if $T_0g=0$, then
$2^{-\nu j}P_jg=P_jT_0g=0$ for every $j$, and completeness of the
wavelet system gives $g=0$; here we used
$L^\infty(\Omega)\subset L^2(\Omega)$.  Since
$\kappa\le\widetilde p\le1$, multiplication by $\widetilde p$ is an
automorphism of $L^\infty$.  Thus $T_{\widetilde p}$ is injective,
the infimum in the image norm is attained at the unique preimage
$T_{\widetilde p}^{-1}u$, and $T_{\widetilde p}$ is an isometric
isomorphism from $L^\infty$ onto $Y_{\widetilde p}$.  This proves
completeness.  The $L^\infty$ boundedness of $T_0$ also gives the
continuous embedding $Y_{\widetilde p}\hookrightarrow L^\infty$.

For $t>0$, write
\[
  K_{\widetilde p}(t,f)
  :=K(t,f;L^\infty,Y_{\widetilde p})
  =\inf_{f=a+T_{\widetilde p}g}
    \left\{\norm a_\infty+t\norm g_\infty\right\}.
\]
For any such decomposition, \eqref{eq:localized-wavelet-projection-infty}
and $P_jT_0=2^{-\nu j}P_j$ imply
\[
  \norm{P_jf}_\infty
  \le C\left(\norm a_\infty+2^{-\nu j}\norm g_\infty\right).
\]
Taking the infimum and using \eqref{eq:besov-wavelet-norm},
\[
  \norm f_{B^s_{\infty,\infty}}
  \le C\sup_{j\ge0}2^{sj}K_{\widetilde p}(2^{-\nu j},f)
  \le C\norm f_{(L^\infty,Y_{\widetilde p})_{\theta,\infty}}.
\]

Conversely, let $f\in B^s_{\infty,\infty}$ and
$E_mf=\sum_{j=0}^mP_jf$.  The canonical wavelet reconstruction and
$s>0$ give
\[
  \norm{f-E_mf}_\infty
  \le C2^{-sm}\norm f_{B^s_{\infty,\infty}}.
\]
Moreover,
\[
  E_mf=T_{\widetilde p}g_m,
  \qquad
  g_m:=M_{1/\widetilde p}\sum_{j=0}^m2^{\nu j}P_jf,
\]
and, since $s<\nu$,
\[
  \norm{g_m}_\infty
  \le C2^{(\nu-s)m}\norm f_{B^s_{\infty,\infty}}.
\]
If $2^{-\nu(m+1)}<t\le2^{-\nu m}$, this decomposition yields
\[
\begin{split}
  K_{\widetilde p}(t,f)
  &\le \norm{f-E_mf}_\infty+t\norm{g_m}_\infty\\
  &\le Ct^{s/\nu}\norm f_{B^s_{\infty,\infty}}.
\end{split}
\]
For $t\ge1$, the same bound follows from
$K_{\widetilde p}(t,f)\le\norm f_\infty$ and the Besov embedding.
Taking $\sup_{t>0}t^{-\theta}$ proves the reverse norm inequality and
\eqref{eq:localized-image-interpolation}.  The exchanged-endpoint
identity follows from
$K(t,f;Y_{\widetilde p},L^\infty)
=tK(t^{-1},f;L^\infty,Y_{\widetilde p})$.
All estimates depend on the density only through the common lower bound
$\kappa$.
\end{proof}

Choose $\varepsilon_0\in(0,1]$ so small that $C_G\varepsilon_0\le\kappa/6$, and write $w_z:=w_{z,\varepsilon_0}$ from now on.  Define
\[
  H:=I-G^2+GM_{\widetilde p}G
  =I+G(M_{\widetilde p}-I)G.
\]

\begin{proposition}[Population weighted inverse and operator identities]
\label{prop:localized-population-weighted-inverse}
For this operator $H$,
\begin{align}
  H&\ge\kappa I,
  \label{eq:localized-H-coercive}\\
  \sup_{z\in\Omega}\left\{
  \norm{M_{w_z}H^{-1}M_{w_z}^{-1}}_{\op}
  +\norm{M_{w_z}^{-1}H^{-1}M_{w_z}}_{\op}\right\}&\le C.
  \label{eq:localized-H-weighted-inverse}
\end{align}
Moreover, the population smoother admits the representations
\[
  S_{\lambda,p}
  :=(T_{\widetilde p}+\eta I)^{-1}T_{\widetilde p}
  =(M_{\widetilde p}+\eta A)^{-1}M_{\widetilde p}
  =GH^{-1}GM_{\widetilde p}.
\]
\end{proposition}

\begin{proof}
Write $q=\widetilde p-1$.  Since $-(1-\kappa)\le q\le0$ and $\norm{G}_{\op}\le1$,
\[
  H=I+GM_qG\ge\kappa I.
\]
For $S_z=M_{w_z}GM_{w_z}^{-1}$, \eqref{eq:localized-G-commutator} gives
\[
  M_{w_z}HM_{w_z}^{-1}=I+S_zM_qS_z,
  \qquad
  \norm{S_zM_qS_z}_{\op}
  \le(1-\kappa)(1+C_G\varepsilon_0)^2\le1-\kappa/2.
\]
The Neumann series proves the first weighted inverse bound, and
self-adjointness of $H^{-1}$ gives the bound with the weights reversed.

Since
$L_{K_\nu,p}=c_KT_0M_p=c_Kc_pT_{\widetilde p}$ and
$\eta=\lambda/(c_Kc_p)$, \eqref{eq:population-solution} gives the
first representation of $S_{\lambda,p}$.  If
$(T_{\widetilde p}+\eta I)u=T_{\widetilde p}f$, then
\[
  \eta u=T_0M_{\widetilde p}(f-u)
  \in\operatorname{Ran}(T_0)=\operatorname{Dom}(A).
\]
Applying $A$ gives
$(M_{\widetilde p}+\eta A)u=M_{\widetilde p}f$; conversely, multiplying
the latter equation by $T_0$ gives the former.  This proves the second
representation without a domain ambiguity.  To obtain the third,
substitute $f=Gu$ in the quadratic form of
$M_{\widetilde p}+\eta A$: for $u,v\in L^2$,
\[
  \ip{M_{\widetilde p}Gu}{Gv}
  +\eta\ip{A^{1/2}Gu}{A^{1/2}Gv}
  =\ip{Hu}{v}.
\]
Thus $(M_{\widetilde p}+\eta A)^{-1}=GH^{-1}G$ in the form sense.
\end{proof}

\begin{proof}[Proof of Proposition~\ref{prop:localized-population-bias}]
Put \(R_{\lambda,p}:=I-S_{\lambda,p}\), with \(S_{\lambda,p}\) as in
Proposition~\ref{prop:localized-population-weighted-inverse}.
By Corollary~\ref{cor:localized-ridge-linfty-sun}, whose complete proof
from Sun's finite-lattice Wiener theorem is given in
Appendix~\ref{app:uniform-cdv-resolvent},
\(\norm{S_{\lambda,p}}_{L^\infty\to L^\infty}\le C\) uniformly in
\(p\) and \(\lambda\).  Therefore
\[
  \norm{R_{\lambda,p}a}_\infty
  \le(1+\norm{S_{\lambda,p}}_{L^\infty\to L^\infty})\norm a_\infty
  \le C\norm a_\infty,
  \qquad a\in L^\infty(\Omega).
\]
On $L^2$, the resolvent identity gives
\[
  R_{\lambda,p}T_{\widetilde p}=\eta S_{\lambda,p}.
\]
Both sides extend boundedly to $L^\infty$.  Hence, if
$u=T_{\widetilde p}g\in Y_{\widetilde p}$,
\[
  R_{\lambda,p}u
  =R_{\lambda,p}T_{\widetilde p}g
  =\eta S_{\lambda,p}g,
\]
and hence
\[
  \norm{R_{\lambda,p}u}_\infty
  \le C\eta\norm g_\infty
  =C\eta\norm u_{Y_{\widetilde p}}.
\]
Applying the two endpoint bounds to every decomposition $f=a+u$ and
taking the infimum gives
\[
  \norm{R_{\lambda,p}f}_\infty
  \le C K(\eta,f;L^\infty,Y_{\widetilde p}).
\]
Lemma~\ref{lem:localized-image-interpolation}, the definition of the
real interpolation norm, and $\eta=h^\nu$ now yield
\[
\begin{split}
  \norm{f^\ast-f_\lambda}_\infty
  &=\norm{R_{\lambda,p}f^\ast}_\infty\\
  &\le C\eta^{s/\nu}
    \norm{f^\ast}_{B^s_{\infty,\infty}}\\
  &\le CFh^s.
\end{split}
\]
\end{proof}

\subsection{Empirical factorization and weighted inverse}

For a measurable positive weight $w$, write $M_w$ for multiplication by $w$.  Define the strictly positive operator
\[
  H_n:=I-G^2+\frac1{cn}\sum_{i=1}^n g_{X_i}\otimes g_{X_i}.
\]
It is strictly positive for each fixed $\eta>0$ because $I-G^2\ge\eta(1+\eta)^{-1}I$.

\begin{lemma}[Exact empirical factorization]
\label{lem:localized-factorization}
Let $\widehat f_\lambda^{\,0}$ be the KRR estimator with $Y_i$ replaced by $f^\ast(X_i)$.  For any $r\in L^\infty(\Omega)$ fixed independently of $X_1,\ldots,X_n$, put
\[
  \xi_r:=\frac1c\left\{\frac1n\sum_{i=1}^n r(X_i)g_{X_i}
  -\EE[r(X)g_X]\right\}.
\]
Then
\begin{align}
  \widehat f_\lambda
  &=GH_n^{-1}\left\{\frac1{cn}\sum_{i=1}^nY_ig_{X_i}\right\},
  \label{eq:localized-krr-factorization}\\
  \widehat f_\lambda^{\,0}-f_\lambda
  &=GH_n^{-1}\xi_{r_\lambda},
  \qquad r_\lambda:=f^\ast-f_\lambda.
  \label{eq:localized-design-factorization}
\end{align}
Consequently,
\begin{equation}
  \widehat f_\lambda-f^\ast
  =-r_\lambda+D_{n,\lambda}+N_{n,\lambda},
  \label{eq:localized-master-decomposition}
\end{equation}
where
\[
  D_{n,\lambda}:=GH_n^{-1}\xi_{r_\lambda},
  \qquad
  N_{n,\lambda}:=
  GH_n^{-1}\left\{\frac1{cn}\sum_{i=1}^n\epsilon_i g_{X_i}\right\}.
\]
\end{lemma}

\begin{proof}
The KRR normal equation, after division by $c$, is
\[
  \frac1{cn}\sum_{i=1}^n\widehat f_\lambda(X_i)v(X_i)
  +\eta\ip{A^{1/2}\widehat f_\lambda}{A^{1/2}v}
  =\frac1{cn}\sum_{i=1}^nY_iv(X_i)
\]
for every $v\in\HH_\nu$.  Write $\widehat f_\lambda=Gu$ and $v=Gv_0$.  Spectral calculus gives
\[
  \eta\ip{A^{1/2}Gu}{A^{1/2}Gv_0}
  =\ip{(I-G^2)u}{v_0},
  \qquad
  (Gu)(x)=\ip{u}{g_x}_{L^2}.
\]
The normal equation is therefore $\ip{H_nu}{v_0}=\ip{(cn)^{-1}\sum_iY_ig_{X_i}}{v_0}$, proving \eqref{eq:localized-krr-factorization}.

Subtract the weak equation \eqref{eq:localized-population-equation} from its noiseless empirical counterpart.  For every $v\in\HH_\nu$,
\[
  \frac1n\sum_{i=1}^n(\widehat f_\lambda^{\,0}-f_\lambda)(X_i)v(X_i)
  +\frac{\lambda}{c_K}\ip{A^{1/2}(\widehat f_\lambda^{\,0}-f_\lambda)}
  {A^{1/2}v}
  =
  (P_n-P)(r_\lambda v).
\]
Using the same $G$ parameterization converts the right side into $\ip{\xi_{r_\lambda}}{v_0}$ and proves \eqref{eq:localized-design-factorization}.  Splitting $Y_i=f^\ast(X_i)+\epsilon_i$ gives \eqref{eq:localized-master-decomposition}.
\end{proof}

\begin{lemma}[Weight transfer]
\label{lem:localized-weight-transfer}
For every $x,z\in\Omega$,
\begin{align}
  \norm{w_zg_x}_2&\le Ch^{-d/2}w_z(x),
  \label{eq:localized-weight-transfer-forward}\\
  \norm{w_z^{-1}g_x}_2&\le Ch^{-d/2}w_z(x)^{-1}.
  \label{eq:localized-weight-transfer-reverse}
\end{align}
\end{lemma}

\begin{proof}
The triangle inequality gives
\[
  \frac{w_z(y)}{w_z(x)}
  \le\left(1+\varepsilon_0\norm{x-y}/h\right)^\alpha,
\]
and the same bound with $x,y$ exchanged.  Multiply these inequalities by the kernel row $g_x(y)$ and use the weighted row bound in \eqref{eq:localized-G-weighted-rows}, centered at $x$.
\end{proof}

\begin{proposition}[Empirical weighted inverse]
\label{prop:localized-empirical-inverse}
Let $\mathcal Z_h$ be an $h$-net of $\Omega$ with $|\mathcal Z_h|\le Ch^{-d}$.  There are constants $C_0,C_1,C_2>0$ such that, if
\begin{equation}
  nh^d\ge C_0q_h(t),
  \label{eq:localized-sample-regime}
\end{equation}
then with probability at least $1-C_1e^{-t}$,
\begin{align}
  \norm{H_n^{-1}}_{\op}&\le C_2,
  \label{eq:localized-unweighted-inverse}\\
  \max_{z\in\mathcal Z_h}
  \norm{M_{w_z}H_n^{-1}M_{w_z}^{-1}}_{\op}&\le C_2.
  \label{eq:localized-weighted-empirical-inverse}
\end{align}
The same $C_2$ bounds $\max_{z\in\mathcal Z_h}\norm{M_{w_z}^{-1}H_n^{-1}M_{w_z}}_{\op}$.
\end{proposition}

\begin{proof}
Adjoin a formal center $\star$ with $w_\star\equiv1$.  For $z\in\mathcal Z_h\cup\{\star\}$, define
\[
  a_{z,x}:=\frac{w_zg_x}{w_z(x)},
  \qquad
  b_{z,x}:=w_z(x)w_z^{-1}g_x,
  \qquad
  c_{z,x}:=a_{z,x}\oplus b_{z,x}.
\]
Lemma~\ref{lem:localized-weight-transfer} and \eqref{eq:localized-G-row} give
\[
  \norm{c_{z,x}}_{L^2\oplus L^2}^2\le Ch^{-d}.
  \tag{\thesection.4}
\]
Let $C_z:=\EE(c_{z,X}\otimes c_{z,X})$ and write $S_z=M_{w_z}GM_{w_z}^{-1}$.  The two diagonal blocks of $C_z$ are
\[
\begin{aligned}
  \EE(a_{z,X}\otimes a_{z,X})
  &=M_{w_z}GM_{p w_z^{-2}}GM_{w_z}=S_zM_pS_z^\ast,\\
  \EE(b_{z,X}\otimes b_{z,X})
  &=M_{w_z}^{-1}GM_{p w_z^2}GM_{w_z}^{-1}=S_z^\ast M_pS_z.
\end{aligned}
\]
Thus, for $u,v\in L^2(\Omega)$,
\[
  \EE\abs{\ip{u}{a_{z,X}}+\ip{v}{b_{z,X}}}^2
  \le2\ip{u}{S_zM_pS_z^\ast u}
    +2\ip{v}{S_z^\ast M_pS_zv}.
\]
The conjugation bound for $G$ and $p\le p_+$ therefore give $\norm{C_z}_{\op}\le C$.  Conversely, $p\ge p_-$, $\norm{G}_{\op}=(1+\eta)^{-1/2}\ge2^{-1/2}$, and \eqref{eq:localized-G-commutator} imply $\norm{S_z}_{\op}\ge c>0$ (with the same conclusion for $z=\star$).  Since $S_zM_pS_z^\ast\succeq p_-S_zS_z^\ast$ is a principal compression of $C_z$,
\[
  \norm{C_z}_{\op}\ge p_-\norm{S_z}_{\op}^2\ge c.
\]
Finally, (\thesection.4) gives
\[
  \operatorname{tr}(C_z)=\EE\norm{c_{z,X}}_2^2\le Ch^{-d}.
  \tag{\thesection.5}
\]

For completeness, let $(Q_m)$ be an increasing sequence of finite-dimensional wavelet projections on $L^2(\Omega)$, and use the same notation for $Q_m\oplus Q_m$.  The covariance of $Q_mc_{z,X}$ is $C_{z,m}=Q_mC_zQ_m$, and $C_{z,m}\to C_z$ in operator norm.  Hence, for all sufficiently large $m$, its operator norm is bounded below by $c/2$, its trace is at most $Ch^{-d}$, and the envelope in (\thesection.4) is unchanged.  Put
\[
  Y_{i,m}:=(Q_mc_{z,X_i})\otimes(Q_mc_{z,X_i})-C_{z,m}.
\]
For these centered truncated covariance summands, with $R=Ch^{-d}$,
\[
  \EE Y_{i,m}^2\preceq R C_{z,m}.
\]
After enlarging $R$ if necessary, $\norm{Y_{i,m}}_{\op}\le R$.  Thus $V_m:=nR C_{z,m}$ is a valid variance proxy for $\sum_{i=1}^nY_{i,m}$, and
\[
  \operatorname{intdim}(V_m)
  =\frac{\operatorname{tr}(C_{z,m})}{\norm{C_{z,m}}_{\op}}
  \le Ch^{-d}.
\]
The finite-dimensional intrinsic-dimension covariance Bernstein inequality therefore gives the same bound uniformly in $m$.  Apply it for $u\ge1$, and then let $m\to\infty$.  The empirical covariances converge in operator norm for each fixed sample, the population covariances converge in trace norm, and the norm inequality is preserved in the limit.  Thus, for a fixed $z$, except with probability at most $Ch^{-d}e^{-u}$,
\[
  \norm{M_{w_z}(H_n-H)M_{w_z}^{-1}}_{\op}
  \le C\left\{\sqrt{\frac{u}{nh^d}}+\frac{u}{nh^d}\right\}.
  \tag{\thesection.6}
\]
The $(1,2)$ block of the augmented covariance difference is
\[
  (P_n-P)(a_{z,X}\otimes b_{z,X})
  =M_{w_z}(P_n-P)(g_X\otimes g_X)M_{w_z}^{-1},
\]
Thus this block equals $cM_{w_z}(H_n-H)M_{w_z}^{-1}$.  The factor $c^{-1}$ is absorbed into $C$ above; see \citep[Theorem~7.3.1]{Tropp2015}.  A union bound over $\mathcal Z_h\cup\{\star\}$ with $u=t+2d\log(e/h)+C$ turns (\thesection.6) into
\[
  \max_{z\in\mathcal Z_h\cup\{\star\}}
  \norm{M_{w_z}(H_n-H)M_{w_z}^{-1}}_{\op}
  \le C\left\{\sqrt{\frac{q_h(t)}{nh^d}}+\frac{q_h(t)}{nh^d}\right\}.
  \tag{\thesection.7}
\]
Under \eqref{eq:localized-sample-regime}, enlarge $C_0$ so the right side of (\thesection.7), after multiplication by the corresponding population inverse in \eqref{eq:localized-H-weighted-inverse}, is at most $1/2$.  The identity
\[
  M_{w_z}H_n^{-1}M_{w_z}^{-1}
  =
  \left[I+
  (M_{w_z}HM_{w_z}^{-1})^{-1}
  M_{w_z}(H_n-H)M_{w_z}^{-1}\right]^{-1}
  (M_{w_z}HM_{w_z}^{-1})^{-1}
\]
and the Neumann series prove \eqref{eq:localized-weighted-empirical-inverse}; $z=\star$ gives \eqref{eq:localized-unweighted-inverse}.  Finally,
\[
  \bigl(M_{w_z}H_n^{-1}M_{w_z}^{-1}\bigr)^*
  =M_{w_z}^{-1}H_n^{-1}M_{w_z},
\]
so self-adjointness of $H_n^{-1}$ gives the asserted bound with the weights reversed.
\end{proof}

\subsection{Design fluctuation}

\begin{lemma}[Weighted concentration of the residual vector]
\label{lem:localized-residual-vector}
Let $r\in L^\infty(\Omega)$ be fixed independently of $X_1,\ldots,X_n$.  Under \eqref{eq:localized-sample-regime}, with probability at least $1-Ce^{-t}$,
\begin{equation}
  \max_{z\in\mathcal Z_h}\norm{w_z^{-1}\xi_r}_2
  \le C\norm r_\infty
  \left\{
    \sqrt{\frac{q_h(t)}n}
    +h^{-d/2}\frac{q_h(t)}n
  \right\}.
  \label{eq:localized-residual-vector-bound}
\end{equation}
For one fixed center $z$, the right side remains valid with $q_h(t)$ replaced by $t$.
\end{lemma}

\begin{proof}
For fixed $z$, let
\[
  Z_z(X):=c^{-1}r(X)w_z^{-1}g_X.
\]
Lemma~\ref{lem:localized-weight-transfer} and the volume estimate used in Lemma~\ref{lem:localized-half-resolvent} imply
\[
  \norm{Z_z(X)}_2\le C\norm r_\infty h^{-d/2},
  \qquad
  \EE\norm{Z_z(X)}_2^2\le C\norm r_\infty^2.
  \tag{\thesection.8}
\]
The Hilbert-space Bernstein inequality for centered independent vectors with envelope $L$ and second moment $v^2$ gives
\[
  \PP\left\{\left\|\frac1n\sum_{i=1}^nU_i\right\|_2>
  C\left(v\sqrt{\frac un}+L\frac un\right)\right\}
  \le2e^{-u}.
  \tag{\thesection.9}
\]
This inequality follows, for example, from the martingale inequality of \citep[Theorem~3.4]{Pinelis1994}.  For $U_i=Z_z(X_i)-\EE Z_z(X)$, (\thesection.8) bounds the second moment and gives an envelope at most twice the displayed one.  Apply (\thesection.9) and take a union bound over $|\mathcal Z_h|\le Ch^{-d}$.  Taking $u=t+d\log(e/h)+C$ proves \eqref{eq:localized-residual-vector-bound}; no union bound is needed for a fixed center.
\end{proof}

\begin{proposition}[Design fluctuation]
\label{prop:localized-design-fluctuation}
Assume \eqref{eq:localized-sample-regime}.  Before observing the design,
fix $x_0\in(0,1)^d$.  There is an event of probability at least
$1-Ce^{-t}$ on which
\begin{align}
  \abs{D_{n,\lambda}(x_0)}
  &\le C\norm{r_\lambda}_\infty
  \left\{
    \sqrt{\frac{t}{nh^d}}+\frac{t}{nh^d}
  \right\},
  \label{eq:localized-design-pointwise}\\
  \norm{D_{n,\lambda}}_\infty
  &\le C\norm{r_\lambda}_\infty
  \left\{
    \sqrt{\frac{q_h(t)}{nh^d}}+\frac{q_h(t)}{nh^d}
  \right\}.
  \label{eq:localized-design-sup}
\end{align}
\end{proposition}

\begin{proof}
If necessary, add the fixed point $x_0$ to $\mathcal Z_h$; its cardinality remains of order $h^{-d}$.  Intersect the inverse event of Proposition~\ref{prop:localized-empirical-inverse}, the net event of Lemma~\ref{lem:localized-residual-vector}, and the event for the fixed center $x_0$.

For arbitrary $x\in\Omega$, choose $z(x)\in\mathcal Z_h$ with $\norm{x-z(x)}\le h$.  Then $w_{z(x)}(x)\le(1+\varepsilon_0)^\alpha$ and Lemma~\ref{lem:localized-weight-transfer} yields $\norm{w_{z(x)}g_x}_2\le Ch^{-d/2}$.  Hence
\[
\begin{split}
  \abs{D_{n,\lambda}(x)}
  &=\abs{\ip{g_x}{H_n^{-1}\xi_{r_\lambda}}}\\
  &\le\norm{w_{z(x)}g_x}_2
  \norm{M_{w_{z(x)}}^{-1}H_n^{-1}M_{w_{z(x)}}}_{\op}
  \norm{w_{z(x)}^{-1}\xi_{r_\lambda}}_2.
\end{split}
\]
Taking the supremum and using \eqref{eq:localized-residual-vector-bound} proves \eqref{eq:localized-design-sup}.  At $x_0$, instead use the version of the residual vector bound for a fixed center, which proves \eqref{eq:localized-design-pointwise}.
\end{proof}

\subsection{Noise fluctuation}

\begin{lemma}[Fixed-point noise]
\label{lem:localized-noise-pointwise}
On the event \eqref{eq:localized-unweighted-inverse}, for every fixed $x_0\in\Omega$ and $t\ge1$,
\[
  \PP_\epsilon\left(
    \abs{N_{n,\lambda}(x_0)}
    >C\sigma\sqrt{\frac{t}{nh^d}}
    \,\middle|\,X_1,\ldots,X_n
  \right)\le2e^{-t}.
\]
\end{lemma}

\begin{proof}
Put $\widehat C=(cn)^{-1}\sum_i g_{X_i}\otimes g_{X_i}$, so $H_n=I-G^2+\widehat C$ and $\widehat C\preceq H_n$.  Conditional on the design, $N_{n,\lambda}(x_0)$ is a linear combination of the independent conditionally sub-Gaussian noises.  Its squared proxy is at most
\[
\begin{split}
  \frac{\sigma^2}{c^2n^2}\sum_{i=1}^n
  \ip{H_n^{-1}g_{x_0}}{g_{X_i}}^2
  &=\frac{\sigma^2}{cn}
  \ip{H_n^{-1}g_{x_0}}{\widehat C H_n^{-1}g_{x_0}}\\
  &\le\frac{\sigma^2}{cn}
  \ip{g_{x_0}}{H_n^{-1}g_{x_0}}
  \le\frac{C\sigma^2}{nh^d}.
\end{split}
\]
The last inequality uses \eqref{eq:localized-unweighted-inverse} and \eqref{eq:localized-G-row}.  The claimed conditional tail bound follows.
\end{proof}

\begin{lemma}[Uniform noise]
\label{lem:localized-noise-uniform}
On the event \eqref{eq:localized-unweighted-inverse}, for every $t\ge1$,
\[
  \PP_\epsilon\left(
    \norm{N_{n,\lambda}}_\infty
    >C\sigma\sqrt{\frac{q_h(t)}{nh^d}}
    \,\middle|\,X_1,\ldots,X_n
  \right)\le Ce^{-t}.
\]
\end{lemma}

\begin{proof}
Let $A_{j,k}:=\ip{N_{n,\lambda}}{\psi_{j,k}}_{L^2}$.  Repeating the preceding proxy calculation with $G\psi_{j,k}=\gamma_j\psi_{j,k}$ gives
\[
  \operatorname{proxy}(A_{j,k})^2
  \le\frac{C\sigma^2}{n}\gamma_j^2.
\]
Apply the conditional sub-Gaussian tail bound at level
\[
  u_j=t+(d\log2+2)(j+1)+C
\]
and take a union bound over $k\in\Lambda_j$ and all $j$.  With conditional probability at least $1-Ce^{-t}$,
\[
  \max_{k\in\Lambda_j}\abs{A_{j,k}}
  \le\frac{C\sigma}{\sqrt n}\gamma_j\sqrt{t+j+1}
  \qquad\text{for every }j\ge0.
  \tag{\thesection.10}
\]
Finite overlap gives
\[
  \norm{P_jN_{n,\lambda}}_\infty
  \le C2^{jd/2}\max_{k\in\Lambda_j}\abs{A_{j,k}}.
\]
Choose $J_h$ with $2^{J_h}\asymp h^{-1}$.  The contribution of the low levels is controlled by its largest term.  For $j>J_h$,
\[
  \gamma_j\le C2^{-\nu(j-J_h)/2}.
\]
Because $\nu>d$,
\[
  \sum_{j\ge0}2^{jd/2}\gamma_j\sqrt{t+j+1}
  \le Ch^{-d/2}\sqrt{t+\log(e/h)}.
  \tag{\thesection.11}
\]
Combining (\thesection.10)--(\thesection.11) proves the displayed bound.  The same estimates show that the wavelet series of $N_{n,\lambda}$ converges absolutely and uniformly to its continuous RKHS representative, so the supremum is well defined.
\end{proof}

\begin{proof}[Proof of Lemma~\ref{lem:localized-pointwise-concentration}]
Intersect the design event in Proposition~\ref{prop:localized-design-fluctuation} with the conditional noise event in Lemma~\ref{lem:localized-noise-pointwise}.  The intersection has probability at least $1-Ce^{-t}$.  On this event,
\[
  \widehat f_\lambda-f_\lambda=D_{n,\lambda}+N_{n,\lambda}
\]
and the preceding bounds give \eqref{eq:localized-pointwise-concentration}.  Proposition~\ref{prop:localized-population-bias}, the sample-size condition, and the triangle inequality imply the stated consequence.
\end{proof}

\begin{proof}[Proof of Lemma~\ref{lem:localized-uniform-concentration}]
Intersect Proposition~\ref{prop:localized-design-fluctuation} with Lemma~\ref{lem:localized-noise-uniform}.  The master decomposition gives \eqref{eq:localized-uniform-concentration}; Proposition~\ref{prop:localized-population-bias} and the triangle inequality yield the consequence.
\end{proof}

\begin{proof}[Proof of Theorem~\ref{thm:localized-high-probability} when $m=0$]
The preceding lemmas yield \eqref{eq:localized-pointwise-bound} and \eqref{eq:localized-uniform-bound}, respectively, with $h=h_\lambda$ and $q_h=q_{h_\lambda}$.  A final union bound gives probability at least $1-Ce^{-t}$.
\end{proof}

\subsection{From high-probability bounds to squared risks}

\begin{lemma}[Moment bound for the tail event]
\label{lem:localized-moment-upgrade}
There are constants $c,C>0$ such that, whenever
\[
  m:=nh^d\ge C\log(e/h),
\]
the following bounds hold uniformly over $f^\ast\in\FF_s(F)$, all admissible design densities, and the fixed point $x_0$:
\begin{align}
  \EE\abs{\widehat f_\lambda(x_0)-f^\ast(x_0)}^2
  &\le C\left\{
    F^2h^{2s}+\frac{\sigma^2}{m}
    +(F^2+\sigma^2)\lambda^{-1}e^{-cm}
  \right\},
  \label{eq:localized-pointwise-moment}\\
  \EE\norm{\widehat f_\lambda-f^\ast}_\infty^2
  &\le C\left\{
    F^2h^{2s}+\frac{\sigma^2\log(e/h)}{m}
    +(F^2+\sigma^2)\lambda^{-1}e^{-cm}
  \right\}.
  \label{eq:localized-uniform-moment}
\end{align}
\end{lemma}

\begin{proof}
First, the diagonal of $K_\nu$ is uniformly bounded because $\nu>d$ and Assumption~\ref{ass:localized-wavelet} holds.  Let $\kappa_K$ be a uniform bound for $K_\nu(x,x)^{1/2}$.  Comparing the KRR objective at $\widehat f_\lambda$ and at $0$ gives
\[
  \lambda\norm{\widehat f_\lambda}_{\HH_\nu}^2
  \le\frac1n\sum_{i=1}^nY_i^2.
\]
Since $\norm{\widehat f_\lambda}_\infty\le\kappa_K\norm{\widehat f_\lambda}_{\HH_\nu}$ and the Besov embedding gives $\norm{f^\ast}_\infty\le CF$,
\[
  \norm{\widehat f_\lambda-f^\ast}_\infty^4
  \le C\left\{
    \lambda^{-2}\left(\frac1n\sum_{i=1}^nY_i^2\right)^2+F^4
  \right\}.
\]
Conditional sub-Gaussianity yields
\[
  \EE Y_i^4\le C(F^2+\sigma^2)^2,
  \qquad
  \EE\left(\frac1n\sum_{i=1}^nY_i^2\right)^2
  \le\frac1n\sum_{i=1}^n\EE Y_i^4
  \le C(F^2+\sigma^2)^2.
\]
Using $\lambda\le c_Kc_p$ to absorb the last $F^4$ term, we obtain
\[
  \EE\norm{\widehat f_\lambda-f^\ast}_\infty^4
  \le C\lambda^{-2}(F^2+\sigma^2)^2.
  \tag{\thesection.12}
\]

Choose $\delta\le(4C_0)^{-1}$ and put $T=\delta m$.  After increasing the constant in the premise to ensure $m\ge4C_0\log(e/h)$ and $T\ge1$, the hypotheses of Lemmas~\ref{lem:localized-pointwise-concentration} and \ref{lem:localized-uniform-concentration} hold for each $1\le t\le T$, because $C_0\{t+\log(e/h)\}\le m$.  Put
\[
  A_{\mathrm{pt}}:=CFh^s,\qquad
  A_\infty:=C\left\{Fh^s+\sigma\sqrt{\frac{\log(e/h)}m}\right\},
  \qquad
  B:=C\sigma m^{-1/2}.
\]
The two concentration lemmas, together with $\sqrt{t+\log(e/h)}\le\sqrt t+\sqrt{\log(e/h)}$, show that, for $1\le t\le T$,
\[
  \PP\{R_{\mathrm{pt}}>A_{\mathrm{pt}}+B\sqrt t\}
  +\PP\{R_\infty>A_\infty+B\sqrt t\}
  \le Ce^{-t},
  \tag{\thesection.13}
\]
where $R_{\mathrm{pt}}=\abs{\widehat f_\lambda(x_0)-f^\ast(x_0)}$ and $R_\infty=\norm{\widehat f_\lambda-f^\ast}_\infty$.

For either pair $(R,A)=(R_{\mathrm{pt}},A_{\mathrm{pt}})$ or $(R_\infty,A_\infty)$, the tail integral formula and (\thesection.13) give
\[
\begin{split}
  \EE\left[(R-A)_+^2\one\{R\le A+B\sqrt T\}\right]
  &\le B^2+B^2\int_1^T
    \PP\{R>A+B\sqrt t\}\,\dd t\\
  &\le CB^2.
\end{split}
\]
On the remaining event, (\thesection.12), Cauchy--Schwarz, and (\thesection.13) at $t=T$ give
\[
  \EE\left[R^2\one\{R>A+B\sqrt T\}\right]
  \le C\lambda^{-1}(F^2+\sigma^2)e^{-cT}.
\]
Finally, $R^2\le2A^2+2(R-A)_+^2$ on the truncated event.  Combining the last three displays proves \eqref{eq:localized-pointwise-moment} and \eqref{eq:localized-uniform-moment}.
\end{proof}

\begin{proof}[Proof of Theorem~\ref{thm:localized-minimax} when $m=0$]
For $\lambda_{\mathrm{pt}}=a_{\mathrm{pt}}n^{-1}$,
\[
  h_{\lambda_{\mathrm{pt}}}\asymp n^{-1/\nu},
  \qquad
  m_{\mathrm{pt}}:=n h_{\lambda_{\mathrm{pt}}}^d\asymp n^{2s/\nu},
  \qquad
  \log(e/h_{\lambda_{\mathrm{pt}}})\asymp\log n.
\]
Hence $m_{\mathrm{pt}}/\log(e/h_{\lambda_{\mathrm{pt}}})\to\infty$, so Lemma~\ref{lem:localized-moment-upgrade} applies for all sufficiently large $n$.  The first bound has bias and variance terms of the same order, $n^{-2s/\nu}$, while
\[
  \lambda_{\mathrm{pt}}^{-1}e^{-cm_{\mathrm{pt}}}
  =o\bigl(n^{-2s/\nu}\bigr).
\]
Equation \eqref{eq:localized-pointwise-moment} therefore yields
\[
  \sup_{P\in\mathcal P(s,F,\sigma,p_-,p_+)}
  \EE_P\abs{\widehat f_{\lambda_{\mathrm{pt}}}(x_0)-f_P^\ast(x_0)}^2
  \le Cn^{-2s/\nu}.
\]
For $\lambda_\infty=a_\infty(\log n)/n$,
\[
  h_{\lambda_\infty}\asymp\left(\frac{\log n}{n}\right)^{1/\nu},
  \qquad
  m_\infty:=n h_{\lambda_\infty}^d
  \asymp n^{2s/\nu}(\log n)^{d/\nu},
  \qquad
  \log(e/h_{\lambda_\infty})\asymp\log n.
\]
Thus
\[
  \frac{m_\infty}{\log(e/h_{\lambda_\infty})}
  \asymp\left(\frac n{\log n}\right)^{2s/\nu}\longrightarrow\infty,
  \qquad
  \frac{\log(e/h_{\lambda_\infty})}{m_\infty}
  \asymp\left(\frac{\log n}{n}\right)^{2s/\nu}.
\]
The squared bias has the same order, and
\[
  \lambda_\infty^{-1}e^{-cm_\infty}
  =o\left\{\left(\frac{\log n}{n}\right)^{2s/\nu}\right\}.
\]
Equation \eqref{eq:localized-uniform-moment} gives
\[
  \sup_{P\in\mathcal P(s,F,\sigma,p_-,p_+)}
  \EE_P\norm{\widehat f_{\lambda_\infty}-f_P^\ast}_\infty^2
  \le C\left(\frac{\log n}{n}\right)^{2s/\nu}.
\]
Since $\nu=2s+d$, these are precisely the two rates stated in the theorem.
\end{proof}

\section{Derivative bounds for the localized wavelet kernel}
\label{app:derivative-losses}

Fix $\alpha\in\mathbb N_0^d$, put $m=|\alpha|<s$ and
$R=s+d/2=\nu/2$, and retain $c,\widetilde p,\eta,h,G,H_n$
from Appendix~\ref{app:localized-kernel}.  We prove the derivative
versions of Theorems~\ref{thm:localized-high-probability} and
\ref{thm:localized-minimax} without changing the fitted kernel.

\subsection{Derivative scale estimates}

\begin{lemma}[Differentiated wavelet and projection bounds]
\label{lem:derivative-wavelet-scales}
Under Assumption~\ref{ass:localized-wavelet},
\begin{align}
  \norm{D^\alpha\psi_{j,k}}_\infty
  &\le C2^{j(m+d/2)},&
  \sup_x\sum_k|D^\alpha\psi_{j,k}(x)|^2
  &\le C2^{j(d+2m)},
  \label{eq:derivative-wavelet-scales}\\
  \norm{D^\alpha P_jv}_\infty
  &\le C2^{jm}\norm{P_jv}_\infty
   \le C2^{jm}\norm v_\infty.
  \label{eq:derivative-block-bernstein}
\end{align}
The kernel $e_j$ of the cumulative projection $E_j$ satisfies
\begin{equation}
  |D_x^\alpha e_j(x,y)|
  \le C2^{j(d+m)}\one\{\norm{x-y}\le C2^{-j}\},
  \qquad
  \norm{D^\alpha E_j}_{L^\infty\to L^\infty}\le C2^{jm}.
  \label{eq:derivative-cumulative-projection}
\end{equation}
\end{lemma}
\begin{proof}
Because $R>m+d/2$, Sobolev interpolation on the cube gives
\begin{equation}
  \norm{D^\alpha v}_\infty
  \le C\norm v_2^{1-\theta}\norm v_{H^R}^{\theta},
  \qquad \theta=(m+d/2)/R<1.
  \label{eq:derivative-gn}
\end{equation}
For clarity, apply a common bounded extension in $L^2$ and $H^R$
and split Fourier inversion at frequency $B\ge1$.  Cauchy--Schwarz
bounds the two parts by
$CB^{m+d/2}\norm v_2$ and
$CB^{m+d/2-R}\norm v_{H^R}$, respectively.  Optimizing $B$
proves \eqref{eq:derivative-gn}.
The coefficient characterization gives
$\norm{\psi_{j,k}}_{H^R}\lesssim2^{jR}$ and
$\norm{\psi_{j,k}}_2=1$.  Applying \eqref{eq:derivative-gn} and
finite overlap proves \eqref{eq:derivative-wavelet-scales}.
Differentiation does not enlarge support, and
$\norm{\psi_{j,k}}_1\lesssim2^{-jd/2}$.
Use these facts on the wavelet coefficients of $P_jv$ to prove
\eqref{eq:derivative-block-bernstein}.

For each $y$, the coefficient and square-sum bounds imply
$\norm{e_j(\cdot,y)}_2\lesssim2^{jd/2}$ and
$\norm{e_j(\cdot,y)}_{H^R}\lesssim2^{j(R+d/2)}$.
Equation~\eqref{eq:derivative-gn} then bounds its derivative by
$C2^{j(d+m)}$.  The original finite propagation support persists.
Integrating in $y$ proves the last assertion.
\end{proof}

\subsection{Population derivative bias}

\begin{lemma}[Bias with bounded measurable design density]
\label{lem:derivative-population-bias}
For $r_\lambda=f^\ast-f_\lambda$,
\begin{equation}
  \norm{D^\alpha r_\lambda}_\infty\le CFh^{s-m}.
  \label{eq:derivative-population-bias}
\end{equation}
\end{lemma}
\begin{proof}
Proposition~\ref{prop:localized-population-bias} gives
$\norm{r_\lambda}_\infty\lesssim Fh^s$.
The normalized population equation is
$\eta A f_\lambda=M_{\widetilde p}r_\lambda$.
Applying a wavelet block yields the exact identity
\begin{equation}
  P_jf_\lambda
  =\eta^{-1}2^{-\nu j}P_j(\widetilde p\,r_\lambda).
  \label{eq:derivative-block-equation}
\end{equation}
This does not commute $M_{\widetilde p}$ through $A$.
Choose $J=\lfloor\log_2(1/h)\rfloor$ and write
\[
  D^\alpha r_\lambda
  =D^\alpha E_Jr_\lambda
    +D^\alpha(I-E_J)f^\ast-D^\alpha(I-E_J)f_\lambda.
\]
The first term is bounded by $C2^{Jm}Fh^s\lesssim Fh^{s-m}$.
The Besov characterization and \eqref{eq:derivative-block-bernstein}
bound the second by
$CF\sum_{j>J}2^{-j(s-m)}\lesssim Fh^{s-m}$.
For the third, use \eqref{eq:derivative-block-equation},
$\eta=h^\nu$ and $\norm{\widetilde p}_\infty\le1$ to obtain
\[
  \sum_{j>J}\norm{D^\alpha P_jf_\lambda}_\infty
  \le C h^{-\nu}\norm{r_\lambda}_\infty
       \sum_{j>J}2^{-j(\nu-m)}
  \le CFh^{s-m}.
\]
These derivative series converge uniformly because $m<s<\nu$,
which also justifies differentiating the reconstructions.
\end{proof}

\subsection{Derivative evaluation rows and empirical fluctuations}

Define the differentiated half-resolvent row in $L^2(\Omega)$ by
\[
  g_x^{(\alpha)}:=\sum_{j,k}\gamma_j
       D^\alpha\psi_{j,k}(x)\psi_{j,k},
  \qquad \gamma_j=(1+(h2^j)^\nu)^{-1/2}.
\]
Orthogonality and \eqref{eq:derivative-wavelet-scales} show that
\begin{equation}
  \sup_x\norm{g_x^{(\alpha)}}_2^2
  \le C\sum_{j\ge0}\frac{2^{j(d+2m)}}{1+(h2^j)^\nu}
  \le Ch^{-d-2m},
  \label{eq:derivative-row-l2}
\end{equation}
since $\nu-d-2m=2(s-m)>0$.  The uniform square-summable tails
also show that $x\mapsto g_x^{(\alpha)}$ is $L^2$-continuous and
$D^\alpha(Gu)(x)=\ip{g_x^{(\alpha)}}u_2$.

Choose a scalar exponent $d/2<\zeta<\nu$ for the weights in
Lemma~\ref{lem:localized-half-resolvent}, and fix the small
$\varepsilon_0$ used in the empirical inverse proposition.  Put
$w_z(y)=(1+\varepsilon_0\norm{y-z}/h)^\zeta$.
Differentiate the Abel expansion $G=\sum_j a_jE_j$ away from its
diagonal.  With $\rho=\norm{x-y}\ge h$, only $2^j\lesssim\rho^{-1}$
contribute, and $a_j\lesssim(h2^j)^\nu$ gives
\[
  |g_x^{(\alpha)}(y)|
  \le Ch^\nu\sum_{2^j\lesssim\rho^{-1}}2^{j(d+\nu+m)}
  \le Ch^\nu\rho^{-d-\nu-m}.
\]
Use \eqref{eq:derivative-row-l2} within distance $h$ and integrate
this last estimate on dyadic annuli outside.  Since $\zeta<\nu$,
\begin{equation}
  \norm{w_xg_x^{(\alpha)}}_2\le Ch^{-d/2-m},
  \qquad
  \norm{w_zg_x^{(\alpha)}}_2\le Ch^{-d/2-m}w_z(x).
  \label{eq:derivative-weighted-rows}
\end{equation}
The second inequality uses the triangle inequality for weights.

The exact decomposition \eqref{eq:localized-master-decomposition} is
\[
  \widehat f_\lambda-f^\ast=-r_\lambda+D_{n,\lambda}+N_{n,\lambda},
  \quad
  D_{n,\lambda}=GH_n^{-1}\xi_{r_\lambda},
  \quad
  N_{n,\lambda}=GH_n^{-1}\frac1{cn}\sum_i\epsilon_i g_{X_i}.
\]
Only the output evaluation row is differentiated.  In particular,
$H_n$, its concentration event, and the residual vector are unchanged.
For any center $z$,
\[
  |D^\alpha D_{n,\lambda}(x)|
  \le\norm{w_zg_x^{(\alpha)}}_2
      \norm{M_{w_z}^{-1}H_n^{-1}M_{w_z}}_{\op}
      \norm{w_z^{-1}\xi_{r_\lambda}}_2.
\]
Proposition~\ref{prop:localized-empirical-inverse},
Lemma~\ref{lem:localized-residual-vector}, and
\eqref{eq:derivative-weighted-rows} therefore imply, with probability
at least $1-Ce^{-t}$ under $nh^d\ge C_0q_h(t)$,
\begin{align}
 |D^\alpha D_{n,\lambda}(x_0)|
 &\le C\norm{r_\lambda}_\infty h^{-m}
       \left\{\sqrt{\frac{t}{nh^d}}+\frac{t}{nh^d}\right\},
       \label{eq:derivative-design-pointwise}\\
 \norm{D^\alpha D_{n,\lambda}}_\infty
 &\le C\norm{r_\lambda}_\infty h^{-m}
       \left\{\sqrt{\frac{q_h(t)}{nh^d}}
                        +\frac{q_h(t)}{nh^d}\right\}.
       \label{eq:derivative-design-uniform}
\end{align}
For the uniform bound take an $h$-net and a center within distance
$h$ of each $x$, so $w_z(x)\lesssim1$.  For the pointwise bound use
the single center $z=x_0$.  The net in the inverse event changes its
probability requirement but does not put a spatial logarithm into
the single residual-vector amplitude.

Conditionally on the design, put
$\widehat C=(cn)^{-1}\sum_i g_{X_i}\otimes g_{X_i}\preceq H_n$.
The sub-Gaussian variance proxy of $D^\alpha N_{n,\lambda}(x)$ is
bounded by
\[
  \frac{\sigma^2}{cn}
  \ip{H_n^{-1}g_x^{(\alpha)}}{\widehat C H_n^{-1}g_x^{(\alpha)}}
  \le\frac{\sigma^2}{cn}
       \ip{g_x^{(\alpha)}}{H_n^{-1}g_x^{(\alpha)}}
  \le\frac{C\sigma^2}{nh^{d+2m}}
\]
on the unweighted inverse event.  This proves the pointwise noise
bound.  For uniform noise, use the same coefficient event as in
Lemma~\ref{lem:localized-noise-uniform}:
\[
 \max_k|\ip{N_{n,\lambda}}{\psi_{j,k}}_2|
 \le C\sigma n^{-1/2}\gamma_j\sqrt{t+j+1}
 \quad\text{for every }j.
\]
Finite overlap and \eqref{eq:derivative-wavelet-scales} imply
\begin{equation}
 \norm{D^\alpha N_{n,\lambda}}_\infty
 \le\frac{C\sigma}{\sqrt n}
     \sum_{j\ge0}2^{j(d/2+m)}\gamma_j\sqrt{t+j+1}
 \le C\sigma\sqrt{\frac{q_h(t)}{nh^{d+2m}}}.
 \label{eq:derivative-noise-uniform}
\end{equation}
Above the cutoff $2^J\asymp h^{-1}$ the summands decay geometrically
at rate $2^{-(s-m)(j-J)}$, up to their square-root factors.  This also
proves uniform convergence of the differentiated noise expansion.

\begin{proof}[Proof of Theorem~\ref{thm:localized-high-probability}]
Combine \eqref{eq:derivative-population-bias}, the design bounds,
and the two noise bounds.  Since $\norm{r_\lambda}_\infty\lesssim Fh^s$
and $nh^d\ge C_0q_h(t)$, the design terms are absorbed into
$CFh^{s-m}$.  A union bound gives the asserted common probability.
\end{proof}

\subsection{Expected squared derivative risks}

For the fixed kernel, derivative evaluation is uniformly bounded:
\[
 \sup_x c_K\sum_{j,k}2^{-\nu j}|D^\alpha\psi_{j,k}(x)|^2
 \le C\sum_{j\ge0}2^{-2(s-m)j}<\infty.
\]
Thus $\norm{D^\alpha v}_\infty\lesssim\norm v_{\HH_\nu}$.
The KRR objective inequality and the fourth moment of the responses
give, uniformly over the distribution class,
\begin{equation}
 \EE\norm{D^\alpha\widehat f_\lambda-D^\alpha f^\ast}_\infty^4
 \le C\lambda^{-2}(F^2+\sigma^2)^2.
 \label{eq:derivative-fourth-moment}
\end{equation}
Here $\norm{D^\alpha f^\ast}_\infty\lesssim F$ uses $m<s$.
Integrate the high-probability inequalities over
$1\le t\le T=c_1nh^d$, with $c_1>0$ small enough that their sample-size
condition holds throughout.  As in
Lemma~\ref{lem:localized-moment-upgrade}, integrate the truncated
sub-Gaussian tail and use Cauchy--Schwarz with
\eqref{eq:derivative-fourth-moment} on the remaining event.  For
$nh^d\ge C\log(e/h)$ this yields
\begin{align}
 \EE|D^\alpha\widehat f_\lambda(x_0)-D^\alpha f^\ast(x_0)|^2
 &\le C\left\{F^2h^{2(s-m)}+\frac{\sigma^2}{nh^{d+2m}}
 +(F^2+\sigma^2)\lambda^{-1}e^{-c_0nh^d}\right\},
 \label{eq:derivative-pointwise-moment}\\
 \EE\norm{D^\alpha\widehat f_\lambda-D^\alpha f^\ast}_\infty^2
 &\le C\left\{F^2h^{2(s-m)}
 +\frac{\sigma^2\log(e/h)}{nh^{d+2m}}
 +(F^2+\sigma^2)\lambda^{-1}e^{-c_0nh^d}\right\}.
 \label{eq:derivative-uniform-moment}
\end{align}

\begin{proof}[Proof of Theorem~\ref{thm:localized-minimax}]
For $\lambda=a_{\rm pt}/n$, $h\asymp n^{-1/(2s+d)}$; the first two
terms of \eqref{eq:derivative-pointwise-moment} both have order
$n^{-2(s-m)/(2s+d)}$.  For $\lambda=a_\infty(\log n)/n$,
$h\asymp(\log n/n)^{1/(2s+d)}$; the first two terms of
\eqref{eq:derivative-uniform-moment} both have order
$(\log n/n)^{2(s-m)/(2s+d)}$.  In both cases
$nh^d/\log(e/h)\to\infty$, and the exponential remainder is smaller
than every fixed negative power of $n$.  This proves both squared-risk
bounds.  The condition $m<s$ is essential for the derivative bias,
row square sum, noise series, and derivative RKHS embedding.
\end{proof}

\numberwithin{equation}{section}
\section{\texorpdfstring{Uniform \(L^\infty\) ridge stability via Sun's
Wiener lemma}{Uniform L-infinity ridge stability via Sun's Wiener lemma}}
\label{app:uniform-cdv-resolvent}

This appendix records, within the paper, the complete resolvent argument
used for the population part of Proposition~\ref{prop:localized-population-bias}.
It is independent of the half-resolvent localization used later for the
empirical process.  The only external inverse-localization input is the
norm-controlled Wiener theorem of Sun~\cite{Sun2007}; all reductions to that
theorem are proved below.

\subsection{CDV assumptions and the main theorem}

Let
\begin{equation}\label{cdv:eq:basic-parameters}
  \Omega=[0,1]^d,\qquad s>0,\qquad r=s+\frac d2>\frac d2.
\end{equation}
All Hilbert-space inner products are linear in the first variable. We use
a compactly supported tensor-product Cohen--Daubechies--Vial (CDV)
boundary wavelet system. The interval construction and its finite boundary
filters are described in \cite{CohenDaubechiesVial1993}. The precise Sobolev coefficient
equivalence needed here is retained as an explicit hypothesis rather than
inferred from an unspecified regularity threshold.

\begin{assumption}[CDV$(r)$]\label{cdv:ass:cdv}
After absorbing a fixed number of initial scales into $V_0$, the wavelet
system satisfies the following properties.
\begin{enumerate}[label=\textup{(\roman*)}]
\item There is an orthogonal multiresolution decomposition
\begin{equation}\label{cdv:eq:mra}
  V_{j+1}=V_j\oplus W_j,
  \qquad
  L^2(\Omega)=V_0\oplus\bigoplus_{j\geq0}W_j.
\end{equation}
The finite family $\{\phi_\nu\}_{\nu\in\mathcal I_0}$ is an orthonormal
basis of $V_0$. For every $j\geq0$, the families
\[
  \{\phi_{j,k,\tau}:(k,\tau)\in X_j\},
  \qquad
  \{\psi_{j,k,e,\tau}:(k,e,\tau)\in Y_j\}
\]
are orthonormal bases of $V_j$ and $W_j$, respectively.

\item There are finite type sets $\mathcal T_\phi$, $\mathcal T_\psi$ and
a finite orientation set $\mathcal E$, all independent of $j$, such that
\[
  X_j\subset\mathbb Z^d\times\mathcal T_\phi,
  \qquad
  Y_j\subset\mathbb Z^d\times\mathcal E\times\mathcal T_\psi.
\]
The type labels include the finitely many interior, face, edge, and corner
configurations produced by tensor products.

\item There is $A_\Psi<\infty$ such that every level-$j$ scaling function
or wavelet $h_{j,k,\vartheta}$ satisfies
\begin{equation}\label{cdv:eq:support-localization}
  \supp h_{j,k,\vartheta}
  \subset
  \Omega\cap 2^{-j}\bigl(k+[-A_\Psi,A_\Psi]^d\bigr).
\end{equation}

\item There is $B_\Psi<\infty$ such that
\begin{equation}\label{cdv:eq:L1-Linfty-scaling}
  \norm{h_{j,k,\vartheta}}_{L^\infty(dx)}
  \leq B_\Psi2^{jd/2},
  \qquad
  \norm{h_{j,k,\vartheta}}_{L^1(dx)}
  \leq B_\Psi2^{-jd/2}.
\end{equation}
At every $x\in\Omega$, at most $B_\Psi$ level-$j$ scaling functions and
at most $B_\Psi$ level-$j$ wavelets are nonzero.

\item If $j\geq\ell$, then a fixed level-$\ell$ support meets at most
$B_\Psi2^{d(j-\ell)}$ level-$j$ supports, whereas a fixed level-$j$
support meets at most $B_\Psi$ level-$\ell$ supports. These estimates
hold uniformly for all interior, face, edge, and corner types. Restriction
to the boundary can only reduce the number of available lattice indices.

\item The refinement and wavelet filters are finite. There are only
finitely many boundary refinement matrices, and their dimensions and
operator norms are uniformly bounded. Consequently, every fixed-depth
refinement has constants depending only on the depth and on the selected
wavelet family $\Psi$, not on the scale.

\item With $\Pi_0$ denoting the orthogonal projection onto $V_0$ and $Q_j$
the orthogonal projection onto $W_j$, one has the Sobolev characterization
\begin{equation}\label{cdv:eq:sobolev-characterization}
  \norm{u}_{H^r(\Omega)}^2
  \asymp_{r,\Psi}
  \norm{\Pi_0u}_2^2+
  \sum_{j\geq0}2^{2rj}\norm{Q_ju}_2^2.
\end{equation}
This equivalence is part of the assumption on the chosen CDV system.
\end{enumerate}
\end{assumption}

The symbol $\Psi$ includes all constants and finite type cardinalities in
Assumption~\ref{cdv:ass:cdv}. Every implicit constant below is independent of
$p,\lambda,J$, and $L$ unless an additional dependence is explicitly
stated.

Define
\begin{equation}\label{cdv:eq:kernel-definition}
\begin{split}
  K_s(x,y)
  ={}&
  \sum_{\nu\in\mathcal I_0}
  \phi_\nu(x)\overline{\phi_\nu(y)}
  \\
  &+
  \sum_{j\geq0}2^{-j(2s+d)}
  \sum_{k,e,\tau}
  \psi_{j,k,e,\tau}(x)
  \overline{\psi_{j,k,e,\tau}(y)}.
\end{split}
\end{equation}
Since $2r=2s+d$, the high-frequency coefficient is $2^{-2rj}$.
For $0<c\leq C<\infty$, let
\begin{equation}\label{cdv:eq:density-class}
  \mathcal P(c,C)
  =
  \left\{
  p:\ \int_\Omega p(x)\,dx=1,
  \quad c\leq p(x)\leq C\ \text{a.e.}
  \right\}.
\end{equation}
For $p\in\mathcal P(c,C)$, set $d\mu_p=p\,dx$ and define on $L^2(\mu_p)$
\begin{equation}\label{cdv:eq:Tp-definition}
  T_pf(x)=\int_\Omega K_s(x,y)f(y)\,d\mu_p(y).
\end{equation}
Because $p\geq c>0$, Lebesgue measure and $\mu_p$ have the same null sets.
Therefore
\begin{equation}\label{cdv:eq:Linfty-isometry}
  L^\infty(\mu_p)=L^\infty(dx)
  \quad\text{isometrically}.
\end{equation}

\begin{theorem}\label{thm:uniform-cdv-resolvent}
Under Assumption~\ref{cdv:ass:cdv}, there is a constant $M(d,s,c,C,\Psi)<\infty$ such
that
\begin{equation}\label{cdv:eq:main-bound}
  \boxed{
  \sup_{p\in\mathcal P(c,C)}
  \sup_{\lambda>0}
  \norm{
  T_p(T_p+\lambda I)^{-1}
  }_{L^\infty(\mu_p)\to L^\infty(\mu_p)}
  \leq M(d,s,c,C,\Psi).}
\end{equation}
\end{theorem}

\subsection{Uniform kernel convergence and the RKHS}

Define the Hilbert space
\begin{equation}\label{cdv:eq:H-definition}
  \mathcal H_s
  =
  \left\{
  u\in L^2(dx):
  \norm{\Pi_0u}_2^2+
  \sum_{j\geq0}2^{2rj}\norm{Q_ju}_2^2<\infty
  \right\}
\end{equation}
with inner product
\begin{equation}\label{cdv:eq:form-definition}
  \mathfrak a(u,v)
  =
  (\Pi_0u,\Pi_0v)_{L^2(dx)}
  +
  \sum_{j\geq0}2^{2rj}(Q_ju,Q_jv)_{L^2(dx)}.
\end{equation}
By \eqref{cdv:eq:sobolev-characterization},
\begin{equation}\label{cdv:eq:H-Sobolev}
  \mathcal H_s=H^r(\Omega)
  \quad\text{with equivalent norms}.
\end{equation}

\begin{lemma}[Uniform absolute convergence]\label{cdv:lem:kernel-convergence}
The series in \eqref{cdv:eq:kernel-definition} converges uniformly and
absolutely on $\Omega\times\Omega$. More precisely, for every integer
$J\geq0$,
\begin{equation}\label{cdv:eq:kernel-tail}
\begin{split}
  &\sup_{x,y\in\Omega}
  \sum_{\substack{j\geq J\\k,e,\tau}}
  2^{-2rj}
  \abs{\psi_{j,k,e,\tau}(x)
  \psi_{j,k,e,\tau}(y)}
  \\
  &\hspace{5em}\lesssim_{\Psi}
  \sum_{j\geq J}2^{-2sj}
  \longrightarrow0
  \qquad(J\to\infty).
\end{split}
\end{equation}
\end{lemma}

\begin{proof}
Fix $j,x,y$. At the point $x$, at most $B_\Psi$ level-$j$ wavelets are
nonzero. For every such wavelet, \eqref{cdv:eq:L1-Linfty-scaling} gives
\[
  \abs{\psi_{j,k,e,\tau}(x)
  \psi_{j,k,e,\tau}(y)}
  \leq B_\Psi^2 2^{jd}.
\]
Thus
\[
  \sup_{x,y\in\Omega}
  \sum_{k,e,\tau}
  \abs{\psi_{j,k,e,\tau}(x)
  \psi_{j,k,e,\tau}(y)}
  \leq B_\Psi^3 2^{jd}.
\]
Multiplication by $2^{-2rj}$ yields $B_\Psi^3 2^{-2sj}$.
Summation over $j\geq J$ proves \eqref{cdv:eq:kernel-tail}. The coarse-level
sum is finite, so the full series converges uniformly and absolutely.
\end{proof}

The family
\begin{equation}\label{cdv:eq:H-ONB}
  \{\phi_\nu:\nu\in\mathcal I_0\}
  \cup
  \{2^{-rj}\psi_{j,k,e,\tau}:j\geq0\}
\end{equation}
is an orthonormal basis of $\mathcal H_s$. By the preceding lemma,
$K_s(x,x)$ is uniformly bounded. Consequently, for every
$u\in\mathcal H_s$,
\begin{equation}\label{cdv:eq:reproducing-estimate}
  \abs{u(x)}
  \leq K_s(x,x)^{1/2}\norm{u}_{\mathcal H_s}
  \lesssim_{s,\Psi}\norm{u}_{\mathcal H_s}.
\end{equation}
Thus $K_s$ is the reproducing kernel of $\mathcal H_s$ and
\begin{equation}\label{cdv:eq:H-Linfty-embedding}
  \mathcal H_s\hookrightarrow L^\infty(\Omega).
\end{equation}
The same embedding follows from \eqref{cdv:eq:H-Sobolev} and $r>d/2$.
Uniform absolute convergence in Lemma~\ref{cdv:lem:kernel-convergence} permits
pointwise evaluation and termwise integration against every
$L^1(\mu_p)$ function.

\subsection{Form-theoretic identification of the integral operator}

Fix $p\in\mathcal P(c,C)$ and abbreviate $\mu=\mu_p$.
Define the embedding
\begin{equation}\label{cdv:eq:Jp-definition}
  J_p:\mathcal H_s\longrightarrow L^2(\mu),
  \qquad J_pu=u.
\end{equation}
It is bounded because
\begin{equation}\label{cdv:eq:Jp-bounded}
  \norm{J_pu}_{L^2(\mu)}^2
  =\int_\Omega p\abs{u}^2
  \leq C\norm{u}_{L^2(dx)}^2
  \leq C\norm{u}_{\mathcal H_s}^2.
\end{equation}
It is injective because $p>0$ almost everywhere. It has dense range:
finite CDV expansions are dense in $L^2(dx)$, and
\begin{equation}\label{cdv:eq:L2-equivalence}
  c\norm{v}_{L^2(dx)}^2
  \leq\norm{v}_{L^2(\mu)}^2
  \leq C\norm{v}_{L^2(dx)}^2.
\end{equation}

The adjoint $J_p^*:L^2(\mu)\to\mathcal H_s$ is characterized by
\begin{equation}\label{cdv:eq:Jp-adjoint}
  \mathfrak a(J_p^*f,v)
  =(f,J_pv)_{L^2(\mu)},
  \qquad v\in\mathcal H_s.
\end{equation}
Its wavelet expansion is
\begin{equation}\label{cdv:eq:Jp-adjoint-expansion}
\begin{split}
  J_p^*f
  ={}&
  \sum_{\nu\in\mathcal I_0}
  (f,\phi_\nu)_{L^2(\mu)}\phi_\nu
  \\
  &+
  \sum_{j\geq0}2^{-2rj}
  \sum_{k,e,\tau}
  (f,\psi_{j,k,e,\tau})_{L^2(\mu)}
  \psi_{j,k,e,\tau},
\end{split}
\end{equation}
with convergence in $\mathcal H_s$, hence uniformly by
\eqref{cdv:eq:reproducing-estimate}. Since $\mu(\Omega)=1$, every
$f\in L^2(\mu)$ belongs to $L^1(\mu)$. Using
Lemma~\ref{cdv:lem:kernel-convergence} to interchange summation and integration,
we obtain
\begin{equation}\label{cdv:eq:JJstar-kernel}
  J_pJ_p^*f(x)
  =\int_\Omega K_s(x,y)f(y)\,d\mu(y)
  =T_pf(x).
\end{equation}
Therefore
\begin{equation}\label{cdv:eq:T-JJstar}
  T_p=J_pJ_p^*.
\end{equation}
It follows that $T_p$ is bounded, positive, and self-adjoint on
$L^2(\mu)$. It is injective: if $T_pf=0$, then
\[
  0=(T_pf,f)_{L^2(\mu)}
  =\norm{J_p^*f}_{\mathcal H_s}^2,
\]
so $f$ is orthogonal to the dense range of $J_p$, and hence $f=0$.

View $\mathfrak a$ as a densely defined form on $L^2(\mu)$ with domain
$\mathcal H_s$. It is closed because
\begin{equation}\label{cdv:eq:form-closed-equivalence}
  \mathfrak a(u,u)
  \leq
  \norm{u}_{L^2(\mu)}^2+\mathfrak a(u,u)
  \leq
  (1+C)\mathfrak a(u,u).
\end{equation}
Let $A_p$ be the positive self-adjoint operator associated with this
form. Thus
\begin{equation}\label{cdv:eq:Ap-domain}
\begin{split}
  D(A_p)=\bigl\{u\in\mathcal H_s:\ &\text{there is }g\in L^2(\mu)
  \text{ such that}\\
  &\mathfrak a(u,v)=(g,v)_{L^2(\mu)}
  \text{ for all }v\in\mathcal H_s\bigr\},
\end{split}
\end{equation}
and $A_pu=g$. Since
\begin{equation}\label{cdv:eq:Ap-lower-bound}
  \mathfrak a(u,u)
  \geq\norm{u}_{L^2(dx)}^2
  \geq C^{-1}\norm{u}_{L^2(\mu)}^2,
\end{equation}
we have $A_p\geq C^{-1}I$. Since $A_p$ is self-adjoint, this lower bound
places $0$ in the resolvent set of $A_p$. Hence $A_p$ is bijective and
\[
  \norm{A_p^{-1}}_{L^2(\mu)\to L^2(\mu)}\leq C.
\]

Let $f\in L^2(\mu)$ and put $h=J_p^*f\in\mathcal H_s$. Equation
\eqref{cdv:eq:Jp-adjoint} gives
\[
  \mathfrak a(h,v)=(f,v)_{L^2(\mu)},
  \qquad v\in\mathcal H_s.
\]
Thus the $L^2(\mu)$ element $J_ph$ lies in $D(A_p)$ and
$A_p(J_ph)=f$. By \eqref{cdv:eq:JJstar-kernel}, $J_ph=T_pf$. Therefore
\begin{equation}\label{cdv:eq:T-Ainverse}
  \boxed{T_p=A_p^{-1}}
\end{equation}
as bounded operators on $L^2(\mu)$.

The spectral calculus now gives
\begin{equation}\label{cdv:eq:resolvent-reduction}
  \boxed{
  T_p(T_p+\lambda I)^{-1}
  =(I+\lambda A_p)^{-1}.}
\end{equation}
For $f\in L^\infty(\mu)\subset L^2(\mu)$, the function
$u=(I+\lambda A_p)^{-1}f$ is the unique $u\in\mathcal H_s$ satisfying
\begin{equation}\label{cdv:eq:variational-equation}
  (u,v)_{L^2(\mu)}
  +\lambda\mathfrak a(u,v)
  =(f,v)_{L^2(\mu)},
  \qquad v\in\mathcal H_s.
\end{equation}

\subsection{The regime \texorpdfstring{$\lambda\geq1$}{lambda >= 1}}

Take $v=u$ in \eqref{cdv:eq:variational-equation}. Taking real parts and using
$\mu(\Omega)=1$ gives
\begin{equation}\label{cdv:eq:large-lambda-energy}
\begin{split}
  \norm{u}_{L^2(\mu)}^2
  +\lambda\norm{u}_{\mathcal H_s}^2
  &\leq
  \norm{f}_\infty\norm{u}_{L^1(\mu)}
  \\
  &\leq
  \norm{f}_\infty\norm{u}_{L^2(\mu)}.
\end{split}
\end{equation}
It follows that
\begin{equation}\label{cdv:eq:large-lambda-bounds}
  \norm{u}_{L^2(\mu)}\leq\norm{f}_\infty,
  \qquad
  \lambda\norm{u}_{\mathcal H_s}^2\leq\norm{f}_\infty^2.
\end{equation}
For $\lambda\geq1$, \eqref{cdv:eq:reproducing-estimate} yields
\begin{equation}\label{cdv:eq:large-lambda-Linfty}
  \norm{u}_\infty
  \lesssim_{d,s,\Psi}\norm{f}_\infty.
\end{equation}
Thus
\begin{equation}\label{cdv:eq:large-lambda-operator}
  \sup_{p\in\mathcal P(c,C)}
  \sup_{\lambda\geq1}
  \norm{(I+\lambda A_p)^{-1}}_{L^\infty\to L^\infty}
  \lesssim_{d,s,\Psi}1.
\end{equation}

\subsection{The adapted scale and uniform energy stability}

Assume for the rest of the proof that $0<\lambda<1$. Define
\begin{equation}\label{cdv:eq:J-definition}
  J=J_\lambda
  =\min\{j\geq0:\lambda2^{2rj}\geq1\}.
\end{equation}
Then
\begin{equation}\label{cdv:eq:qJ-range}
  1\leq q_J:=\lambda2^{2rJ}<2^{2r}.
\end{equation}
Use
\begin{equation}\label{cdv:eq:J-decomposition}
  L^2(\Omega)=V_J\oplus\bigoplus_{j\geq J}W_j.
\end{equation}
For $j\geq J$, set
\begin{equation}\label{cdv:eq:sigma-definition}
  \sigma_j=(1+\lambda2^{2rj})^{-1/2}.
\end{equation}
If $j=J+m$, then
\[
  \sigma_{J+m}=(1+q_J2^{2rm})^{-1/2},
\]
and \eqref{cdv:eq:qJ-range} implies
\begin{equation}\label{cdv:eq:sigma-asymptotics}
  (1+2^{2r})^{-1/2}2^{-rm}
  \leq\sigma_{J+m}\leq2^{-rm}.
\end{equation}
In particular,
\begin{equation}\label{cdv:eq:sigma-comparison}
  \sigma_{J+m}\asymp_r2^{-rm}.
\end{equation}

Define
\begin{equation}\label{cdv:eq:b-form}
  b_{\lambda,p}(u,v)
  =(u,v)_{L^2(\mu_p)}+\lambda\mathfrak a(u,v).
\end{equation}
First use the adapted system
\begin{equation}\label{cdv:eq:J-adapted-system}
  \{\phi_{J,k,\tau}\}_{(k,\tau)\in X_J}
  \cup
  \{\sigma_j\psi_{j,k,e,\tau}:j\geq J\}.
\end{equation}
For finitely supported coefficients, write
\[
  v_J=\sum_{(k,\tau)\in X_J}a_{J,k,\tau}\phi_{J,k,\tau},
  \qquad
  w_j=\sum_{k,e,\tau}a_{j,k,e,\tau}\psi_{j,k,e,\tau},
\]
and set
\[
  u_a=v_J+\sum_{j\geq J}\sigma_jw_j.
\]
Orthogonality in $L^2(dx)$ and diagonality of $\mathfrak a$ across the
spaces in \eqref{cdv:eq:J-decomposition} give
\begin{equation}\label{cdv:eq:adapted-lower}
\begin{split}
  b_{\lambda,p}(u_a,u_a)
  \geq{}&c\norm{v_J}_2^2+\lambda\mathfrak a(v_J,v_J)
  \\
  &+\sum_{j\geq J}
  \frac{c+\lambda2^{2rj}}{1+\lambda2^{2rj}}
  \norm{w_j}_2^2.
\end{split}
\end{equation}
The last coefficient is at least $\min\{c,1\}$. For $v\in V_J$,
\begin{equation}\label{cdv:eq:Bernstein-J}
\begin{split}
  \mathfrak a(v,v)
  &=\norm{\Pi_0v}_2^2+
  \sum_{q<J}2^{2rq}\norm{Q_qv}_2^2
  \\
  &\leq2^{2rJ}\norm{v}_2^2.
\end{split}
\end{equation}
Using $p\leq C$, \eqref{cdv:eq:qJ-range}, and
$\lambda2^{2rj}\sigma_j^2\leq1$, we obtain constants
$c_0=\min\{c,1\}$ and $C_0=C_0(C,r)$ such that
\begin{equation}\label{cdv:eq:J-energy-stability}
  c_0\norm{a}_{\ell^2}^2
  \leq b_{\lambda,p}(u_a,u_a)
  \leq C_0\norm{a}_{\ell^2}^2.
\end{equation}
The associated infinite matrix is therefore uniformly bounded and
coercive on $\ell^2$. This fact does not imply a Schur bound for its
inverse, which is why the fixed-depth split below is required.

\subsection{Fixed-depth splitting, mixed synthesis, and cube geometry}

Let $m_\ast\geq1$ be an integer that will be fixed once in
Section~\ref{cdv:sec:tail-inversion}, and set
\begin{equation}\label{cdv:eq:L-definition}
  L=J_\lambda+m_\ast.
\end{equation}
The multiresolution identity gives
\begin{equation}\label{cdv:eq:fixed-depth-identity}
  V_L
  =V_J\oplus W_J\oplus\cdots\oplus W_{L-1}.
\end{equation}
We work from this point onward only in the mixed basis
\begin{equation}\label{cdv:eq:mixed-basis}
  \mathcal B_{\lambda,L}
  =
  \{\phi_{L,k,\tau}:(k,\tau)\in X_L\}
  \cup
  \{\sigma_j\psi_{j,k,e,\tau}:j\geq L\}.
\end{equation}
The finite refinement filters in Assumption~\ref{cdv:ass:cdv}, including the finitely
many face, edge, and corner refinement matrices, ensure that this is the
actual level-$L$ CDV scaling space followed by the CDV tail.

Let
\begin{equation}\label{cdv:eq:tail-index-set}
  Y_L
  =\{(j,k,e,\tau):j\geq L,\ (k,e,\tau)\in Y_j\},
  \qquad
  \Gamma_L=X_L\mathbin{\dot\cup}Y_L.
\end{equation}
For $a=(x,y)\in\ell^2(X_L)\oplus\ell^2(Y_L)$, define
\begin{equation}\label{cdv:eq:U-definition}
\begin{split}
  U_{\lambda,L}a
  ={}&
  \sum_{(k,\tau)\in X_L}x_{k,\tau}\phi_{L,k,\tau}
  \\
  &+
  \sum_{j\geq L}\sigma_j
  \sum_{k,e,\tau}y_{j,k,e,\tau}\psi_{j,k,e,\tau}.
\end{split}
\end{equation}

\begin{lemma}[Mixed synthesis isomorphism]\label{cdv:lem:U-isomorphism}
For each $0<\lambda<1$, the map
$U_{\lambda,L}:\ell^2(\Gamma_L)\to\mathcal H_s$ is bounded and boundedly
invertible onto $\mathcal H_s$. More precisely,
\begin{equation}\label{cdv:eq:U-norm-equivalence}
  \frac12\norm{a}_{\ell^2}^2
  \leq
  \norm{U_{\lambda,L}a}_{\mathcal H_s}^2
  \leq
  \frac{C(r,m_\ast)}{\lambda}\norm{a}_{\ell^2}^2.
\end{equation}
\end{lemma}

\begin{proof}
Let $v\in V_L$ be the low-frequency part of $U_{\lambda,L}a$ and let
$w_j\in W_j$ denote the unnormalized level-$j$ tail part. Then
\begin{equation}\label{cdv:eq:U-H-norm}
  \norm{U_{\lambda,L}a}_{\mathcal H_s}^2
  =\mathfrak a(v,v)
  +\sum_{j\geq L}2^{2rj}\sigma_j^2\norm{w_j}_2^2.
\end{equation}
On $V_L$,
\[
  \norm{v}_2^2\leq\mathfrak a(v,v)\leq2^{2rL}\norm{v}_2^2.
\]
For $j\geq L\geq J$,
\begin{equation}\label{cdv:eq:tail-H-normalization}
  \frac1{2\lambda}
  \leq
  2^{2rj}\sigma_j^2
  =\frac1\lambda
  \frac{\lambda2^{2rj}}{1+\lambda2^{2rj}}
  \leq\frac1\lambda.
\end{equation}
Since $\lambda<1$, the lower estimate in
\eqref{cdv:eq:U-norm-equivalence} follows. Furthermore,
\[
  2^{2rL}
  =2^{2rm_\ast}2^{2rJ}
  <\frac{2^{2r(m_\ast+1)}}{\lambda},
\]
which proves the upper estimate.

To prove surjectivity, let $P_L$ denote the orthogonal projection onto
$V_L$, let $u\in\mathcal H_s$, and decompose
$u=P_Lu+\sum_{j\geq L}Q_ju$. Expand $P_Lu$ in the orthonormal scaling
basis and define the tail coefficient vector by dividing the level-$j$
wavelet coefficients by $\sigma_j$. Then
\[
  \sum_{j\geq L}\sigma_j^{-2}\norm{Q_ju}_2^2
  =\sum_{j\geq L}(1+\lambda2^{2rj})\norm{Q_ju}_2^2
  <\infty.
\]
Thus the resulting coefficient vector belongs to $\ell^2(\Gamma_L)$ and
is mapped to $u$.
\end{proof}

Let $\widehat B_{\lambda,p}$ be the matrix of $b_{\lambda,p}$ in the
mixed basis. Equivalently, as an operator on $\ell^2(\Gamma_L)$,
\begin{equation}\label{cdv:eq:B-operator-form}
  \widehat B_{\lambda,p}
  =U_{\lambda,L}^*(J_p^*J_p+\lambda I_{\mathcal H_s})U_{\lambda,L}.
\end{equation}
A direct calculation, using $p\geq c$, gives
\begin{equation}\label{cdv:eq:mixed-coercivity}
  \widehat B_{\lambda,p}\geq c_\ast I,
  \qquad c_\ast=\min\{c,1\}.
\end{equation}
Indeed, the low coefficient is bounded below by $c$, and every tail
coefficient equals
\[
  \frac{c+\lambda2^{2rj}}{1+\lambda2^{2rj}}
  \geq\min\{c,1\}.
\]
The Bernstein estimate on $V_L$ and
$\lambda2^{2rL}<2^{2r(m_\ast+1)}$ also give
\begin{equation}\label{cdv:eq:mixed-upper-bound}
  \norm{\widehat B_{\lambda,p}}_{\ell^2\to\ell^2}
  \leq C_{\mathrm{en}}(C,r,m_\ast).
\end{equation}

We now define the geometry used for localization. Let $\mathcal T$ be a
fixed finite set containing every scaling, orientation, interior, face,
edge, and corner label occurring in the mixed basis. For a low index
$\gamma=(k,\tau)\in X_L$, set
\[
  a_L(\gamma)=2^{-L}k,
  \qquad t(\gamma)=\tau.
\]
For a tail index $\eta=(j,k,e,\tau)\in Y_L$, set
\[
  a_L(\eta)=2^{-j}k,
  \qquad t(\eta)=(e,\tau).
\]
Define the pseudometric
\begin{equation}\label{cdv:eq:rhoL}
  \rho_L(\gamma,\eta)
  =2^L\abs{a_L(\gamma)-a_L(\eta)}_\infty
  +\ind_{\{t(\gamma)\neq t(\eta)\}}.
\end{equation}
Different-scale indices may have zero spatial separation, so $\rho_L$ is
used only as a pseudometric on $\Gamma_L$.

Fix
\begin{equation}\label{cdv:eq:alpha-choice}
  \alpha=r,
  \qquad 0<\alpha<2r,
\end{equation}
and define
\begin{equation}\label{cdv:eq:weight-definition}
  w_{\alpha,L}(\gamma,\eta)
  =(1+\rho_L(\gamma,\eta))^\alpha.
\end{equation}
For a rectangular matrix $M=(M_{\gamma\eta})_{\gamma\in I,\eta\in K}$,
where $I,K\subset\Gamma_L$, set
\begin{equation}\label{cdv:eq:weighted-Schur-norm}
\begin{split}
  \norm{M}_{\mathscr S_{\alpha,L}(I,K)}
  =\max\Bigg\{&
  \sup_{\gamma\in I}
  \sum_{\eta\in K}
  \abs{M_{\gamma\eta}}w_{\alpha,L}(\gamma,\eta),
  \\
  &
  \sup_{\eta\in K}
  \sum_{\gamma\in I}
  \abs{M_{\gamma\eta}}w_{\alpha,L}(\gamma,\eta)
  \Bigg\}.
\end{split}
\end{equation}
The triangle inequality for $\rho_L$ implies
\[
  w_{\alpha,L}(\gamma,\zeta)
  \leq
  w_{\alpha,L}(\gamma,\eta)
  w_{\alpha,L}(\eta,\zeta).
\]
Consequently,
\begin{equation}\label{cdv:eq:Schur-product}
  \norm{MN}_{\mathscr S_{\alpha,L}(I,H)}
  \leq
  \norm{M}_{\mathscr S_{\alpha,L}(I,K)}
  \norm{N}_{\mathscr S_{\alpha,L}(K,H)}.
\end{equation}
The square classes are Banach algebras and the rectangular classes are
compatible Banach modules.

On the low index set, write
\begin{equation}\label{cdv:eq:low-rho}
  \rho((k,\tau),(l,\tau'))
  =\abs{k-l}_\infty+\ind_{\{\tau\neq\tau'\}}.
\end{equation}
Since $X_L\subset\mathbb Z^d\times\mathcal T_\phi$ and
$\mathcal T_\phi$ is finite,
\begin{equation}\label{cdv:eq:polynomial-growth}
  \sup_{\gamma\in X_L}
  \#\{\gamma'\in X_L:\rho(\gamma,\gamma')\leq R\}
  \lesssim_{d,\Psi}(1+R)^d,
\end{equation}
with a constant independent of $L$. This follows by comparing with the
complete finite-type lattice; removal of indices near the boundary cannot
increase the count. On $X_L$, the weights based on $\rho$ and on $\rho_L$
are equivalent with constants depending only on $\alpha$ and $\Psi$.

\subsection{Mass-matrix localization}

\begin{lemma}[Mass localization]\label{cdv:lem:mass-localization}
Let $j\geq\ell$, and let $h_{j,k,\vartheta}$ and
$h_{\ell,m,\vartheta'}$ be any CDV scaling functions or wavelets at the
indicated levels. Then
\begin{equation}\label{cdv:eq:mass-localization}
\begin{split}
  &\abs{
  \int_\Omega
  p(x)h_{j,k,\vartheta}(x)
  \overline{h_{\ell,m,\vartheta'}(x)}\,dx}
  \\
  &\quad\lesssim_{C,\Psi}
  2^{-(j-\ell)d/2}
  \ind_{\{\supp h_{j,k,\vartheta}
  \cap\supp h_{\ell,m,\vartheta'}\neq\varnothing\}}.
\end{split}
\end{equation}
No cancellation against $p$ is used.
\end{lemma}

\begin{proof}
The integral vanishes when the supports are disjoint. Otherwise, using
only $p\leq C$ and \eqref{cdv:eq:L1-Linfty-scaling},
\[
\begin{split}
  \abs{\int p h_{j,k,\vartheta}
  \overline{h_{\ell,m,\vartheta'}}}
  &\leq
  C\norm{h_{j,k,\vartheta}}_1
  \norm{h_{\ell,m,\vartheta'}}_\infty
  \\
  &\leq
  CB_\Psi^2 2^{-jd/2}2^{\ell d/2}.
\end{split}
\]
The same constants apply to the finitely many boundary types.
\end{proof}

If two basis supports at levels at least $L$ intersect, then
\eqref{cdv:eq:support-localization} gives
\begin{equation}\label{cdv:eq:overlap-weight-bounded}
  \rho_L(\gamma,\eta)\leq 2A_\Psi+1.
\end{equation}
Hence the weight $w_{\alpha,L}$ is uniformly bounded on every nonzero
mass-matrix entry.

\subsection{Infinite-dimensional tail inversion}\label{cdv:sec:tail-inversion}

Decompose the mixed matrix as
\begin{equation}\label{cdv:eq:block-matrix}
  \widehat B_{\lambda,p}
  =
  \begin{pmatrix}
    G&C\\
    C^*&D
  \end{pmatrix}
\end{equation}
on $\ell^2(X_L)\oplus\ell^2(Y_L)$. For a tail index
$\gamma=(j,k,e,\tau)$, define
\begin{equation}\label{cdv:eq:E-diagonal}
  E_{\gamma\gamma}
  =e_j
  =\lambda2^{2rj}\sigma_j^2
  =\frac{\lambda2^{2rj}}{1+\lambda2^{2rj}}.
\end{equation}
The energy part is diagonal on the tail, so
\begin{equation}\label{cdv:eq:D-E-R}
  D=E+R,
\end{equation}
where $R$ is the normalized mass matrix. Since $j\geq L\geq J$,
\begin{equation}\label{cdv:eq:E-lower}
  e_j\geq\frac12,
  \qquad
  \norm{E^{-1}}_{\mathscr S_{\alpha,L}(Y_L,Y_L)}\leq2.
\end{equation}

We estimate $R$ for an arbitrary integer $m_\ast\geq1$. Fix a tail
index at level $\ell\geq L$ and compare it first with the finer level
$\ell+q$, where $q\geq0$. By Lemma~\ref{cdv:lem:mass-localization}, every
overlapping normalized entry is bounded by
\begin{equation}\label{cdv:eq:R-entry-fine}
  \lesssim_{C,\Psi}
  \sigma_\ell\sigma_{\ell+q}2^{-qd/2}.
\end{equation}
A fixed level-$\ell$ support meets at most
$B_\Psi2^{dq}$ level-$(\ell+q)$ supports. By
\eqref{cdv:eq:overlap-weight-bounded}, the weighted sum over that level is
therefore
\begin{equation}\label{cdv:eq:R-fine-level-sum}
  \lesssim_{C,\Psi}
  \sigma_\ell\sigma_{\ell+q}2^{qd/2}.
\end{equation}
Because $\lambda2^{2r\ell}\geq1$,
\begin{equation}\label{cdv:eq:sigma-finer-ratio}
  \frac{\sigma_{\ell+q}}{\sigma_\ell}
  =
  \left(
  \frac{1+\lambda2^{2r\ell}}
  {1+\lambda2^{2r(\ell+q)}}
  \right)^{1/2}
  \leq\sqrt2\,2^{-rq}.
\end{equation}
Thus \eqref{cdv:eq:R-fine-level-sum} is bounded by
\[
  \lesssim_{C,r,\Psi}
  \sigma_\ell^2 2^{-(r-d/2)q}
  =\sigma_\ell^2 2^{-sq}.
\]
Writing $\ell=L+a=J+m_\ast+a$, we have
\[
  \sigma_\ell^2\leq2^{-2r(m_\ast+a)}.
\]
Summation over all finer levels gives
\begin{equation}\label{cdv:eq:R-finer-total}
  \lesssim_{C,s,\Psi}2^{-2rm_\ast}.
\end{equation}

For coarser tail levels, write the level as $\ell-q$ with
$0\leq q\leq a$. A fixed fine support meets only $O_\Psi(1)$ supports
at each such level, and
\begin{equation}\label{cdv:eq:sigma-coarser-ratio}
  \frac{\sigma_{\ell-q}}{\sigma_\ell}
  \leq\sqrt2\,2^{rq}.
\end{equation}
The weighted contribution at level $\ell-q$ is therefore
\[
  \lesssim_{C,r,\Psi}\sigma_\ell^2 2^{sq}.
\]
Summing from $q=0$ to $a$ gives
\begin{equation}\label{cdv:eq:R-coarser-total}
\begin{split}
  \sum_{q=0}^a\sigma_\ell^2 2^{sq}
  &\lesssim_s
  2^{-2r(m_\ast+a)}2^{sa}
  \\
  &=2^{-2rm_\ast}2^{-(2r-s)a}
  =2^{-2rm_\ast}2^{-(s+d)a}.
\end{split}
\end{equation}
The matrix $R$ and the weight are symmetric, so the same estimates give
both row and column bounds. We have proved
\begin{equation}\label{cdv:eq:R-Schur-small}
  \norm{R}_{\mathscr S_{\alpha,L}(Y_L,Y_L)}
  \leq K_R(d,s,C,\Psi)2^{-2rm_\ast}.
\end{equation}

We now fix $m_\ast$ once and for all so that
\begin{equation}\label{cdv:eq:mstar-choice}
  2K_R(d,s,C,\Psi)2^{-2rm_\ast}\leq\frac12.
\end{equation}
This choice is independent of $p,\lambda,J$, and $L$. Equations
\eqref{cdv:eq:E-lower} and \eqref{cdv:eq:R-Schur-small} imply
\[
  \norm{E^{-1}R}_{\mathscr S_{\alpha,L}}\leq\frac12.
\]
The Neumann series
\begin{equation}\label{cdv:eq:D-Neumann}
\begin{split}
  D^{-1}
  &=(I+E^{-1}R)^{-1}E^{-1}
  \\
  &=\sum_{n=0}^\infty(-E^{-1}R)^nE^{-1}
\end{split}
\end{equation}
converges in the weighted-Schur Banach algebra. Since the ordinary Schur
test embeds this algebra continuously into $\mathcal B(\ell^2)$, the sum
is the genuine bounded-operator inverse of $D$. Moreover,
\begin{equation}\label{cdv:eq:D-inverse-Schur}
  \boxed{
  \norm{D^{-1}}_{\mathscr S_{\alpha,L}(Y_L,Y_L)}\leq4.}
\end{equation}
No finite-section limit is involved.

\subsection{The low-tail coupling}

There is no energy cross term between $V_L$ and $W_j$ for $j\geq L$.
Thus $C$ consists only of mass entries. Let $j=L+q$. By
Lemma~\ref{cdv:lem:mass-localization},
\begin{equation}\label{cdv:eq:C-entry}
\begin{split}
  \abs{C_{(k,\tau),(j,m,e,\tau')}}
  \lesssim_{C,\Psi}{}&
  \sigma_j2^{-qd/2}
  \\
  &\times
  \ind_{\{\supp\phi_{L,k,\tau}
  \cap\supp\psi_{j,m,e,\tau'}\neq\varnothing\}}.
\end{split}
\end{equation}
A fixed level-$L$ scaling support meets at most $B_\Psi2^{dq}$
level-$(L+q)$ wavelet supports. Hence
\begin{equation}\label{cdv:eq:C-row-level}
\begin{split}
  &\sum_{\substack{m,e,\tau'\\\operatorname{level}=L+q}}
  \abs{C_{(k,\tau),(j,m,e,\tau')}}w_{\alpha,L}
  \\
  &\hspace{7em}\lesssim_{C,\Psi}
  \sigma_{L+q}2^{qd/2}
  \lesssim_{C,r,\Psi}
  2^{-rm_\ast}2^{-sq}.
\end{split}
\end{equation}
Summation over $q\geq0$ controls every weighted row. Conversely, a fixed
level-$(L+q)$ wavelet meets only $O_\Psi(1)$ level-$L$ scaling supports,
so every weighted column is bounded by
\[
  \lesssim_{C,r,\Psi}
  \sigma_{L+q}2^{-qd/2}
  \lesssim_{C,r,\Psi}
  2^{-rm_\ast}2^{-(r+d/2)q}.
\]
Therefore
\begin{equation}\label{cdv:eq:C-Schur}
  \boxed{
  \norm{C}_{\mathscr S_{\alpha,L}(X_L,Y_L)}
  \lesssim_{C,s,\Psi}2^{-rm_\ast}.}
\end{equation}
The finite collection of face, edge, and corner types changes only the
constant: every such function has support diameter $O_\Psi(2^{-j})$ and
satisfies the same scale-normalized norm and overlap estimates.

\subsection{Localization of the low-frequency block}

The low block is
\begin{equation}\label{cdv:eq:G-entry}
\begin{split}
  G_{(k,\tau),(l,\tau')}
  ={}&
  \int_\Omega
  p(x)\phi_{L,l,\tau'}(x)
  \overline{\phi_{L,k,\tau}(x)}\,dx
  \\
  &+
  \lambda\mathfrak a(
  \phi_{L,l,\tau'},\phi_{L,k,\tau}).
\end{split}
\end{equation}
The mass part is uniformly banded. Every row and column has
$O_\Psi(1)$ nonzero entries, each bounded by $C$, and the weight is
bounded on these entries. Its weighted-Schur norm is therefore bounded
by a constant depending only on $C,r$, and $\Psi$.

We next estimate the energy part. For $0\leq q<L$, put $h=L-q$. The
wavelet expansion of $Q_q$ gives
\begin{equation}\label{cdv:eq:Qq-expansion}
\begin{split}
  &(Q_q\phi_{L,l,\tau'},Q_q\phi_{L,k,\tau})_{L^2(dx)}
  \\
  &\quad=
  \sum_{m,e,\vartheta}
  (\phi_{L,l,\tau'},\psi_{q,m,e,\vartheta})
  (\psi_{q,m,e,\vartheta},\phi_{L,k,\tau}).
\end{split}
\end{equation}
The $L^1$--$L^\infty$ bounds imply
\begin{equation}\label{cdv:eq:low-wavelet-coefficient}
  \abs{(\phi_{L,k,\tau},\psi_{q,m,e,\vartheta})}
  \lesssim_\Psi
  2^{-hd/2}
  \ind_{\{\supp\phi_{L,k,\tau}
  \cap\supp\psi_{q,m,e,\vartheta}\neq\varnothing\}}.
\end{equation}
For fixed $q,k,l,\tau,\tau'$, only $O_\Psi(1)$ level-$q$ wavelets can
meet both level-$L$ scaling supports. Such common intersection requires
\begin{equation}\label{cdv:eq:common-overlap-distance}
  \abs{k-l}_\infty\lesssim_\Psi2^h.
\end{equation}
Consequently,
\begin{equation}\label{cdv:eq:Qq-low-bound}
  \abs{(Q_q\phi_{L,l,\tau'},Q_q\phi_{L,k,\tau})}
  \lesssim_\Psi
  2^{-hd}
  \ind_{\{\abs{k-l}_\infty\lesssim_\Psi2^h\}}.
\end{equation}
Multiplying by $\lambda2^{2rq}$ and using
$\lambda2^{2rL}<2^{2r(m_\ast+1)}$ gives
\begin{equation}\label{cdv:eq:weighted-lower-level}
\begin{split}
  &\lambda2^{2rq}
  \abs{(Q_q\phi_{L,l,\tau'},Q_q\phi_{L,k,\tau})}
  \\
  &\quad\lesssim_{r,m_\ast,\Psi}
  2^{-(2r+d)h}
  \ind_{\{\abs{k-l}_\infty\lesssim_\Psi2^h\}}.
\end{split}
\end{equation}

The finite $\Pi_0$ term satisfies the analogous estimate with $h=L$.
Indeed,
\[
  \abs{(\phi_{L,k,\tau},\phi_\nu)}
  \lesssim_\Psi2^{-Ld/2},
\]
and a coarse function can meet both fine supports only if
$\abs{k-l}_\infty\lesssim_\Psi2^L$. Since
\[
  \lambda=q_J2^{2rm_\ast}2^{-2rL},
\]
the resulting contribution is
\begin{equation}\label{cdv:eq:P0-low-bound}
  \lesssim_{r,m_\ast,\Psi}
  2^{-(2r+d)L}
  \ind_{\{\abs{k-l}_\infty\lesssim_\Psi2^L\}}.
\end{equation}

\begin{lemma}[Dyadic summation]\label{cdv:lem:dyadic-sum}
For every $\eta>0$ and $A>0$, there is $C_{\eta,A}<\infty$ such that
\begin{equation}\label{cdv:eq:dyadic-sum}
  \sum_{h\geq1}2^{-\eta h}
  \ind_{\{D\leq A2^h\}}
  \leq C_{\eta,A}(1+D)^{-\eta},
  \qquad D\geq0.
\end{equation}
\end{lemma}

\begin{proof}
If $D\leq2A$, the claim follows from the full geometric series. If
$D>2A$, let $h_0$ be the smallest integer with $D\leq A2^{h_0}$. The
left-hand side is the geometric series beginning at $h_0$, and
$2^{-h_0}\leq A/D$.
\end{proof}

Combining \eqref{cdv:eq:weighted-lower-level},
\eqref{cdv:eq:P0-low-bound}, and Lemma~\ref{cdv:lem:dyadic-sum} yields
\begin{equation}\label{cdv:eq:low-energy-decay}
  \boxed{
  \lambda
  \abs{\mathfrak a(
  \phi_{L,l,\tau'},\phi_{L,k,\tau})}
  \lesssim_{r,m_\ast,\Psi}
  (1+\abs{k-l}_\infty)^{-(2r+d)}.}
\end{equation}
This calculation uses the actual lower-level CDV functions on the cube.
The finitely many boundary types enter only through the uniform support,
norm, and overlap constants in Assumption~\ref{cdv:ass:cdv}.

Since $\alpha=r<2r$ and the number of types is finite,
\begin{equation}\label{cdv:eq:low-decay-Schur-sum}
\begin{split}
  &\sup_{(k,\tau)\in X_L}
  \sum_{(l,\tau')\in X_L}
  (1+\abs{k-l}_\infty)^{-(2r+d)}
  (1+\rho((k,\tau),(l,\tau')))^\alpha
  \\
  &\quad\lesssim_{d,r,\Psi}
  \sum_{n\geq0}(1+n)^{d-1-(2r+d)+r}
  =\sum_{n\geq0}(1+n)^{-r-1}<\infty.
\end{split}
\end{equation}
Therefore
\begin{equation}\label{cdv:eq:G-Schur}
  \boxed{
  \norm{G}_{\mathscr S_\alpha(X_L)}
  \leq M_G(d,s,C,\Psi),}
\end{equation}
where the already fixed $m_\ast$ is absorbed into the dependence on
$d,s,C,\Psi$.

\subsection{Uniform finite-lattice Wiener inversion}

We use one external matrix inversion theorem, in a norm-controlled form.

\begin{lemma}[Uniform finite-lattice weighted-Schur inversion]
\label{cdv:lem:finite-Wiener}
Let $\mathcal T$ be fixed and finite, let
$X\subset\mathbb Z^d\times\mathcal T$ be finite, and set
\begin{equation}\label{cdv:eq:finite-lattice-rho}
  \rho((k,\tau),(l,\tau'))
  =\abs{k-l}_\infty+\ind_{\{\tau\neq\tau'\}}.
\end{equation}
For $\alpha>0$, let
\[
  w_\alpha(\gamma,\gamma')
  =(1+\rho(\gamma,\gamma'))^\alpha
\]
and define
\begin{equation}\label{cdv:eq:finite-Schur-norm}
\begin{split}
  \norm{A}_{\mathscr S_\alpha(X)}
  =\max\Bigg\{&
  \sup_{\gamma\in X}
  \sum_{\gamma'\in X}
  \abs{A_{\gamma\gamma'}}w_\alpha(\gamma,\gamma'),
  \\
  &
  \sup_{\gamma'\in X}
  \sum_{\gamma\in X}
  \abs{A_{\gamma\gamma'}}w_\alpha(\gamma,\gamma')
  \Bigg\}.
\end{split}
\end{equation}
For every $M,\kappa>0$, there is
\[
  F=F(d,\alpha,\abs{\mathcal T},M,\kappa)<\infty
\]
such that
\begin{equation}\label{cdv:eq:finite-Wiener-assumptions}
  A=A^*,\qquad A\geq\kappa I,
  \qquad
  \norm{A}_{\mathscr S_\alpha(X)}\leq M
\end{equation}
imply
\begin{equation}\label{cdv:eq:finite-Wiener-conclusion}
  \norm{A^{-1}}_{\mathscr S_\alpha(X)}
  \leq F(d,\alpha,\abs{\mathcal T},M,\kappa).
\end{equation}
The constant is independent of the finite subset $X$. Equivalently,
there is a norm-control function depending only on the fixed lattice and
weight data such that
\begin{equation}\label{cdv:eq:finite-Wiener-norm-control}
  \norm{A^{-1}}_{\mathscr S_\alpha}
  \leq
  F\!\left(
  d,\alpha,\abs{\mathcal T},
  \norm{A}_{\mathscr S_\alpha},
  \norm{A^{-1}}_{\ell^2\to\ell^2}
  \right).
\end{equation}
\end{lemma}

\begin{proof}
Let $\Lambda=\mathbb Z^d\times\mathcal T$ and decompose
\[
  \ell^2(\Lambda)=\ell^2(X)\oplus\ell^2(X^c).
\]
Define
\begin{equation}\label{cdv:eq:A-tilde}
  \widetilde A=A\oplus I.
\end{equation}
Then
\[
  \widetilde A^{-1}=A^{-1}\oplus I,
  \qquad
  \norm{\widetilde A^{-1}}_{\ell^2\to\ell^2}
  \leq\max\{\kappa^{-1},1\}.
\]
The complete finite-type lattice has polynomial growth of degree $d$:
\[
  \#\{\gamma'\in\Lambda:\rho(\gamma,\gamma')\leq R\}
  \leq\abs{\mathcal T}(2\lfloor R\rfloor+1)^d.
\]
Moreover, the weighted-Schur norm of $\widetilde A$ is bounded by a fixed
multiple of $\max\{M,1\}$.

Sun's norm-controlled Wiener theorem for Schur classes
\cite[Theorem~4.1]{Sun2007} applies to discrete pseudometric spaces with
counting measure of polynomial growth. For the $p=1$ Schur class,
\cite[Example~A.2]{Sun2007} shows that every polynomial weight
$(1+\rho)^\alpha$ with $\alpha>0$ is admissible: the general threshold
there is $\alpha>d_\Lambda(1-1/p)$, which equals $0$ when $p=1$. The
estimate in Sun's theorem bounds the weighted-Schur norm of the inverse
in terms of the fixed geometric and weight constants, the weighted-Schur
norm of the matrix, and its $\ell^2$ inverse norm. Sun uses the sum of the
row and column seminorms; it is equivalent within a factor of two to the
maximum in \eqref{cdv:eq:finite-Schur-norm}.

Applying that theorem to $\widetilde A$ gives a bound depending only on
$d,\alpha,\abs{\mathcal T},M$, and $\kappa$. Restricting the row and
column sums of $\widetilde A^{-1}$ to $X\times X$ proves
\eqref{cdv:eq:finite-Wiener-conclusion}.
\end{proof}

\subsection{Schur complement and the full inverse}

Define the Schur complement on $\ell^2(X_L)$ by
\begin{equation}\label{cdv:eq:S-definition}
  S=G-CD^{-1}C^*.
\end{equation}
For $x\in\ell^2(X_L)$ and $y\in\ell^2(Y_L)$, completing the square gives
the bounded-operator identity
\begin{equation}\label{cdv:eq:complete-square}
\begin{split}
  \ip{\widehat B_{\lambda,p}\binom{x}{y}}{\binom{x}{y}}
  ={}&\ip{Sx}{x}
  \\
  &+
  \ip{D(y+D^{-1}C^*x)}{y+D^{-1}C^*x}.
\end{split}
\end{equation}
Therefore
\begin{equation}\label{cdv:eq:Schur-infimum}
  \boxed{
  \ip{Sx}{x}
  =\inf_{y\in\ell^2(Y_L)}
  \ip{\widehat B_{\lambda,p}\binom{x}{y}}{\binom{x}{y}},}
\end{equation}
and the unique minimizer is
\begin{equation}\label{cdv:eq:Schur-minimizer}
  y=-D^{-1}C^*x.
\end{equation}
By \eqref{cdv:eq:mixed-coercivity},
\begin{equation}\label{cdv:eq:S-coercive}
  S\geq c_\ast I,
  \qquad
  \norm{S^{-1}}_{\ell^2\to\ell^2}\leq c_\ast^{-1}.
\end{equation}

Using \eqref{cdv:eq:Schur-product}, \eqref{cdv:eq:D-inverse-Schur},
\eqref{cdv:eq:C-Schur}, and \eqref{cdv:eq:G-Schur}, we obtain
\begin{equation}\label{cdv:eq:S-Schur-bound}
\begin{split}
  \norm{S}_{\mathscr S_\alpha(X_L)}
  &\leq
  \norm{G}_{\mathscr S_\alpha(X_L)}
  +
  \norm{C}_{\mathscr S_{\alpha,L}}^2
  \norm{D^{-1}}_{\mathscr S_{\alpha,L}}
  \\
  &\leq M_\ast(d,s,C,\Psi).
\end{split}
\end{equation}
The low index set is a finite subset of the fixed finite-type lattice in
\eqref{cdv:eq:polynomial-growth}. Hence Lemma~\ref{cdv:lem:finite-Wiener}, together
with \eqref{cdv:eq:S-coercive}, gives
\begin{equation}\label{cdv:eq:S-inverse-Schur}
  \boxed{
  \norm{S^{-1}}_{\mathscr S_\alpha(X_L)}
  \leq C_S(d,s,c,C,\Psi).}
\end{equation}

The bounded-operator block inverse formula is
\begin{equation}\label{cdv:eq:block-inverse}
  \widehat B_{\lambda,p}^{-1}
  =
  \begin{pmatrix}
    S^{-1}&-S^{-1}CD^{-1}\\[1mm]
    -D^{-1}C^*S^{-1}&
    D^{-1}+D^{-1}C^*S^{-1}CD^{-1}
  \end{pmatrix}.
\end{equation}
Every product on the right belongs to the corresponding rectangular
weighted-Schur module. Thus
\begin{equation}\label{cdv:eq:B-inverse-Schur}
  \boxed{
  \norm{\widehat B_{\lambda,p}^{-1}}
  _{\mathscr S_{\alpha,L}(\Gamma_L,\Gamma_L)}
  \leq C_B(d,s,c,C,\Psi).}
\end{equation}
In particular,
\begin{equation}\label{cdv:eq:B-inverse-row-column}
\begin{split}
  \sup_{\gamma\in\Gamma_L}
  \sum_{\eta\in\Gamma_L}
  \abs{\widehat B_{\lambda,p}^{-1}(\gamma,\eta)}
  &\leq C_B,
  \\
  \sup_{\eta\in\Gamma_L}
  \sum_{\gamma\in\Gamma_L}
  \abs{\widehat B_{\lambda,p}^{-1}(\gamma,\eta)}
  &\leq C_B.
\end{split}
\end{equation}
These estimates are for the infinite mixed matrix itself; no convergence
of finite-section inverses is asserted or needed.

\subsection{Coefficient equation and endpoint estimates}

Let $f\in L^\infty(\mu_p)\subset L^2(\mu_p)$ and let
$u=(I+\lambda A_p)^{-1}f$. By Lemma~\ref{cdv:lem:U-isomorphism}, there is a
unique $a\in\ell^2(\Gamma_L)$ such that
\begin{equation}\label{cdv:eq:u-Ua}
  u=U_{\lambda,L}a.
\end{equation}
Define
\begin{equation}\label{cdv:eq:g-operator-definition}
  g=U_{\lambda,L}^*J_p^*f.
\end{equation}
Since $J_p^*f\in\mathcal H_s$ and $U_{\lambda,L}^*$ is bounded,
\begin{equation}\label{cdv:eq:g-l2}
  g\in\ell^2(\Gamma_L).
\end{equation}
For the coordinate vector $e_\eta$,
\begin{equation}\label{cdv:eq:g-coordinates}
\begin{split}
  g_\eta
  &=(U_{\lambda,L}^*J_p^*f,e_\eta)_{\ell^2}
  \\
  &=(J_p^*f,U_{\lambda,L}e_\eta)_{\mathcal H_s}
  \\
  &=(f,\zeta_\eta)_{L^2(\mu_p)}
  =\int_\Omega f(y)\overline{\zeta_\eta(y)}\,d\mu_p(y),
\end{split}
\end{equation}
where $\zeta_\eta=U_{\lambda,L}e_\eta$ is the corresponding mixed basis
function. Thus \eqref{cdv:eq:g-l2} is not an application of Parseval in
$L^2(\mu_p)$; it follows from the Hilbert-space adjoints in
\eqref{cdv:eq:g-operator-definition}.

For every $b\in\ell^2(\Gamma_L)$, the variational equation gives
\begin{equation}\label{cdv:eq:coefficient-weak-equation}
\begin{split}
  (\widehat B_{\lambda,p}a,b)_{\ell^2}
  &=b_{\lambda,p}(U_{\lambda,L}a,U_{\lambda,L}b)
  \\
  &=(f,U_{\lambda,L}b)_{L^2(\mu_p)}
  =(g,b)_{\ell^2}.
\end{split}
\end{equation}
Therefore, as an identity in $\ell^2(\Gamma_L)$,
\begin{equation}\label{cdv:eq:a-Binverse-g}
  \boxed{a=\widehat B_{\lambda,p}^{-1}g.}
\end{equation}

We now prove the endpoint estimates directly in the mixed basis. At any
$x\in\Omega$, only $O_\Psi(1)$ level-$L$ scaling functions are nonzero,
so
\begin{equation}\label{cdv:eq:scaling-pointwise-sum}
  \sup_{x\in\Omega}
  \sum_{(k,\tau)\in X_L}
  \abs{\phi_{L,k,\tau}(x)}
  \lesssim_\Psi2^{Ld/2}.
\end{equation}
At level $j=L+q$, only $O_\Psi(1)$ wavelets are nonzero at $x$, and
\begin{equation}\label{cdv:eq:tail-pointwise-level}
\begin{split}
  \sum_{k,e,\tau}
  \abs{\sigma_j\psi_{j,k,e,\tau}(x)}
  &\lesssim_\Psi\sigma_j2^{jd/2}
  \\
  &\leq
  2^{Ld/2}2^{-rm_\ast}2^{-sq}.
\end{split}
\end{equation}
Hence
\begin{equation}\label{cdv:eq:full-pointwise-sum}
\begin{split}
  \sup_{x\in\Omega}
  \sum_{\gamma\in\Gamma_L}\abs{\zeta_\gamma(x)}
  &\lesssim_\Psi
  2^{Ld/2}
  \left(1+\sum_{q\geq0}2^{-sq}\right)
  \\
  &\lesssim_{s,\Psi}2^{Ld/2}.
\end{split}
\end{equation}
The finite number of boundary types is included in the pointwise overlap
constant.

For a level-$L$ scaling function,
\begin{equation}\label{cdv:eq:scaling-L1mu}
  \norm{\phi_{L,k,\tau}}_{L^1(\mu_p)}
  \leq C\norm{\phi_{L,k,\tau}}_{L^1(dx)}
  \lesssim_{C,\Psi}2^{-Ld/2}.
\end{equation}
For a tail wavelet at level $L+q$,
\begin{equation}\label{cdv:eq:tail-L1mu}
\begin{split}
  \norm{\sigma_{L+q}\psi_{L+q,k,e,\tau}}_{L^1(\mu_p)}
  &\lesssim_{C,\Psi}
  \sigma_{L+q}2^{-(L+q)d/2}
  \\
  &\lesssim_{C,r,\Psi}
  2^{-Ld/2}2^{-rm_\ast}2^{-(r+d/2)q}.
\end{split}
\end{equation}
Thus
\begin{equation}\label{cdv:eq:uniform-L1mu}
  \sup_{\eta\in\Gamma_L}
  \norm{\zeta_\eta}_{L^1(\mu_p)}
  \lesssim_{C,\Psi}2^{-Ld/2}.
\end{equation}
Equation \eqref{cdv:eq:g-coordinates} then yields
\begin{equation}\label{cdv:eq:g-infinity-bound}
  \sup_{\eta\in\Gamma_L}\abs{g_\eta}
  \lesssim_{C,\Psi}
  \norm{f}_\infty2^{-Ld/2}.
\end{equation}

The expansion $u=U_{\lambda,L}a$ converges in $\mathcal H_s$ and hence
uniformly by \eqref{cdv:eq:reproducing-estimate}. Using
\eqref{cdv:eq:a-Binverse-g}, \eqref{cdv:eq:B-inverse-row-column}, and
\eqref{cdv:eq:full-pointwise-sum}, the series also converges absolutely at
each point, and
\begin{equation}\label{cdv:eq:resolvent-pointwise-representation}
\begin{split}
  u(x)
  ={}&
  \sum_{\gamma,\eta\in\Gamma_L}
  \zeta_\gamma(x)
  \widehat B_{\lambda,p}^{-1}(\gamma,\eta)
  \\
  &\hspace{5em}\times
  \int_\Omega f(y)\overline{\zeta_\eta(y)}\,d\mu_p(y).
\end{split}
\end{equation}
Consequently,
\begin{equation}\label{cdv:eq:small-lambda-final-estimate}
\begin{split}
  \abs{u(x)}
  &\leq
  \sup_\eta\abs{g_\eta}
  \sum_\gamma\abs{\zeta_\gamma(x)}
  \sum_\eta
  \abs{\widehat B_{\lambda,p}^{-1}(\gamma,\eta)}
  \\
  &\lesssim_{d,s,c,C,\Psi}
  \norm{f}_\infty
  \bigl(2^{Ld/2}\bigr)
  \bigl(2^{-Ld/2}\bigr)
  \\
  &=C(d,s,c,C,\Psi)\norm{f}_\infty.
\end{split}
\end{equation}
The scale cancellation is exactly
\begin{equation}\label{cdv:eq:scale-cancellation}
  \boxed{2^{Ld/2}\cdot2^{-Ld/2}=1.}
\end{equation}
We have proved
\begin{equation}\label{cdv:eq:small-lambda-operator}
  \sup_{p\in\mathcal P(c,C)}
  \sup_{0<\lambda<1}
  \norm{(I+\lambda A_p)^{-1}}_{L^\infty\to L^\infty}
  \lesssim_{d,s,c,C,\Psi}1.
\end{equation}
Combining \eqref{cdv:eq:small-lambda-operator} with
\eqref{cdv:eq:large-lambda-operator} and using
\eqref{cdv:eq:resolvent-reduction} proves Theorem~\ref{thm:uniform-cdv-resolvent}.

\subsection{Dependence of the constants and scope of the result}

The proof has the following precise scope.
\begin{enumerate}[label=\textup{(\arabic*)}]
\item No regularity of $p$ beyond $c\leq p\leq C$ was used. Every mass
estimate followed from compact support, scale-invariant $L^1$ and
$L^\infty$ bounds, overlap counting, and the inequality $p\leq C$. The
lower bound $p\geq c$ was used for coercivity and for the isometric
identification \eqref{cdv:eq:Linfty-isometry}.

\item After $m_\ast$ is fixed by \eqref{cdv:eq:mstar-choice}, every constant
is independent of $p,\lambda,J$, and $L$.

\item The argument is valid for every fixed $s>0$. The decisive
cross-scale series is $\sum_{q\geq0}2^{-sq}$.

\item The estimates are not uniform as $s\downarrow0$. In particular,
$(1-2^{-s})^{-1}$ diverges, and the depth required in
\eqref{cdv:eq:mstar-choice} may increase.

\item The estimates are not uniform as $c\downarrow0$ or
$C\uparrow\infty$. The parameter $c$ controls coercivity and the inverse
of the Schur complement, while $C$ controls the mass matrices and the
tail-splitting depth.

\item Sobolev norm equivalence alone would give the Hilbert-space
identification and embedding, but it would not imply localization of the
inverse energy matrix. Compact support, multiresolution localization,
and the finite-type lattice geometry are essential.

\item The only external matrix inversion result is the norm-controlled
weighted-Schur Wiener theorem of Sun on the fixed finite-type lattice
$\mathbb Z^d\times\mathcal T$. The tail inverse is obtained directly by
a Neumann series, and the full inverse is obtained by the bounded-operator
Schur complement.
\end{enumerate}

The final conclusion is
\begin{equation*}
  \boxed{
  \sup_{p\in\mathcal P(c,C)}
  \sup_{\lambda>0}
  \norm{
  T_p(T_p+\lambda I)^{-1}
  }_{L^\infty(\mu_p)\to L^\infty(\mu_p)}
  \leq M(d,s,c,C,\Psi).}
\end{equation*}

\subsection{Identification with the population ridge operator}
\label{sec:cdv-main-kernel-identification}

The preceding theorem was written with the conventional CDV indexing
\(V_0\oplus\bigoplus_{j\ge0}W_j\).  Under
\eqref{eq:coarse-block-identification}, the coarse block in the main
text is \(P_0=\Pi_0+Q_0\), while \(P_j=Q_j\) for \(j\ge1\).
Consequently, the unscaled kernel \(K_\nu/c_K\) in
\eqref{eq:localized-wavelet-kernel} is exactly the kernel in
\eqref{cdv:eq:kernel-definition}.  More generally, a fixed shift of the
detail index or the absorption of finitely many initial detail spaces
into the coarse block changes only fixed form constants, which can be
absorbed into \(\Psi\).

\begin{corollary}[Uniform population ridge stability]
\label{cor:localized-ridge-linfty-sun}
Let
\[
  S_{\lambda,p}
  :=(L_{K_\nu,p}+\lambda I)^{-1}L_{K_\nu,p}.
\]
Under Assumptions~\ref{ass:design-noise} and
\ref{ass:localized-wavelet},
\[
  \sup_{0<\lambda\le c_Kc_p}
  \norm{S_{\lambda,p}}_{L^\infty(\Omega)\to L^\infty(\Omega)}
  \le C,
\]
where \(C\) is uniform over all admissible design densities.
\end{corollary}

\begin{proof}
Let \(\overline K_\nu:=K_\nu/c_K\) and let
\[
  (\overline T_pf)(x)
  :=\int_\Omega\overline K_\nu(x,y)f(y)p(y)\dd y.
\]
The exact coarse-block identification above makes \(\overline T_p\) the
operator \(T_p\) of Theorem~\ref{thm:uniform-cdv-resolvent}.  Apply that theorem to the
original probability density \(p\), with \(c=p_-\) and \(C=p_+\).
This point is important: no normalization of \(p\) is needed in the
application of Sun's theorem.  Since
\[
  L_{K_\nu,p}=c_K\overline T_p,
\]
functional calculus on \(L^2(P_X)\) gives
\[
\begin{split}
  S_{\lambda,p}
  &=(\overline T_p+(\lambda/c_K)I)^{-1}\overline T_p\\
  &=\overline T_p(\overline T_p+(\lambda/c_K)I)^{-1}.
\end{split}
\]
Theorem~\ref{thm:uniform-cdv-resolvent} therefore gives the asserted
\(L^\infty\) bound.  Finally, \(p_->0\) implies
\(L^\infty(P_X)=L^\infty(dx)\) isometrically.

For compatibility with Appendix~\ref{app:localized-kernel}, one may then
write \(\overline T_p=c_pT_{\widetilde p}\) and
\(\eta=\lambda/(c_Kc_p)\), which yields
\[
  S_{\lambda,p}
  =(T_{\widetilde p}+\eta I)^{-1}T_{\widetilde p}.
\]
This is an algebraic rewriting after the theorem has been applied to the
probability density \(p\); \(\widetilde p=p/c_p\) itself need not be a
probability density.
\end{proof}

\section{Proof of the Gaussian derivative bounds}
\label{app:gaussian-derivatives}

Throughout this appendix, $m=|\alpha|<s$, $L=\log(en)$,
$h=n^{-1/(2s+d)}$, $\lambda=n^{-b}$ with fixed $b\ge1$, and
$\ell=h\sqrt L$.  We work in the Gaussian RKHS on $\RR^d$.
Its KRR minimizer and the minimizer for the restricted kernel have the
same finite kernel expansion on $\Omega$.  Constants may depend on the
fixed model parameters and the specified remainder order, but are uniform
in $n$, the density, and the target.

\subsection{A common design event}

Partition $\Omega$ into cubes $Q$ of side
$a_h=1/\lceil1/h\rceil\in[h/2,h]$ and let
$M=\lceil C_b L\rceil$.  Let $\mathcal V_M(Q)$ contain polynomials of
coordinatewise degree at most $M$.
There is an event $\mathcal E_n$ such that, simultaneously for all
cells and all $P\in\mathcal V_M(Q)$,
\begin{equation}
 \norm P_{L^2(Q)}^2\le\frac Cn\sum_{X_i\in Q}P(X_i)^2,
 \qquad P_n(Q)\le Ch^d,
 \label{eq:gaussian-design-event}
\end{equation}
and
\begin{equation}
 \PP_X(\mathcal E_n^c)
 \le Ch^{-d}(M+1)^d
       \exp\{-cnh^d/(M+1)^{2d}\}.
 \label{eq:gaussian-design-failure}
\end{equation}
To prove this, take the tensor Legendre basis on $Q$, orthonormal in
$L^2(Q)$ and extended by zero outside $Q$.  Its feature vector $\Phi_Q$
has dimension $(M+1)^d$ and
\[
 \sup_x\norm{\Phi_Q(x)}_2^2\le Ch^{-d}(M+1)^{2d},
 \qquad p_-I\preceq\EE\Phi_Q(X)\Phi_Q(X)^\top\preceq p_+I.
\]
The matrix Chernoff lower-tail inequality
\citep[Corollary~5.2]{tropp2012userfriendly}, followed by a union bound
over the $O(h^{-d})$ cells, gives the first part of
\eqref{eq:gaussian-design-event} and the displayed failure probability.
Scalar Chernoff bounds give the cell counts.  Since
$nh^d=n^{2s/(2s+d)}$ and $M\asymp L$, the right side of
\eqref{eq:gaussian-design-failure} is smaller than every fixed power
$n^{-A}$ for all sufficiently large $n$.

On the same event, counting intersected cells in successive annuli gives,
for every fixed $c_0>0$ and $a\ge h$,
\begin{equation}
 \sup_{x\in\Omega}\frac1n\sum_i
       e^{-c_0\norm{x-X_i}/a}\le C a^d,
 \qquad
 \norm{K_{h/\sqrt2}/n}_{\op}\le Ch^d.
 \label{eq:gaussian-exponential-mass}
\end{equation}
Indeed the empirical mass within radius $(j+1)a$ is at most
$C(j+1)^d a^d$, and $\sum_j(j+1)^d e^{-c_0j}<\infty$.
The Gaussian version at scale $h$ and the matrix Schur bound give the
second inequality.

\subsection{Derivative leverage}

For each multi-index $\beta$, define
\[
 Q_{\beta,n}(x):=
 \sup_{g\in\HH_h\setminus\{0\}}
 \frac{|D^\beta g(x)|^2}{\norm g_n^2+\lambda\norm g_{\HH_h}^2},
 \qquad \norm g_n^2=\frac1n\sum_i g(X_i)^2.
\]
\begin{lemma}[Uniform derivative leverage]
\label{lem:gaussian-derivative-leverage}
On $\mathcal E_n$, for all sufficiently large $n$ and each fixed
$\beta$ with $a=|\beta|\le m$,
\begin{equation}
 \sup_x Q_{\beta,n}(x)
 \le C h^{-d-2a}L^{2d+4a}=:\overline Q_a.
 \label{eq:gaussian-derivative-leverage}
\end{equation}
\end{lemma}
\begin{proof}
The derivative reproducing property of the Gaussian RKHS
\citep{steinwart2006explicit} gives, for a unit vector $u$ and integer $q$,
\begin{equation}
 \norm{D_u^q g}_\infty
 \le h^{-q}\sqrt{(2q)!/q!}\norm g_{\HH_h}
 \le(C\sqrt q/h)^q\norm g_{\HH_h}.
 \label{eq:gaussian-rkhs-derivatives}
\end{equation}
For a fixed mixed derivative of order $a$ it also gives
$\norm{D^\beta g}_\infty\le C_\beta h^{-a}\norm g_{\HH_h}$.
These estimates hold on all of $\RR^d$.

Fix $x\in Q$ and let $P_x$ be the Taylor polynomial of $g$ at $x$
of total degree $M-1$.  For $M>m$,
$D^\beta P_x(x)=D^\beta g(x)$.  The directional Taylor remainder,
\eqref{eq:gaussian-rkhs-derivatives}, and $\operatorname{diam}(Q)\le\sqrt d h$
show that
\[
 \norm{g-P_x}_{L^\infty(Q)}\le\eta_M\norm g_{\HH_h},
 \qquad \eta_M\le(C_d/\sqrt M)^M.
\]
The Legendre evaluation bound and the polynomial Markov inequality,
applied once for each derivative, imply
\[
 |D^\beta P(x)|
 \le C h^{-a}(M+1)^{2a}\norm P_{L^\infty(Q)}
 \le C h^{-d/2-a}(M+1)^{d+2a}\norm P_{L^2(Q)}.
\]
Use \eqref{eq:gaussian-design-event} for $P_x$ and the cell-count bound
to conclude
\[
 |D^\beta g(x)|
 \le C h^{-d/2-a}(M+1)^{d+2a}\norm g_n
      +C h^{-a}(M+1)^{d+2a}\eta_M\norm g_{\HH_h}.
\]
As $M\asymp L$, $\eta_M$ is smaller than every fixed negative power
of $n$; in particular $\eta_M\le h^{-d/2}\sqrt\lambda$ for large $n$.
Squaring gives \eqref{eq:gaussian-derivative-leverage}.
\end{proof}

\subsection{Exponential conjugation of the empirical inverse}

Write $T=K_h/n$, $A=T+\lambda I_n$, and endow $\mathbb C^n$ with
$\ip uv_n=n^{-1}u^*v$.  For $t\in\RR^d$ define
\[
 D_t=\operatorname{diag}(e^{t\cdot X_i}),\qquad
 T_t=D_tTD_t^{-1},\qquad A_t=T_t+\lambda I_n.
\]
We claim that, on $\mathcal E_n$, there is a fixed $c_*>0$ such that
\begin{equation}
 \norm{A^{-1/2}(T_t-T)A^{-1/2}}_{\op}\le\frac12,
 \qquad
 \norm{A^{1/2}A_t^{-1}A^{1/2}}_{\op}\le2
 \quad\text{if }\norm t\le c_*/\ell.
 \label{eq:gaussian-conjugated-resolvent}
\end{equation}
Here $T_t$ is generally not self-adjoint.

To verify the claim, use the Fourier transform
$\widehat u(\omega)=\int e^{-i\omega\cdot z}u(z)\dd z$ and
$\kappa_h(z)=e^{-|z|^2/h^2}$.  Direct calculation gives
\[
 \widehat\kappa_h(\omega)=\pi^{d/2}h^d e^{-h^2|\omega|^2/4},
 \qquad
 \widehat{e^{t\cdot z}\kappa_h}(\omega)
 =e^{h^2|t|^2/4}e^{-ih^2t\cdot\omega/2}\widehat\kappa_h(\omega).
\]
For any fixed $N>b+1$ and small $\varepsilon>0$, choosing $B$ large
and then $c_*$ small ensures
\begin{equation}
 |\widehat{e^{t\cdot z}\kappa_h}(\omega)-\widehat\kappa_h(\omega)|
 \le\varepsilon\widehat\kappa_h(\omega)
       +Cn^{-N}\widehat\kappa_{h/\sqrt2}(\omega).
 \label{eq:gaussian-fourier-perturbation}
\end{equation}
On $h|\omega|\le B\sqrt L$, the relative multiplier difference is
small since $h^2|t\cdot\omega|\le c_*B$.  On the complement,
$\widehat\kappa_h/\widehat\kappa_{h/\sqrt2}
=2^{d/2}e^{-h^2|\omega|^2/8}\le Cn^{-N}$ after increasing $B$.
Fourier inversion for the discrete measures
$n^{-1}\sum_i v_i\delta_{X_i}$ and Cauchy--Schwarz give the form bound
\[
 |\ip u{(T_t-T)v}_n|
 \le\varepsilon\ip u{Tu}_n^{1/2}\ip v{Tv}_n^{1/2}
 +Cn^{-N}\ip u{T'u}_n^{1/2}\ip v{T'v}_n^{1/2},
 \qquad T'=K_{h/\sqrt2}/n.
\]
Consequently the first norm in \eqref{eq:gaussian-conjugated-resolvent}
is at most $\varepsilon+Cn^{-N}h^d/\lambda\le1/2$ for large $n$.
Factoring $A_t=A^{1/2}(I+B_t)A^{1/2}$ and using a Neumann series
proves the second bound.

\subsection{Differentiated empirical weights}

For $v\in\RR^n$, write
\[
 Sv(x)=k_{h,x}^\top(K_h+n\lambda I_n)^{-1}v
       =\frac1n\sum_i a_x(i)v_i,
 \qquad a_x=n(K_h+n\lambda I_n)^{-1}k_{h,x}.
\]
Let
\[
 U_tv(x)=\frac1n\sum_i e^{t\cdot(x-X_i)}k_h(x,X_i)v_i,
 \qquad S_t=U_tA_t^{-1}.
\]
Completing the square translates all kernel centers by $h^2t/2$,
and hence
\[
 \norm{U_tv}_{\HH_h}^2=e^{h^2|t|^2/2}\ip v{Tv}_n.
\]
Also $(S_tv)|_X=(I-\lambda A_t^{-1})v$.
Equation~\eqref{eq:gaussian-conjugated-resolvent} therefore implies
\begin{equation}
 \norm{S_tv}_n+\sqrt\lambda\norm{S_tv}_{\HH_h}\le C\norm v_n.
 \label{eq:gaussian-conjugated-energy}
\end{equation}
On the other hand, $S_tv(x)=e^{t\cdot x}S(D_t^{-1}v)(x)$, so
\begin{equation}
 (D-t)^\alpha S_tv(x)
 =\frac1n\sum_i e^{t\cdot(x-X_i)}D^\alpha a_x(i)v_i.
 \label{eq:gaussian-conjugated-derivative-row}
\end{equation}
Use the binomial expansion of $(D-t)^\alpha$,
Lemma~\ref{lem:gaussian-derivative-leverage}, and
\eqref{eq:gaussian-conjugated-energy} to bound its absolute value by
\[
 C\norm v_n\sum_{\beta\le\alpha}
      |t^{\alpha-\beta}|\,h^{-d/2-|\beta|}L^{d+2|\beta|}
 \le C h^{-d/2-m}L^{d+2m}\norm v_n.
\]
The Riesz representation in \eqref{eq:gaussian-conjugated-derivative-row}
and the $2^d$ choices $t=c_*\eta/(\sqrt d\ell)$,
$\eta\in\{-1,1\}^d$, give
\begin{equation}
 \sup_x\frac1n\sum_i
     e^{c\norm{x-X_i}/\ell}|D^\alpha a_x(i)|^2
 \le C h^{-d-2m}L^{2d+4m}.
 \label{eq:gaussian-derivative-weighted-row}
\end{equation}
Thus differentiating a zeroth-order weight estimate is unnecessary;
the differentiated row is bounded through its own evaluation functional.

Put $w_i^{(\alpha)}(x)=D^\alpha a_x(i)/n$.
Cauchy--Schwarz, \eqref{eq:gaussian-exponential-mass} at scale $\ell$,
and \eqref{eq:gaussian-derivative-weighted-row} imply
\begin{align}
 \sup_x\sum_i|w_i^{(\alpha)}(x)|
 &\le C\overline Q_m^{1/2}\ell^{d/2}
   \le Ch^{-m}L^{5d/4+2m},
 \label{eq:gaussian-derivative-lebesgue}\\
 \sup_x\sum_{\norm{X_i-x}>r}|w_i^{(\alpha)}(x)|
 &\le Ch^{-m}L^{5d/4+2m}e^{-cr/\ell}.
 \label{eq:gaussian-derivative-weight-tail}
\end{align}
For the second line split the decaying exponential in the
Cauchy--Schwarz factor into two equal parts.

\subsection{A local RKHS approximant with derivative accuracy}

Choose an extension $F_0$ of $f^\ast$ to $\RR^d$ with
$\norm{F_0}_{B^s_{\infty,\infty}}\le2F$.
Let $P_h$ denote convolution with
$\pi^{-d/2}h^{-d}e^{-|z|^2/h^2}$.
Fix an integer $q$ with $2q>s$ and define
\[
 \mathcal G_h=I-(I-P_h)^q
   =\sum_{j=1}^q(-1)^{j+1}\binom qj P_h^j.
\]
The heat-semigroup identity and the band-limited Bernstein inequality
give, on each Littlewood--Paley block $\Delta_k$,
\[
 \norm{D^\beta(I-P_h)^q\Delta_kF_0}_\infty
 \le C2^{k|\beta|}\min\{1,(h2^k)^{2q}\}
       \norm{\Delta_kF_0}_\infty.
\]
Summing below and above $2^k\asymp h^{-1}$ yields
\begin{equation}
 \norm{D^\beta(\mathcal G_hF_0-F_0)}_\infty
 \le CFh^{s-|\beta|},\qquad |\beta|\le m<s.
 \label{eq:gaussian-derivative-approximation}
\end{equation}
This proof includes integer $s$.
For $u\in L^2(\RR^d)$, the Gaussian Fourier RKHS norm and Plancherel
give
\begin{equation}
 \norm{\mathcal G_hu}_{\HH_h}\le Ch^{-d/2}\norm u_2,
 \qquad
 \norm{\mathcal G_hu}_\infty\le(2^q-1)\norm u_\infty.
 \label{eq:gaussian-approximation-rkhs}
\end{equation}
Indeed $P_h^j$ has multiplier $e^{-jh^2|\omega|^2/4}$, whose squared
modulus divided by $\widehat\kappa_h$ is bounded by $Ch^{-d}$.

Take $r_n=A_0hL^{3/2}=A_0L\ell$, where $A_0$ is a sufficiently large
fixed constant.  For each fixed center $x$, set
\[
 F_x=F_0\one_{B(x,2r_n)},\qquad g_x=\mathcal G_hF_x.
\]
Equation~\eqref{eq:gaussian-approximation-rkhs} gives
\begin{equation}
 \norm{g_x}_{\HH_h}\le CF(r_n/h)^{d/2}\le CF L^{3d/4},
 \qquad \norm{g_x}_\infty\le CF.
 \label{eq:gaussian-local-norm}
\end{equation}
If $|y-x|\le r_n$, the removed part of $F_0$ is at distance at least
$r_n$ from $y$.  The derivatives of each Gaussian convolution kernel
have $L^1$ tails bounded by
$C h^{-|\beta|}e^{-c(r_n/h)^2}$, with constants allowed to absorb a
fixed polynomial factor.  These tails are smaller than every power of
$n^{-1}$.  Combining them with
\eqref{eq:gaussian-derivative-approximation}, for large $n$,
\begin{equation}
 \sup_{|y-x|\le r_n}|g_x(y)-f^\ast(y)|\le CFh^s,
 \qquad
 |D^\alpha g_x(x)-D^\alpha f^\ast(x)|\le CFh^{s-m},
 \label{eq:gaussian-local-approximation}
\end{equation}
where $y$ is restricted to $\Omega$ in the first supremum.
When taking derivatives, the center $x$ specifying $g_x$ is held fixed;
we differentiate the argument of that function and then evaluate at $x$.

\subsection{Conditional bias}

Let $\overline f=\EE[\widehat f^G_{h,\lambda}\mid X]=Sf^\ast$.
For a fixed $x$ decompose
\begin{align*}
 D^\alpha\overline f(x)-D^\alpha f^\ast(x)
 &=\sum_i w_i^{(\alpha)}(x)\{f^\ast(X_i)-g_x(X_i)\}\\
 &\quad+D^\alpha(Sg_x-g_x)(x)
        +D^\alpha(g_x-f^\ast)(x).
\end{align*}
The near and far parts of the first sum are bounded by
\eqref{eq:gaussian-local-approximation},
\eqref{eq:gaussian-derivative-lebesgue}, and
\eqref{eq:gaussian-derivative-weight-tail}:
\[
 \left|\sum_i w_i^{(\alpha)}(x)\{f^\ast(X_i)-g_x(X_i)\}\right|
 \le CFh^{s-m}L^{5d/4+2m}
       +CFh^{-m}L^{5d/4+2m}e^{-cA_0L}.
\]
Increasing $A_0$ makes the last term at most $C_AF n^{-A}$ for
any specified fixed $A>0$.

For $g\in\HH_h$, let $C_X=n^{-1}\sum_i k_h(\cdot,X_i)\otimes
k_h(\cdot,X_i)$.  The normal equation gives
$Sg-g=-\lambda(C_X+\lambda I)^{-1}g$.
If $r_{\alpha,x}$ represents derivative evaluation in $\HH_h$,
then
\[
 Q_{\alpha,n}(x)
 =\norm{(C_X+\lambda I)^{-1/2}r_{\alpha,x}}_{\HH_h}^2.
\]
Consequently
\[
 |D^\alpha(Sg-g)(x)|
 \le\sqrt{\lambda Q_{\alpha,n}(x)}\norm g_{\HH_h}.
\]
Applying this to $g_x$ and using
\eqref{eq:gaussian-derivative-leverage} and
\eqref{eq:gaussian-local-norm}, we obtain, on $\mathcal E_n$,
\begin{equation}
 \norm{D^\alpha\overline f-D^\alpha f^\ast}_\infty
 \le C\left\{Fh^{s-m}L^{5d/4+2m}
     +F\sqrt{\frac\lambda{h^{d+2m}}}L^{7d/4+2m}\right\}
     +C_AF n^{-A}.
 \label{eq:gaussian-derivative-bias}
\end{equation}

\subsection{Noise in second moment and the bad design event}

Put $Z=S\epsilon$.  Conditionally on the design,
$D^\alpha Z(x)=\sum_i w_i^{(\alpha)}(x)\epsilon_i$ is centered
sub-Gaussian with variance proxy at most
\begin{equation}
 \sigma^2\sum_i|w_i^{(\alpha)}(x)|^2
 \le C\sigma^2\overline Q_m/n
 \le C\sigma^2 h^{-d-2m}L^{2d+4m}/n.
 \label{eq:gaussian-derivative-noise-proxy}
\end{equation}
This proves the fixed-point root mean square bound without a spatial
logarithm.  To control a supremum, compare the objective for $Z$ with
zero and use conditional variances to get
\[
 \EE_\epsilon[\norm Z_{\HH_h}^2\mid X]\le\sigma^2/\lambda.
\]
Fixed-order derivative reproducing bounds imply
$\norm{\nabla D^\alpha Z}_\infty\le Ch^{-m-1}\norm Z_{\HH_h}$.
Take a coordinate grid with spacing
$\delta_n=h^{m+1}\sqrt\lambda\,n^{-(A+2)}$.
Its cardinality $N_n$ satisfies $\log(2N_n)\le C_A L$.
A union bound for centered sub-Gaussian variables, integrated in
second moment, gives
\[
 \EE_\epsilon\left[\max_{x\text{ on grid}}|D^\alpha Z(x)|^2
                  \,\middle|\,X\right]
 \le C\sigma^2\overline Q_m\log(2N_n)/n.
\]
No independence among grid values is needed.
The root mean square grid interpolation error is at most
$C\delta_nh^{-m-1}\sigma/\sqrt\lambda\le C\sigma n^{-(A+2)}$.
Minkowski's inequality therefore yields, on $\mathcal E_n$,
\begin{equation}
 \left(\EE_\epsilon[\norm{D^\alpha Z}_\infty^2\mid X]\right)^{1/2}
 \le\frac{C\sigma L^{d+2m+1/2}}{\sqrt{nh^{d+2m}}}
      +C_A\sigma n^{-A}.
 \label{eq:gaussian-derivative-noise-rms}
\end{equation}

For every design, the KRR objective inequality and the Gaussian
derivative reproducing bound give
\[
 \norm{D^\alpha\widehat f^G_{h,\lambda}}_\infty
 \le Ch^{-m}\norm{\widehat f^G_{h,\lambda}}_{\HH_h}
 \le Ch^{-m}\lambda^{-1/2}\norm Y_n.
\]
Since $\EE[\norm Y_n^2\mid X]\le C(F^2+\sigma^2)$ and
$\norm{D^\alpha f^\ast}_\infty\le CF$, the conditional squared
error is at most
$C(F^2+\sigma^2)(h^{-2m}\lambda^{-1}+1)$.
Equation~\eqref{eq:gaussian-design-failure} then implies
\[
 \left(\EE\left[\one_{\mathcal E_n^c}
   \norm{D^\alpha\widehat f^G_{h,\lambda}-D^\alpha f^\ast}_\infty^2
           \right]\right)^{1/2}
 \le C_A(F+\sigma)n^{-A}.
\]
Together with \eqref{eq:gaussian-derivative-bias} and
\eqref{eq:gaussian-derivative-noise-rms}, this proves
Proposition~\ref{prop:gaussian-intermediate} in second moment.

\begin{proof}[Proof of Theorem~\ref{thm:gaussian-derivatives}]
Under \eqref{eq:gaussian-tuning},
\begin{align*}
 h^{s-m}&=(nh^{d+2m})^{-1/2}=n^{-(s-m)/(2s+d)},\\
 \sqrt{\lambda/h^{d+2m}}
 &=n^{-(s-m)/(2s+d)}n^{-(b-1)/2}
   \le n^{-(s-m)/(2s+d)}.
\end{align*}
For $d\ge1$, $7d/4+2m$ dominates both $5d/4+2m$ and
$d+2m+1/2$.  Substitute into
\eqref{eq:gaussian-intermediate-rms} and choose a fixed sufficiently
large $A$ to absorb the remainder.  This proves
\eqref{eq:gaussian-derivative-rms}; squaring it proves
\eqref{eq:gaussian-derivative-squared}.
\end{proof}

\bibliographystyle{plainnat}
\bibliography{references}

@article{tropp2012userfriendly,
  author = {Tropp, Joel A.},
  title = {User-Friendly Tail Bounds for Sums of Random Matrices},
  journal = {Foundations of Computational Mathematics},
  volume = {12},
  number = {4},
  pages = {389--434},
  year = {2012},
  doi = {10.1007/s10208-011-9099-z}
}

@article{steinwart2006explicit,
  author = {Steinwart, Ingo and Hush, Don and Scovel, Clint},
  title = {An Explicit Description of the Reproducing Kernel {Hilbert} Spaces of {Gaussian RBF} Kernels},
  journal = {IEEE Transactions on Information Theory},
  volume = {52},
  number = {10},
  pages = {4635--4643},
  year = {2006},
  doi = {10.1109/TIT.2006.881713}
}

@article{zhang2024optimality,
  title={On the optimality of misspecified spectral algorithms},
  author={Zhang, Haobo and Li, Yicheng and Lin, Qian},
  journal={Journal of Machine Learning Research},
  volume={25},
  number={188},
  pages={1--50},
  year={2024}
}

@article{hamm2021adaptive,
  title={Adaptive learning rates for support vector machines working on data with low intrinsic dimension},
  author={Hamm, Thomas and Steinwart, Ingo},
  journal={The Annals of Statistics},
  volume={49},
  number={6},
  pages={3153--3180},
  year={2021},
  publisher={Institute of Mathematical Statistics}
}

@article{CohenDaubechiesVial1993,
  author  = {Cohen, Albert and Daubechies, Ingrid and Vial, Pierre},
  title   = {Wavelets on the Interval and Fast Wavelet Transforms},
  journal = {Applied and Computational Harmonic Analysis},
  volume  = {1},
  number  = {1},
  pages   = {54--81},
  year    = {1993},
  doi     = {10.1006/acha.1993.1005}
}

@article{Sun2007,
  author  = {Sun, Qiyu},
  title   = {Wiener's Lemma for Infinite Matrices},
  journal = {Transactions of the American Mathematical Society},
  volume  = {359},
  number  = {7},
  pages   = {3099--3123},
  year    = {2007},
  doi     = {10.1090/S0002-9947-07-04303-6}
}

@book{GineNickl2016,
  author    = {Gin{\'e}, Evarist and Nickl, Richard},
  title     = {Mathematical Foundations of Infinite-Dimensional Statistical Models},
  publisher = {Cambridge University Press},
  year      = {2016},
  series    = {Cambridge Series in Statistical and Probabilistic Mathematics},
  doi       = {10.1017/CBO9781107337862}
}

@article{chhor2022benign,
  title={Benign overfitting and adaptive nonparametric regression},
  author={Chhor, Julien and Sigalla, Suzanne and Tsybakov, Alexandre B},
  journal={arXiv preprint arXiv:2206.13347},
  year={2022}
}

@article{stone1982optimal,
  title={Optimal global rates of convergence for nonparametric regression},
  author={Stone, Charles J},
  journal={The annals of statistics},
  pages={1040--1053},
  year={1982},
  publisher={JSTOR}
}

@article{Pinelis1994,
  author  = {Pinelis, Iosif},
  title   = {Optimum Bounds for the Distributions of Martingales in Banach Spaces},
  journal = {The Annals of Probability},
  volume  = {22},
  number  = {4},
  pages   = {1679--1706},
  year    = {1994},
  doi     = {10.1214/aop/1176988477}
}

@article{Tropp2015,
  author  = {Tropp, Joel A.},
  title   = {An Introduction to Matrix Concentration Inequalities},
  journal = {Foundations and Trends in Machine Learning},
  volume  = {8},
  number  = {1--2},
  pages   = {1--230},
  year    = {2015},
  doi     = {10.1561/2200000048}
}

@book{Tsybakov2009,
  author    = {Tsybakov, Alexandre B.},
  title     = {Introduction to Nonparametric Estimation},
  publisher = {Springer},
  year      = {2009},
  series    = {Springer Series in Statistics},
  doi       = {10.1007/b13794}
}

@article{stone1977consistent,
  title={Consistent nonparametric regression},
  author={Stone, Charles J},
  journal={The annals of statistics},
  pages={595--620},
  year={1977},
  publisher={JSTOR}
}

@article{stone1980optimal,
  title={Optimal rates of convergence for nonparametric estimators},
  author={Stone, Charles J},
  journal={The annals of Statistics},
  pages={1348--1360},
  year={1980},
  publisher={JSTOR}
}

@article{yang1999information,
  title={Information-theoretic determination of minimax rates of convergence},
  author={Yang, Yuhong and Barron, Andrew},
  journal={Annals of Statistics},
  pages={1564--1599},
  year={1999},
  publisher={JSTOR}
}

@article{fan1995local,
  title={Local polynomial kernel regression for generalized linear models and quasi-likelihood functions},
  author={Fan, Jianqing and Heckman, Nancy E and Wand, Matt P},
  journal={Journal of the American Statistical Association},
  volume={90},
  number={429},
  pages={141--150},
  year={1995},
  publisher={Taylor \& Francis}
}

@article{tuo2024asymptotic,
  title={Asymptotic theory for linear functionals of kernel ridge regression},
  author={Tuo, Rui and Zou, Lu},
  journal={arXiv preprint arXiv:2403.04248},
  year={2024}
}

@inproceedings{steinwart2009optimal,
  title={Optimal Rates for Regularized Least Squares Regression.},
  author={Steinwart, Ingo and Hush, Don R and Scovel, Clint and others},
  booktitle={COLT},
  pages={79--93},
  year={2009}
}

@article{caponnetto2007optimal,
  title={Optimal rates for the regularized least-squares algorithm},
  author={Caponnetto, Andrea and De Vito, Ernesto},
  journal={Foundations of Computational mathematics},
  volume={7},
  number={3},
  pages={331--368},
  year={2007},
  publisher={Springer}
}

@article{fischer2020sobolev,
  author = {Fischer, Simon and Steinwart, Ingo},
  title = {{Sobolev} Norm Learning Rates for Regularized Least-Squares Algorithms},
  journal = {Journal of Machine Learning Research},
  volume = {21},
  number = {205},
  pages = {1--38},
  year = {2020},
  url = {https://jmlr.org/papers/v21/19-734.html}
}

@misc{lin2025fixed,
  author = {Lin, Haotian and Reimherr, Matthew},
  title = {Fixed-{Gaussian} Spectral Algorithms: Minimax Optimal Rates for Misspecified Learning and Transfer},
  year = {2025},
  eprint = {2501.10870},
  archivePrefix = {arXiv},
  note = {Version 2, updated 7 July 2026},
  url = {https://arxiv.org/abs/2501.10870}
}

@article{liu2023derivatives,
  author = {Liu, Zejian and Li, Meng},
  title = {On the Estimation of Derivatives Using Plug-in Kernel Ridge Regression Estimators},
  journal = {Journal of Machine Learning Research},
  volume = {24},
  number = {266},
  pages = {1--37},
  year = {2023},
  url = {https://jmlr.org/papers/v24/21-1110.html}
}

@article{liu2026optimal,
  author = {Liu, Zejian and Li, Meng},
  title = {Optimal Plug-in {Gaussian} Processes for Modeling Derivatives},
  journal = {Journal of the American Statistical Association},
  volume = {121},
  number = {554},
  pages = {1562--1573},
  year = {2026},
  doi = {10.1080/01621459.2026.2612778},
  url = {https://arxiv.org/abs/2210.11626}
}

@misc{yang2017frequentist,
  author = {Yang, Yun and Bhattacharya, Anirban and Pati, Debdeep},
  title = {Frequentist Coverage and Sup-Norm Convergence Rate in {Gaussian} Process Regression},
  year = {2017},
  eprint = {1708.04753},
  archivePrefix = {arXiv},
  url = {https://arxiv.org/abs/1708.04753}
}

@article{yoo2016supremum,
  author = {Yoo, William Weimin and Ghosal, Subhashis},
  title = {Supremum Norm Posterior Contraction and Credible Sets for Nonparametric Multivariate Regression},
  journal = {The Annals of Statistics},
  volume = {44},
  number = {3},
  pages = {1069--1102},
  year = {2016},
  doi = {10.1214/15-AOS1398},
  url = {https://arxiv.org/abs/1411.6716}
}

\end{document}